%% file: main.tex
\documentclass[Afour,sageh,times]{sagej}
\usepackage{amsmath,amsfonts}
\usepackage{algorithmic}
\usepackage{array}
\usepackage[font=small, labelfont=bf]{subfig}
\usepackage{textcomp}
\usepackage{url}
\usepackage{verbatim}
\usepackage{graphicx}
\usepackage{multirow}
\usepackage{multicol}
\usepackage[bookmarks=true]{hyperref}
\usepackage{graphics} 
\usepackage{epsfig} 
\usepackage{mathrsfs}
\usepackage[font=small,labelfont=bf]{caption}
\usepackage{times} 
\usepackage{amssymb}  
\usepackage[ruled,vlined,linesnumbered]{algorithm2e}
\usepackage{optidef}
\usepackage{booktabs}
\usepackage{amsthm}
\usepackage{svg}
\usepackage{hyperref}
\usepackage{cleveref}
\usepackage{dsfont}
\usepackage{graphics}
\usepackage{float}
\usepackage{booktabs}          
\newtheorem{lemma}{Lemma}
\newtheorem{example}{Example}

\newtheorem{definition}{Definition}

\hypersetup{
    colorlinks=true,
    linkcolor=blue,
    filecolor=green,      
    urlcolor=magenta,
    citecolor=teal,
}

\def\volumeyear{2025}
\let\cite\citep
\begin{document}
\runninghead{Suh and Pang et al.}
\title{Dexterous Contact-Rich Manipulation via the Contact Trust Region}

\author{
H.J. Terry Suh\affilnum{*, 1}, 
Tao Pang\affilnum{*, 2}, 
Tong Zhao\affilnum{2},
and Russ Tedrake\affilnum{1} 
\affiliation{%
  \affilnum{1} Massachusetts Institute of Technology, USA. \\
  \affilnum{2} RAI Institute (formerly Boston Dynamics AI Institute), USA.\\
  * These authors contributed equally to this work.}
\corrauth{Tao Pang, RAI Institute, USA.}
\email{pangtao@csail.mit.edu}
}

\input{commands}

\begin{abstract}
\input{body/abstract.tex}

\end{abstract}
\maketitle

\input{body/introduction}

\input{body/list_of_symbols}

\input{body/prelim}
\input{body/motion-sets-new}
\input{body/inverse-dynamics}
\input{body/mpc-qsim-results}

\input{body/mpc-hardware-results}

\input{body/grasp-sampling}
\input{body/prm}

\input{body/conclusion}
\input{body/acknowledgement}

\bibliographystyle{SageH}
\bibliography{references.bib}

\appendix
\input{body/appendix}
\end{document}

%% file: commands.tex
\newcommand{\dt}{h}
\newcommand{\qu}{{q^\mathrm{o}}}
\newcommand{\qa}{{q^\mathrm{a}}}
\newcommand{\plus}{{\texttt{+}}}
\newcommand{\qaplus}{{q^\mathrm{a}_\plus}}
\newcommand{\quplus}{{q^\mathrm{o}_\plus}}
\newcommand{\qplus}{{q_\plus}}
\newcommand{\qps}{{q_\plus^\star}}
\newcommand{\qbps}{{\bar{q}_{\plus}^\star}}
\newcommand{\qbpu}{{\bar{q}_{\plus}^\mathrm{o}}}
\newcommand{\qbpss}{{\bar{q}_{\plus\kappa}^{\star}}}
\newcommand{\qhpss}{{\hat{q}_{\plus\kappa}^{\star}}}
\newcommand{\qpss}{{q_\plus^{\star\kappa}}}
\newcommand{\lambss}[1][]{{\lambda^{\star}_{{\kappa}_{#1}}}}
\newcommand{\lambssb}[1][]{{\bar{\lambda}^{\star}_{{\kappa}_{#1}}}}
\newcommand{\lambssh}[1][]{{\hat{\lambda}^{\star}_{{\kappa}_{#1}}}}
\newcommand{\qacmd}{{q^\mathrm{a}_\mathrm{cmd}}}
\newcommand{\dqu}{h{\delta q^\mathrm{o}}}
\newcommand{\dqa}{h{\delta q^\mathrm{a}}}
\newcommand{\vU}{{v^\mathrm{o}}}
\newcommand{\vA}{{v^\mathrm{a}}}
\newcommand{\du}{{\delta u}}
\newcommand{\dq}{{\delta q}}
\newcommand{\Ka}{\mathbf{K}_\mathrm{a}}
\newcommand{\ka}{k_\mathrm{a}}
\newcommand{\Mu}{\mathbf{M}_\mathrm{o}}
\newcommand{\Ba}{\mathbf{B}^\mathrm{a}}
\newcommand{\Bu}{\mathbf{B}^\mathrm{o}}
\newcommand{\BuBundle}{\mathbf{\hat{B}}_\mathrm{o}}
\newcommand{\tauU}{\tau^\mathrm{o}}
\newcommand{\tauA}{\tau^\mathrm{a}}
\newcommand{\J}{\mathbf{J}}
\newcommand{\Ju}[1][]{\mathbf{J}_{\mathrm{o}_{#1}}}
\newcommand{\Ja}[1][]{\mathbf{J}_{\mathrm{a}_{#1}}}
\newcommand{\Jn}[1][]{\mathbf{J}_{\mathrm{n}_{#1}}}
\newcommand{\Jt}[1][]{\mathbf{J}_{\mathrm{t}_{#1}}}
\newcommand{\JuActive}{\tilde{\mathbf{J}}_\mathrm{o}}
\newcommand{\JaActive}{\tilde{\mathbf{J}}_\mathrm{a}}
\newcommand{\R}[1][]{\mathbb{R}^{#1}}
\newcommand{\nU}{{n_{q_\mathrm{o}}}}
\newcommand{\nA}{n_{q_\mathrm{a}}}
\newcommand{\nC}{n_\mathrm{c}}
\newcommand{\nCG}{n_\mathrm{cg}}
\newcommand{\RbarU}[1]{\bar{\mathcal{R}}^\mathrm{u}_{#1}}
\newcommand{\norm}[1]{\left\lVert{#1}\right\rVert}
\newcommand{\SigmaInverse}{\hat{\mathbf{\Sigma}}^{-1}}
\newcommand{\distanceU}{\hat{d}^\mathrm{u}}
\newcommand{\qNominalU}{\bar{q}^\mathrm{o}}
\newcommand{\qNominalA}{\bar{q}^\mathrm{a}}
\newcommand{\qNominalACmd}{\bar{q}^\mathrm{a}_\mathrm{cmd}}
\newcommand{\muHatU}{\hat{\mu}^\mathrm{u}}
\newcommand{\xgoal}{x_\text{goal}}
\newcommand{\qugoal}{q^\mathrm{o}_g}

\newcommand{\minimize}{\mathrm{min}.}
\newcommand{\DfDx}[2]{\frac{\partial {#1}}{\partial {#2}}}
\newcommand{\DfDxLine}[2]{\partial {#1} / \partial {#2}}
\newcommand{\A}{\mathbf{A}}
\newcommand{\B}{\mathbf{B}}
\newcommand{\I}{\mathbf{I}}
\newcommand{\code}[1]{{\fontfamily{cmss}\selectfont {#1}}}

\newcommand{\Nearest}{\mathtt{Nearest}}
\newcommand{\Extend}{\mathtt{Extend}}
\newcommand{\ContactSample}{\mathtt{ContactSample}}
\newcommand{\Rue}{\mathcal{R}^\mathrm{u}_\varepsilon}

\newcommand{\russ}[1]{\textcolor{teal}{russ: {#1}}}
\newcommand{\pang}[1]{\textcolor{red}{pang: {#1}}}
\newcommand{\terry}[1]{\textcolor{blue}{terry: {#1}}}
\newcommand{\tong}[1]{\textcolor{purple}{tong: {#1}}}

%% file: body/abstract.tex
What is a good local description of contact dynamics for contact-rich manipulation, and where can we trust this local description? While many approaches often rely on the Taylor approximation of dynamics with an ellipsoidal trust region, we argue that such approaches are fundamentally inconsistent with the unilateral nature of contact. As a remedy, we present the Contact Trust Region (CTR), which captures the unilateral nature of contact while remaining efficient for computation. 
With CTR, we first develop a Model-Predictive Control (MPC) algorithm capable of synthesizing local contact-rich plans.
Then, we extend this capability to plan globally by stitching together local MPC plans, enabling efficient and dexterous contact-rich manipulation.
To verify the performance of our method, we perform comprehensive evaluations, both in high-fidelity simulation and on hardware, on two contact-rich systems: a planar \code{IiwaBimanual} system and a 3D \code{AllegroHand} system. 
On both systems, our method offers a significantly lower-compute alternative to existing RL-based approaches to contact-rich manipulation.
In particular, our Allegro in-hand manipulation policy, in the form of a roadmap, takes fewer than 10 minutes to build offline on a standard laptop \emph{using just its CPU}, with online inference taking just a few seconds.
Experiment data, video and code are available at \href{ctr.theaiinstitute.com}{\color{magenta} \code{ctr.theaiinstitute.com}}.

%% file: body/introduction.tex
\section{Introduction}
Robots today rarely leverage their embodiment to the fullest due to the limitations of our computational algorithms: robot arms only establish contact with the end-effectors and only perform collision-free motion planning, and robot hands often only establish contact with the fingertips instead of leveraging the entire surface of the hand. This stands in stark contrast to humans, as we are able to utilize every part of our body to strategically establish contact with the environment. In order to address this gap, dexterous contact-rich manipulation, where a robot must autonomously decide where to establish contact without restricting possible contacts, remains an important problem for us to solve.

At the heart of many iterative algorithms for manipulation lies the question: what is a good local description of contact mechanics for contact-rich manipulation that (i) faithfully captures local behavior, and (ii) is sufficiently simple for efficient and scalable computation? 
Many classical works have considered how contact forces, subject to friction cone constraints, can locally affect changes in manipuland configurations via the contact Jacobian \citep{craig, mls, mason-salisbury}. 
These mechanics have been used heavily as local representations of contact for planning and control in subsequent works. For instance, classical grasping analysis and synthesis relies on local wrench-space arguments \cite{robotics-handbook-grasping, han-grasp-analysis}. 
More recently, many contact-implicit trajectory optimization (CITO) formulations have considered both friction cone and the unilateral nature of contact via complementarity constraints \cite{yanran, posa, wensing, hank2025cito}.
When iteratively solving the resulting optimization problems, the constraints are linearized, revealing similar contact Jacobians as their classical counterparts. 

At surface level, these approaches seem different from those that use a differentiable simulator \cite{diffsim, pangsimulator, gqdp, dojo, brax, warp, drake} to build Taylor approximations of the dynamics, which replace the explicit consideration of friction cone, non-penetration and complementarity constraints with a simple linear equality constraint.
Previous works have utilized this locally-linear first-order Taylor approximation for the purposes of trajectory optimization and control \cite{yuval, optopt, bundled, gqdp, kurtzilqr, kong}.

\begin{figure}[t]
\centering\includegraphics[width = 0.48\textwidth]{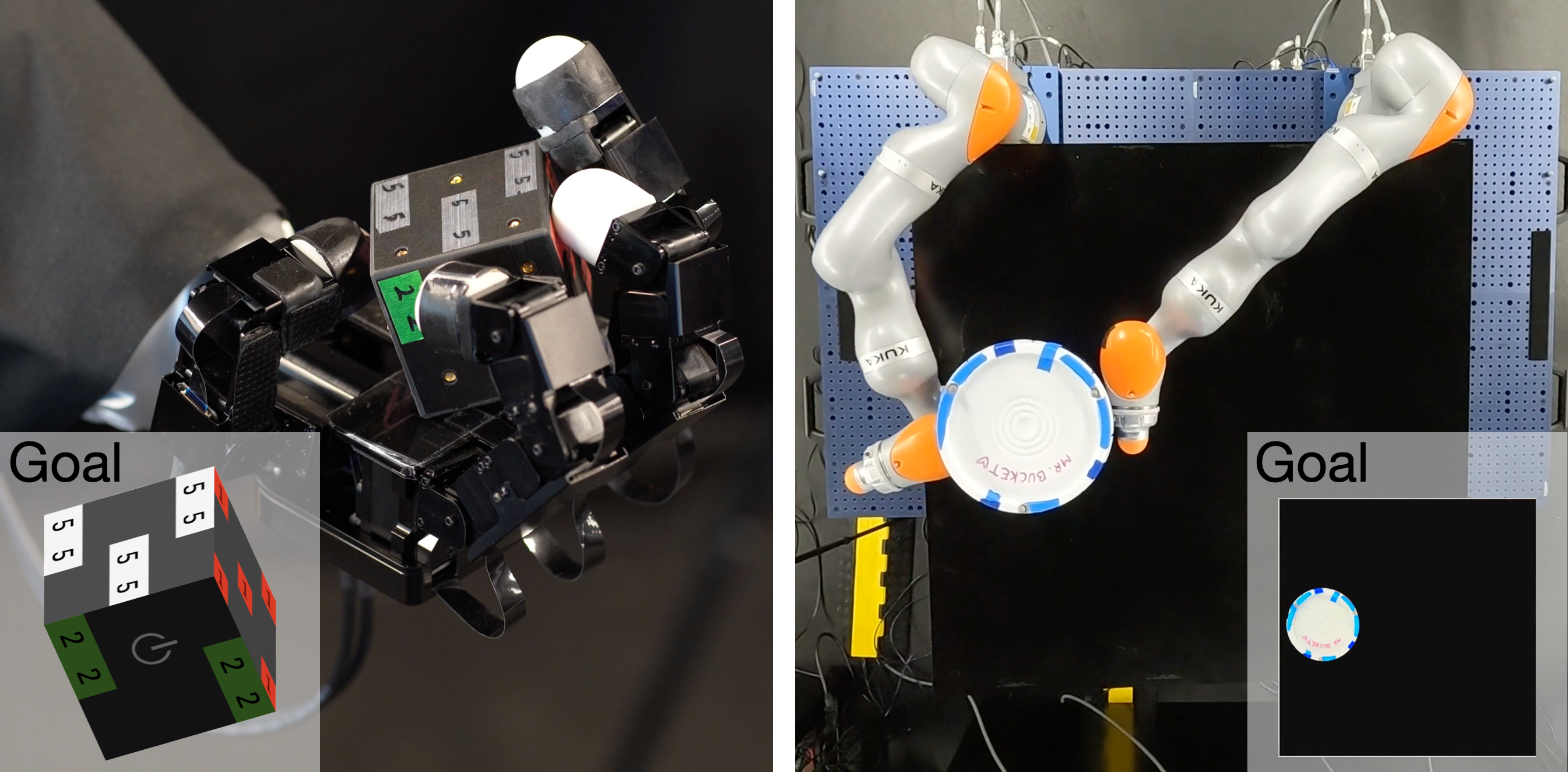}
\caption{Hardware experiments illustrating the utility of our proposed method in contact-rich manipulation. Left: Dexterous in-hand manipulation with the Allegro hand moving a cube. Right: Whole-body manipulation with bimanual iiwas moving a bucket.} 
\label{fig:banner}
\vskip -0.2 true in
\end{figure}

The apparent discrepancy between these two approaches raises the question of how they are related. 
In fact, analytically deriving the gradients from differentiable simulators shows that the Taylor expansions they produce and the complementarity constraints ubiquitous in contact-implicit trajectory optimization are fundamentally connected \cite[Example 5]{gqdp}:
the Taylor expansions are obtained by applying the Implicit Function Theorem to the force-balance and complementarity conditions, which form a \emph{subset} of the contact dynamics constraints typically included in CITO formulations \cite{fiacco1983introduction, giorgi2018tutorial, le2023single}.

However, this connection hides a deeper question: \emph{where can we trust the Taylor expansions}? In optimization, this region is known as the \emph{trust region} \cite{trust-region}.
Intuitively, the trust region describes where the local model closely approximates the original function, making it reliable for local improvements. As the quality of Taylor approximations typically degrades as we move further from the nominal point, previous works have often employed an \emph{ellipsoidal trust region} (ETR) \cite{trust-region, nocedal-wright, trust-region-generalization}.

Our first contribution (\Cref{sec:local-approximation}) is to reveal the inconsistency between a standard, \emph{symmetric} ETR and the \emph{unilateral} nature of contact. We address this by proposing the Contact Trust Region (CTR), which consists of a collection of \emph{convex} constraints that correct this flaw. The key insight is that while the Implicit Function Theorem uses a subset of contact constraints to derive the Taylor expansion, the remaining feasibility constraints (e.g., friction) are merely ignored. The CTR re-incorporates these crucial, omitted constraints, augmenting the Taylor expansion to correctly characterize the local valid region. Not only does this connect the Taylor expansion to classical concepts like the wrench set, but it can also be readily derived from standard components of existing differentiable simulators. Furthermore, despite its convexity, we show that CTR is locally equivalent to the \emph{full} set of smoothed contact dynamics constraint commonly used in CITO, allowing us to circumvent the numerically challenging complementarity constraints \cite{schwartz} when modeling contact.

To build intuition, consider a simple example of a spherical robot pushing a box across a table (see \Cref{fig:1d-pushing-schematic} for a illustration of this system). 
At a state where the robot is touching the box, an ETR, due to its symmetric nature, would locally suggest a ``pulling'' action is as plausible as a ``pushing'' action.
However, pulling on the object is physically impossible due to the nature of contact forces.
The core idea of CTR is to build the physical constraint of contact directly into the trust region representation, pruning away physically infeasible actions like pulling. 

In this work, we develop CTR based on the quasidynamic differentiable simulator proposed in \citet{gqdp}, called the Convex Quasidynamic Differentiable Contact (CQDC) contact model.
The CQDC model, recapped in \Cref{sec:cqdc}, is very representative of approaches taken by modern differentiable simulators through contact \cite{gqdp, dojo, diffsimzhong, werling2021fast}, which makes extending CTR to full second-order dynamics conceptually straightforward.
The CQDC model treats simulation of contact as a convex optimization problem whose primal solution becomes the next state $x_+$, and gives gradients by performing sensitivity analysis in convex optimization. 

However, while a bulk of previous approaches \cite{bundled, gqdp, optopt, kurtzilqr} have only focused on building a Taylor approximation of the \emph{configurations} based on sensitivity gradients, our CTR also builds a linear model over the \emph{contact forces} by utilizing sensitivity analysis to obtain gradients with respect to dual variables. A linear model over the contact forces can further qualify the ellipsoidal trust region by imposing friction cone constraints which are often present in works that directly encode contact complementarity constraints.

The CTR and its connection to convex formulations of contact dynamics also allows us to seamlessly integrate the advances in contact smoothing for planning and control \cite{dojo, bundled, gqdp, posa, aykut}. Without contact smoothing, many classical methods as well as local trajectory optimizers are limited to fixed contact-mode sequences. Thus, their ability to describe contact mechanics locally has been limited to the relationship between possible contact forces and changes in state \emph{within} a contact mode. The myopicness of this local model is inconsistent with the \emph{non-smooth} nature of contact, which also spells trouble for gradient-based Taylor approximations of contact dynamics. 

While complementarity-based representations \cite{alp-lcs} and mixed-integer formulations \cite{tobia-recovery, marcucci2019mixed, marcucci2024shortest, graesdal2024towards} attempt to alleviate this limitation by considering contact mode changes globally, optimizing through these models often require searching through all possible contact-mode sequences. Given the near-exponential scalability of contact modes \cite{huang-mode-enumeration}, these methods have struggled to tackle problems that are rich in contact. Thus, many successful algorithms go through an approximate representation in which complementarity is omitted in practice \cite{alp-admm, mujoco-convex-invertible, jin-complementarity-free, kurtz2023inverse}. 

To alleviate these limitations, smoothing-based methods relax contact dynamics by reasoning about the average result of all the contact modes around a nominal state given some distribution of state and action disturbances. Recent works \cite{bundled,diffsim} have attributed the success and scalability of Reinforcement Learning (RL) to this mechanism. Surprisingly, however, \citet{gqdp, shenao-abs, dojo} show that for contact dynamics simulated by convex optimization, this process can be done without Monte-Carlo by using \emph{barrier smoothing} which performs interior-point relaxation of the original contact dynamics and produces force from a distance. However, despite occasional successes, naively applying classical control methods such as LQR on smoothed linearization of contact dynamics often leads to catastrophic failures \cite{shirai2025linear}.

CTR advances the discussion on smoothing by indicating where the smoothed dynamics remain reliable. As our approach draws directly on convex optimization concepts, particularly the sensitivity analysis of primal and dual solutions, we can enjoy the same smoothing benefits by applying an interior-point relaxation that perturbs the original problem’s Karush-Kuhn-Tucker (KKT) optimality conditions.

As our second contribution, we present a highly efficient, contact-implicit Model Predictive Control (MPC) method \cite{simon-fastmpc, alp-admm} in \Cref{sec:local-planning-control}. Specialized for contact-rich manipulation, it can be viewed as a natural extension of iterative LQR \cite{li2004iterative}: we leverage Taylor expansions from CQDC dynamics as linear dynamics constraints, and we include the CTR as additional constraints to capture local contact dynamics. By formulating the CTR as a convex set—specifically, an intersection of multiple second-order cone constraints—each iteration of the proposed MPC remains a convex optimization problem.

We show the efficacy of the proposed CTR and MPC method on two representative contact-rich manipulation problems in \Cref{fig:banner}: (i) whole-body manipulation on the bimanual iiwa and cylinder system, and (ii) in-hand reorientation on the Allegro hand and cube system. For both systems, we aim to answer the following two questions: (i) is the MPC method successful in planning for goals under quasidynamic equations of motion (\Cref{sec:local-planning-results}), and (ii) can the MPC method successfully stabilize quasidynamic plans under true second-order dynamics (\Cref{sec:stabilization-results})? Through $2000$ trajectory runs in simulation and $100$ runs on hardware, our results delve into the efficacy and the limitations of our approach, as well as the benefits of CTR over the ellipsoidal one in the context of planning and control.

While our MPC is highly effective in local planning, the goals it can achieve are fundamentally limited by local reachability. For further goals that require non-trivial exploration, its capacity is limited in the absence of a good initial guess. Furthermore, when the MPC encounters such goals, its performance is relatively hard to characterize. 


As our final contribution, we address global planning by chaining local trajectories discovered by MPC. To do this, we follow a roadmap \cite{kavraki-prm} approach where we seed each node with a stable object configuration. Then, we connect these nodes by first finding good actuator configurations for local MPC to start from using CTR (\Cref{sec:actuator-placement}), then rolling out the MPC plan. To demonstrate the constructed roadmap's robustness, we perform on the Allegro hand system $150$ consecutive edge traversals before the hardware shuts down due to overheating. For new goals that are not in the roadmap, we simply find the closest node to the goal, then connect them by the same process of finding a good initial actuator configuration and then rolling out MPC.

Our proposed global planning method (\Cref{sec:roadmap}) enables efficient dexterous re-orientation on the Allegro hand, and has several advantages compared to existing approaches. Compared to single-query algorithms \cite{xianyi, hidex, gqdp, khandate}, most of the computation time for our algorithm happens offline in the roadmap-building stage, enabling fast online inference. Moreover, the offline roadmap construction only takes minutes on a standard laptop, which is orders-of-magnitude less than approaches based on deep RL\cite{openai-hand-demo, taojiepulkit, nvidia2022dextreme}. 
%

Through our three contributions of (i) the contact trust region that approximates contact mechanics efficiently (\Cref{sec:cqdc},\labelcref{sec:local-approximation}), (ii) a highly effective gradient-based MPC controller specialized to local contact-rich manipulation (\Cref{sec:local-planning-control}-\labelcref{sec:stabilization-results}), and (iii) a global planner that stitches together local trajectories (\Cref{sec:actuator-placement},\labelcref{sec:roadmap}), 
our method succeeds in dramatically reducing computation time compared to state-of-the-art methods in RL \cite{openai-hand-demo, taojiepulkit, nvidia2022dextreme} or sampling-based MPC methods \cite{mjpc, li2024_drop} without the need for parallelization. We hope digging deep into the structure of contact mechanics can bring efficiency to the current state-of-the-art approaches for whole-body manipulation and dexterous in-hand reorientation.

%% file: body/list_of_symbols.tex
\begin{table}[h]
\centering
\small
\setlength{\tabcolsep}{3pt}
\caption{List of Symbols}
\begin{tabular}{lcl}
\toprule
{\footnotesize Symbol} & {\footnotesize Shape} & {\footnotesize Meaning} \\
\midrule
$n_q$ & $\mathbb{N}$ & {\footnotesize number of positions in a system} \\
$n_{q^\mathrm{a}}$ & $\mathbb{N}$ & {\footnotesize  number of actuated positions }\\
$n_{q^\mathrm{o}}$ & $\mathbb{N}$ & {\footnotesize number of object positions } \\
$q$ & $\R[n_q]$ & {\footnotesize system configuration vector,  $q \coloneqq (\qa, \qu)$} \\
$\qa$ & $\R[n_{q_\mathrm{a}}]$ & {\footnotesize actuated (robot) configuration vector} \\
$\qu$ & $\R[n_{q_\mathrm{o}}]$ & {\footnotesize object configuration vector} \\
$q_+^{\mathrm{o}/\mathrm{a}}$ & $\R[n_{q_\mathrm{o}} /n_{q_\mathrm{a}}]$ & {\footnotesize object/robot configuration at the next step} \\
$u$ & $\R[n_{q_\mathrm{a}}]$ & {\footnotesize control input} \\
$\lambda_i$ & $\R[3]$ & {\footnotesize contact force at contact $i$, $\lambda_i \coloneqq (\lambda_{\mathrm{n}_i}, \lambda_{\mathrm{t}_i})$} \\
$\lambda_{\mathrm{n}_i}$ & $\R$ & {\footnotesize normal contact force at contact $i$} \\
$\lambda_{\mathrm{t}_i}$ & $\R[2]$ & {\footnotesize tangential contact force at contact $i$} \\
$\phi_i$ & $\R$ & {\footnotesize signed distance at contact $i$, \eqref{eq:friction_constraints:primal_feasibility}} \\
$\tau^{\mathrm{o}/\mathrm{a}}$ & $\R[n_{q_\mathrm{o}} /n_{q_\mathrm{a}}]$ & {\footnotesize non-contact torque acting on object/robot} \\
$h$ & $\R_+$ & {\footnotesize step size in seconds}\\
$\mathbf{K}_\mathrm{a}$ & $\mathbf{S}^{n_{q_\mathrm{a}}}_{++}$ & {\footnotesize robot stiffness matrix} \\
$\mathbf{M}_\mathrm{o}$ & $\mathbf{S}^{n_{q_\mathrm{o}}}_{++}$ & {\footnotesize object mass matrix} \\
$\mathbf{J}_i$ & $\R[3 \times n_{q}]$ & {\footnotesize contact Jacobian at contact $i$, \eqref{eq:contact_jacobian_i}} \\
$\mathcal{K}_i$  & $\R[3]$ & {\footnotesize feasible velocity cone at contact $i$, \eqref{eq:contact_cones:v}} \\
$\mathcal{K}_i^\star$  & $\R[3]$ & {\footnotesize friction cone at contact $i$, \eqref{eq:contact_cones:lambda} }\\ 
$\mu_i$ & $\R_+$ &  {\footnotesize friction coefficient at contact $i$ } \\
$\nu_i$ & $\R[3]$ & {\footnotesize \eqref{eq:friction_constraints:primal_feasibility} }\\
$\mathbf{P}$ & $\mathbf{S}^{n_{q}}_{++}$ & {\footnotesize quadratic cost in contact dynamics, \eqref{eq:contact_dynamics_quadratic_cost}} \\
$b$ & $n_{q}$ & {\footnotesize linear cost in contact dynamics, \eqref{eq:contact_dynamics_linear_cost}} \\
$c_i$ & $\R[3]$ & {\footnotesize \eqref{eq:q_dynamic_socp:c_definition}} \\
$\kappa$ & $\mathbb{R}_+$ & {\footnotesize contact dynamics smoothing parameter} \\
\multirow{2}{*}{$f_{(\kappa)}(\cdot,\cdot)$} & $\R[n_q] \times \R[n_{q_\mathrm{a}}]$  & \multirow{2}{*}{\footnotesize ($\kappa$-smoothed) contact dynamics, (\eqref{eq:q_dynamic_f_kappa})\eqref{eq:q_dynamic_f} }\\
& $\rightarrow \R[n_q]$ &  \\
\midrule
$\mathbf{A}_{(\kappa)}$ & $\R[n_q \times n_q]$ & {\footnotesize $\partial f_{(\kappa)} / \partial q $} \\
$\mathbf{B}_{(\kappa)}$ & $\R[n_q \times n_{q_\mathrm{a}}]$ & {\footnotesize $\partial f_{(\kappa)} / \partial u $} \\
$\mathbf{C}_{(\kappa), i}$ & $\R[n_q \times n_q]$ & {\footnotesize $\partial \lambda_{(\kappa),i} / \partial q $} \\
$\mathbf{D}_{(\kappa), i}$ & $\R[n_q \times n_{q_\mathrm{a}}]$ & {\footnotesize $\partial \lambda_{(\kappa), i} / \partial u $} \\
$\bar{\cdot}$ & N.A. & {\footnotesize nominal quantities, e.g. $\bar{q}$, see \eqref{eq:non_smooth_taylor_expansions}}\\
$\delta{\cdot}$ & N.A. & {\footnotesize deviation from nominal quantities, e.g. $\delta{q}$} \\
$\hat{\cdot}$ & N.A. & {\footnotesize Taylor expansions, e.g. $\hat{q}_+$}\\
$\mathbf{\Sigma}$ & $\mathbf{S}_{++}^{(n_q + n_{q_\mathrm{a}})}$ & {\footnotesize trust region radius} \\
$\mathcal{E}_\mathbf{\Sigma}$ & $\R[n_q+ n_{q_\mathrm{a}}]$ & {\footnotesize ETR with radius $\mathbf{\Sigma}$, \eqref{eq:ellpsoidal_trust_region}}\\
$\mathcal{S}_{\mathbf{\Sigma},\kappa}$ & $\R[n_q+ n_{q_\mathrm{a}}]$ & {\footnotesize CTR with radius $\mathbf{\Sigma}$ and smoothing $\kappa$, \eqref{eq:full_contact_trust_region}}\\
$\mathcal{S}_{\mathbf{\Sigma},\kappa}^\mathcal{A}$ & $\R[n_{q_\mathrm{a}}]$ & {\footnotesize Action-only CTR, \eqref{eq:action-contact-trust-region}}\\
$\tilde{\mathcal{S}}_{\mathbf{\Sigma},\kappa}$ & $\R[n_q+ n_{q_\mathrm{a}}]$ & {\footnotesize Relaxed CTR, \eqref{eq:relaxed_contact_trust_region}}\\
$\tilde{\mathcal{S}}_{\mathbf{\Sigma},\kappa}^\mathcal{A}$ & $\R[n_{q_\mathrm{a}}]$ & {\footnotesize Relaxed Action-only CTR, \eqref{eq:relaxed-action-contact-trust-region}}  \\
$\mathcal{M}_\mathbf{\Sigma,\kappa}^{(\mathcal{A})}$ & $\R[n_q]$ & {\footnotesize (action-only) motion set, (\eqref{eq:action-only-motion_set})\eqref{eq:motion_set}} \\
$\mathcal{C}^\mathcal{A}_{\mathbf{\Sigma},\kappa,i}$ & $\R[3]$ & {\footnotesize contact force set at contact $i$, \eqref{eq:contact_force_set}}\\
$\mathcal{JC}^\mathcal{A}_{\mathbf{\Sigma},\kappa,i}$ & $\R[n_{q_\mathrm{o}}]$ & {\footnotesize generalized friction cone at contact $i$, \eqref{eq:generalized_friction_cone}}\\
$\mathcal{W}^\mathcal{A}_{\mathbf{\Sigma}, \kappa}$ & $\R[n_{q_\mathrm{o}}]$ & {\footnotesize wrench set, \eqref{eq:wrench_set}} \\
$\mathcal{M}^{\mathcal{A},\mathrm{o}}_{\mathbf{\Sigma},\kappa}$ & $\R[n_{q_\mathrm{o}}]$ & {\footnotesize motion set projected to the space of $\qu$, \eqref{eq:motion-sets-classical-derivation}}\\
\midrule
$T$ & $\mathbb{N}$ & {\footnotesize CtrTrajOpt planning horizon in \Cref{alg:trajopt_ctr}} \\
$n_\mathrm{max}$ & $\mathbb{N}$ & {\footnotesize CtrTrajOpt iteration limit in \Cref{alg:trajopt_ctr}}\\
$H$ & $\mathbb{N}$ & {\footnotesize MPC rollout horizon in \Cref{alg:mpc}} \\
$r$ & $\R_+$ & {\footnotesize scalar trust region radius in \Cref{sec:trust_region_radius}} \\
$\mathbf{Q}$ & $\mathbf{S}_+^{n_q}$ & {\footnotesize state tracking quadratic in \eqref{eq:linear_trajopt}} \\
$\mathbf{R}$ & $\mathbf{S}_+^{n_{q_\mathrm{a}}}$ & {\footnotesize control action quadratic cost in \eqref{eq:linear_trajopt}} \\
$N$ & $\mathbb{N}$ & {\footnotesize number of MPC re-plans in \Cref{alg:mpc_real_dynamics}} \\
\bottomrule
\end{tabular}
\vskip -0.2 true in
\end{table}

%% file: body/prelim.tex
\section{Preliminaries: Contact Dynamics as Convex Optimization}\label{sec:cqdc}

Many state-of-art methods for simulating contact solve a constrained optimization problem at every step \cite{stewarttrinkle,anitescu,mujoco,pangsimulator,dojo,sap,gqdp}. In this paper, we use the quasi-dynamic formulation of contact dynamics presented in \cite{gqdp}, where we assume the system is highly damped by frictional forces. 

\subsection{Setup \& Notation} We assume the state consists of configurations, which we denote by $q\in\mathbb{R}^{n_q}$, and omit velocities due to the quasi-dynamic assumption. These configurations are divided into actuated configurations $\qa\in\mathbb{R}^{\nA}$ that belong to the robot, and object configurations $\qu\in\mathbb{R}^{\nU}$ which are unactuated. Furthermore, the actions are represented as a position command $u\in\mathbb{R}^{\nA}$ to a stiffness controller with a diagonal gain matrix $\Ka\in\mathbb{R}^{\nA \times \nA}$.

\subsection{Equations of Motion}
\label{sec:cqdc:equations_of_motion}
At a configuration $q$, the equations of motion are framed as a search for next configurations $q_\texttt{+}=(\qaplus,\quplus)$ such that the following force balance equations holds,
\begin{subequations}\label{eq:balance}
\begin{align}
\Ka (\qaplus - u) & = \tauA + \sum^{n_c}_{i=1}(\Ja[i](q))^\top \lambda_i \label{eq:actuatedbalance}\\
\epsilon\Mu(q) \frac{\quplus - \qu}{\dt^2} & = \tauU + \sum^{n_c}_{i=1}(\Ju[i](q))^\top \lambda_i        \label{eq:unactuatedbalance}
\end{align} 
\end{subequations}
where $h\in\mathbb{R}_{>0}$ is the timestep, $\tauA\in\mathbb{R}^{\nA}$ and $\tauU\in\mathbb{R}^{\nU}$ are external torques acting on actuated and unactuated bodies respectively (e.g. gravity), and $\epsilon\in\mathbb{R}_{\geq 0}$ is a regularization constant. 

In addition, we define an index set $\mathcal{I}_c$ over pairs of contact geometries, which we can obtain from a collision detection algorithm. For each contact pair $i\in\mathcal{I}_c$, $\lambda_i\in\mathbb{R}^3$ corresponds to the average contact force over the timestep $h$ at contact $i$, so that $\dt \lambda_i$ is the net impulse generated at contact $i$ during the timestep. By convention, we denote $\lambda_i = (\lambda_{\mathrm{n}_i},\lambda_{\mathrm{t}_i}$) where $\lambda_{\mathrm{n}_i}\in\mathbb{R}$ corresponds to contact forces along the normal direction of the contact frames, and $\lambda_{\mathrm{t}_i}\in\mathbb{R}^2$ corresponds to the frictional forces along the tangential planes.

We also define the contact Jacobians $\J_i\in\mathbb{R}^{3\times n_q}$ as local mappings from the configurations $q$ to the contact frame. The top row corresponds to the normal direction of the contact frame, where as the bottom two rows orthogonally span the tangential plane. The Jacobians can further be decomposed into mappings into actuated and unactuated DOFs,
\vskip -0.1 true in
\begin{equation}
\label{eq:contact_jacobian_i}
\J_i \coloneqq [\Ju[i], \Ja[i]] \coloneqq 
\begin{bmatrix}
\Jn[i] \\
\Jt[i]
\end{bmatrix}
\in \R[3 \times n_q].
\end{equation}

With this notational setup, \eqref{eq:actuatedbalance} describes that the configuration of the actuated bodies at the next step, $\qaplus$, will be decided such that the sum of the forces must be balanced by the force experienced by the stiffness controller at the next step. Similarly, \eqref{eq:unactuatedbalance} states that the unbalanced forces result in relative displacements of the object. We can write \eqref{eq:balance} more succinctly as 
\begin{subequations}
\label{eq:balance_single_equation}
\begin{align}
&\mathbf{P}(q)\qplus + b(q, u) - \sum_{i=1}^{n_c}\mathbf{J}_i(q)^\top \lambda_i = 0, \; \text{where}\\
&\mathbf{P}(q) \coloneqq \begin{bmatrix} \epsilon\Mu(q)/\dt^2 & 0 \\ 0 & \Ka \end{bmatrix}, \label{eq:contact_dynamics_quadratic_cost}\\
&b(q, u) \coloneqq - \begin{bmatrix} \epsilon\Mu(q)\qu/h^2+\tauU \\ \Ka u  + \tauA \end{bmatrix}. \label{eq:contact_dynamics_linear_cost}
\end{align}
\end{subequations}

\subsection{Contact Constraints}
\label{sec:cqdc:contact_constraints}
Additional constraints further qualify the relationship between the next configurations $\qplus$ and contact forces $\lambda_i$:
i) \emph{non-penetration}: next configuration is non-penetrating;
ii) \emph{friction cone}: $\lambda_i$ must be inside the friction cone;
iii) \emph{complementarity}: contact cannot be applied from a distance, and the direction of friction opposes the movement (i.e. maximum dissipation).

To impose these constraints, we first define the feasible velocity cone $\mathcal{K}_i$, and then introduce its dual cone $\mathcal{K}_i^\star$, which corresponds to the friction cone:
\begin{subequations}
\label{eq:contact_cones}
\begin{align}
\mathcal{K}_i &\coloneqq \left\{\nu_i = (\nu_{\mathrm{n}_i}, \nu_{\mathrm{t}_i}) \in \R[3] | \nu_{\mathrm{n}_i} \geq \mu_i \sqrt{\nu_{\mathrm{t}_i}^\top \nu_{\mathrm{t}_i}} \right\},\label{eq:contact_cones:v} \\
\mathcal{K}_i^\star &\coloneqq \left\{\lambda_i = (\lambda_{\mathrm{n}_i}, \lambda_{\mathrm{t}_i}) \in \R[3] | \mu_i \lambda_{\mathrm{n}_i} \geq \sqrt{\lambda_{\mathrm{t}_i}^\top \lambda_{\mathrm{t}_i}} \right\} \label{eq:contact_cones:lambda}.
\end{align}
\end{subequations}

Anistescu's relaxation \cite{anitescu} allows us to write down these contact constraints in a convex manner by introducing a mild non-physical artifact, where sliding in the tangential plane results in separation in the normal direction,
\begin{subequations}
\label{eq:friction_constraints}
\begin{align}
\nu_i \coloneqq
\J_i(q)(\qplus - q)
+
[\phi_i, 0, 0]^\top
&\in \mathcal{K}_i,\label{eq:friction_constraints:primal_feasibility}\\
\lambda_i &\in \mathcal{K}_i^\star, \label{eq:friction_constraints:dual_feasibility}\\
\nu_i^\top \lambda_i &= 0,\label{eq:friction_constraints:complementary_slackness}
\end{align}
\end{subequations}
where $\phi_i$ is the current signed distance for contact pair $i$.

\subsection{Contact Dynamics as Convex Optimization}
We can frame the equations of motion with contact as a search for the next configurations $\qplus$ that satisfy the constraints introduced in \Cref{sec:cqdc:equations_of_motion} and \Cref{sec:cqdc:contact_constraints},
\begin{subequations}\label{eq:q_dynamics_kkt}
\begin{align}
    \mathrm{find} \quad & q_+ \\
    \text{s.t.} \quad & \eqref{eq:balance_single_equation},\eqref{eq:friction_constraints:primal_feasibility},\eqref{eq:friction_constraints:dual_feasibility},\eqref{eq:friction_constraints:complementary_slackness}.
\end{align}
\end{subequations}

Remarkably, this is equivalent to the KKT conditions of the following Second-Order Cone Program (SOCP), 
\begin{subequations}
\label{eq:q_dynamic_socp}
\begin{align}
&\underset{\qplus}{\minimize} \; \frac{1}{2}\qplus^\top \mathbf{P}(q)\qplus + b(q, u)^\top \qplus, \; \text{subject to} \label{eq:q_dynamic_socp:cost}\\
&\qquad \J_i(q) \qplus
+
c_i(q)
\in \mathcal{K}_i, \; \forall i \in \mathcal{I}_c, \; \text{where} \label{eq:q_dynamic_socp:constraint}\\
& c_i(q) \coloneqq [\phi_i(q), 0, 0]^\top  - \mathbf{J}_i(q) q \label{eq:q_dynamic_socp:c_definition}.
\end{align}
\end{subequations}
where $\mathcal{I}_c$ is the index set over potential contact pairs. Note that stationarity \eqref{eq:balance_single_equation}, primal feasibility \eqref{eq:friction_constraints:primal_feasibility}, dual feasibility \eqref{eq:friction_constraints:dual_feasibility}, and complementary slackness \eqref{eq:friction_constraints:complementary_slackness} are the KKT conditions of this SOCP. For complementary slackness of SOCPs, we note that we do not impose element-wise slackness: the primal and dual vectors can both be non-zero as long as they are orthogonal \cite{boyd2004convex, gqdp}.

In dynamical systems and control theory, it is conventional to represent the dynamic system as a map $f$ that takes a state and input then maps it to the next state. As our state only consists of configurations $q$ for quasidynamic systems, we use the following notation for contact dynamics,
\begin{subequations}
\label{eq:q_dynamic_f}
\begin{align}
    q_+ = &  f(q,u) \\
       =  & \text{argmin}_{q_+}\;\; \frac{1}{2}q_+^\top \mathbf{P}(q)q_+ + b(q, u)^\top q_+, \nonumber \\
         &  \text{subject to}\;\;\mathbf{J}_i(q) q_+ + c_i(q) \in \mathcal{K}_i, \forall i \in \mathcal{I}_c.
\end{align}
Due to the quasidynamic assumption, velocity is omitted from the state, which only consists of the system configuration $q$.
We sometimes omit the explicit dependence of $\mathbf{P}, b,\mathbf{J}_i$, and $c_i$ on $(q, u)$ when it is clear from context.

\end{subequations}

\subsection{Smoothing of Contact Dynamics}
Despite being continuous and piecewise smooth, the contact dynamics map \eqref{eq:q_dynamic_f} is well known to have discontinuous gradients \cite{posa}, complicating the application of planning and control methods that rely on local smooth approximations. 
More specifically, gradient-based methods can be destabilized by fast-changing gradients induced by quickly switching between smooth pieces of the contact dynamics (also called \emph{contact modes}) \cite{bundled}.
In addition, the gradient landscape can have flat regions if none of the contact modes are active, leaving optimizers stranded without informative directions of improvement \cite{diffsim}.
As such, various dynamic smoothing methods have been proposed to relax this numerical difficulty \cite{dojo,gqdp,bundled,posa,aykut, quentin}.


In \citet{dojo, gqdp}, a method of smoothing out optimization-based dynamics was introduced based on the log-barrier (interior-point) relaxation of \eqref{eq:q_dynamic_socp}, which solves for the KKT equations where the complementarity condition is perturbed by a positive constant $\kappa$
\footnote{The constant 2 on the RHS of \eqref{eq:perturbed-kkt} is due to the squares inside the $\log$ in \eqref{eq:q_dynamics_log}. For more details, see the degree of generalized logarithms \cite[\S 11.6]{boyd2004convex}.
},
\begin{equation}\label{eq:perturbed-kkt}
\lambda_{\kappa, i}^\top \nu_i = 2\kappa^{-1}.
\end{equation}
The resulting solution is equivalent to solving the log-barrier relaxation of $\eqref{eq:q_dynamic_socp}$,
\begin{equation}\label{eq:q_dynamics_log}
\min_{q_+}\;\frac{1}{2}q_+^\top \mathbf{P}q_+ + b^\top q_+ - \frac{1}{\kappa} \sum^{n_c}_{i=1}\log\left(\nu_{\mathrm{n}_i}^2/\mu_i^2-\|\nu_{\mathrm{t}_i}\|^2\right),
\end{equation}
where $\nu_i=\mathbf{J}_i q_+ + c_i$, and $\nu_i = (\nu_{\mathrm{n}_i}, \nu_{\mathrm{t}_i})$ with $\nu_{\mathrm{n}_i} \in \mathbb{R}$ and $\nu_{\mathrm{t}_i} \in \R[2]$. The resulting dynamics creates a force-field effect where bodies that are not in contact apply forces inversely proportional to their distance \eqref{eq:perturbed-kkt}. 

As the optimization \eqref{eq:q_dynamics_log} is an unconstrained convex program, its optimality conditions can be obtained by simply setting the gradient of the cost function w.r.t. $q_+$ to $0$:
\begin{equation}
\label{eq:unconstrained_staionarity}
\mathbf{P} q_+ + b  
- \sum_{i=1}^{\nC} \frac{2\kappa^{-1}}{\nu_{\mathrm{n}_i}^2/\mu_i^2 - \|\nu_{\mathrm{t}_i}\|^2} \mathbf{J}_{i}^\top
\begin{bmatrix}
\nu_{\mathrm{n}_i}/\mu_i^2 \\
-\nu_{\mathrm{t}_i} \\
\end{bmatrix}
= 0.
\end{equation}

As an unconstrained optimization problem, \eqref{eq:q_dynamics_log} technically does not have dual variables. However, if we consider solving the SOCP dynamics \eqref{eq:q_dynamic_socp} with the interior point method, \eqref{eq:q_dynamics_log} corresponds to one major iteration along the \emph{central path} \cite[\S11.6]{boyd2004convex}. Therefore, we can identify average contact forces with the \emph{dual feasible points} of \eqref{eq:q_dynamic_socp} along the central path:
\begin{equation}
\label{eq:dual_feasible_point}
\lambda_{\kappa,i} = \frac{2\kappa^{-1}}{\nu_{\mathrm{n}_i}^2 / \mu_i^2-\|\nu_{\mathrm{t}_i}\|^2}
\begin{bmatrix}
    \nu_{\mathrm{n}_i} /\mu_i^2 \\
    -\nu_{\mathrm{t}_i}
\end{bmatrix}.
\end{equation}
We note that this force (i) has a direction opposing the movement in the tangential plane, (ii) has a magnitude that satisfies \eqref{eq:perturbed-kkt}, and iii) lies in the friction cone. 
In addition, plugging \eqref{eq:dual_feasible_point} into \eqref{eq:unconstrained_staionarity} yields
\begin{equation}
\label{eq:unconstrained_staionarity_with_lambda}
\mathbf{P} q_+ + b -\sum_{i=1}^{n_c} \J_i^\top \lambda_{\kappa, i} = 0,
\end{equation}
which is equivalent to the force balance equation \eqref{eq:balance_single_equation}.

Similar to \eqref{eq:q_dynamic_f}, we denote the dynamics map over the relaxed contact dynamics throughout the manuscript as 
\begin{subequations}
\label{eq:q_dynamic_f_kappa}
\begin{align}
q_+ & = f_\kappa(q,u) \\
    & = \text{argmin}_{q_+} \Big(\frac{1}{2}q^\top_+ \mathbf{P} q_+ + b^\top q_+ \nonumber \\
    &\quad \quad \quad \quad \quad - \frac{1}{\kappa} \sum^{n_c}_{i=1}\log\left(\nu^2_{\mathrm{n}_i} / \mu_i^2 - \|\nu_{\mathrm{t}_i}\|^2\right)\Big).
\end{align}
\end{subequations}

\subsection{Sensitivity Analysis of Smoothed Contact Dynamics}
The dynamics map \eqref{eq:q_dynamic_f_kappa} can be interpreted as a \emph{solution map} from the change in the state and input pair $(q, u)$ to the optimal solutions of the optimization program \eqref{eq:q_dynamics_log}. As various first-order methods in numerical optimization and optimal control often utilize the gradients of the dynamics map $f$ with respect to $q$ and $u$, we detail the computation of these gradients in this section. 

The gradients can be computed through \emph{sensitivity analysis}, which states that as we locally perturb the problem parameters, the solution of \eqref{eq:q_dynamic_f_kappa} change in a way that preserves (up to first-order) the optimality conditions of \eqref{eq:q_dynamic_f_kappa}.
This is equivalent to stating that the derivative of the equality constraints with respect to $q$ and $u$ are also zero at optimality. For instance, rewriting the derivative of \eqref{eq:unconstrained_staionarity_with_lambda} and \eqref{eq:perturbed-kkt} in matrix form ($\nu_i$ and $c_i$ need to be expanded using their definitions in \eqref{eq:friction_constraints:primal_feasibility} and \eqref{eq:q_dynamic_socp:c_definition}, respectively) gives  
\begin{align}\label{eq:kkt-sensitivity}
    \begin{bmatrix}
        \mathbf{P} & \mathcal{H}_i\left[-\mathbf{J}_i^\top\right] \\
        \mathcal{V}_i\left[\lambda^\top_i\mathbf{J}_i\right] & \mathcal{D}_i\left[\left(\mathbf{J}_iq_+ + c_i\right)^\top\right]
    \end{bmatrix}
    \begin{bmatrix}
        \DfDx{q_+}{u} \\
        \mathcal{V}_i\left[\DfDx{\lambda_{\kappa,i}}{u} \right]
    \end{bmatrix}
    =  
    \begin{bmatrix}
    -\DfDx{b}{u} \\
        0
    \end{bmatrix}
\end{align}
where $\mathcal{H}_i,\mathcal{V}_i,\mathcal{D}_i$ stands for horizontal, vertical, and diagonal stacking of the terms from all contact pairs. Then, we can use the implicit function theorem to solve the system of equations and obtain the derivatives $\partial q_+/\partial u$ and $\partial\lambda_i/\partial u$ (more details in \Cref{proof:kkt-sensitivity}). The derivatives with respect to $q$ may be obtained in a similar fashion, but are more involved due to the $\partial\mathbf{J}_i/\partial q$ term (curvature).

%% file: body/motion-sets-new.tex
\section{Local Approximation of Contact Dynamics}\label{sec:local-approximation}

Consider a simple, single-horizon optimization problem that involves the contact dynamics \eqref{eq:q_dynamic_socp} as a constraint,
\begin{subequations} \label{eq:inverse_dynamics}
\begin{align}
    \min_{q_+,u} \quad& \|q_\mathrm{goal} - q_+\|^2 \\
    \text{s.t.} \quad& q_+ = f(q, u). \label{eq:inverse_dynamics:dynamics_constraint}
\end{align}
\end{subequations}

A majority of first-order algorithms for optimization and optimal control rely on iteratively making a \emph{local} approximation of the problem that is more amenable for computation. For instance, gradient-descent makes a linear approximation, and Newton's method makes a quadratic one. However, we rarely handle the case where the constraint itself is the result of an optimization.
In this section, we answer the question: what is the right local approximation of the contact dynamics that (i) is computationally convenient, and (ii) captures the correct local behavior?

\subsection{Linear Model over Smoothed Primal and Dual Variables}
Previous works on smooth dynamical systems have often resorted to a first-order Taylor approximation of the dynamics around some nominal coordinates $(\bar{q},\bar{u})$, and approximate the dynamics with a linear model. This Taylor approximation can be compactly written as
\begin{subequations}
\label{eq:non_smooth_taylor_expansions}
\begin{align}
    \hat{q}_+ & = \mathbf{A}\delta q + \mathbf{B}\delta u + f(\bar{q},\bar{u}), \\
    \mathbf{A} & \coloneqq \partial f/\partial q (\bar{q},\bar{u}),\;\;\mathbf{B}\coloneqq \partial f/\partial u(\bar{q},\bar{u}), \\
    \delta q & \coloneqq q - \bar{q}, \;\; \delta u \coloneqq u - \bar{u}.
\end{align}
\end{subequations}

However, many previous works \cite{diffsim, dojo, gqdp, shirai2025linear, posa} have noted that due to the non-smooth nature of contact, it is beneficial to smooth the dynamics where a first-order Taylor approximation would be valid beyond the contact mode to which $(\bar{q}, \bar{u})$ belongs \cite{gqdp}. A first-order Taylor expansion under such a smoothed dynamics model can be written as 
\begin{subequations}
\begin{align}
    \hat{q}_+ & = \mathbf{A}_\kappa\delta q + \mathbf{B}_\kappa\delta u + f_\kappa(\bar{q},\bar{u}), \\
    \mathbf{A}_\kappa & \coloneqq \partial f_\kappa/\partial q (\bar{q},\bar{u}),\;\;\mathbf{B}_\kappa\coloneqq \partial f_\kappa/\partial u(\bar{q},\bar{u}),
\end{align}
\end{subequations}
where $\kappa$ is the barrier smoothing parameter introduced in \eqref{eq:perturbed-kkt}.

Still, as this linear model only gives a local approximation, we would naturally expect the quality of this approximation to degrade as we get further from the nominal coordinates $(\bar{q},\bar{u})$. 
In order to more rigorously characterize the validity of the local linear model, we need to analyze the behavior of $(q_+,\lambda_{+,i})$ as we vary $(\delta q, \delta u)$, where $\lambda_{+,i}$ are the linearized dual variables defined by:
\begin{subequations}
\begin{align}
    \hat{\lambda}_{+,i} & = \mathbf{C}_{\kappa,i}\delta q + \mathbf{D}_{\kappa,i}\delta u + \lambda_{\kappa,i}(\bar{q},\bar{u}), \\
    \mathbf{C}_{\kappa,i} & \coloneqq \partial \lambda_{\kappa,i}/\partial q (\bar{q},\bar{u}),\;\;\mathbf{D}_{\kappa,i}\coloneqq \partial \lambda_{\kappa,i}/\partial u(\bar{q},\bar{u}).
\end{align}
\end{subequations}

\begin{lemma}[\bfseries Taylor Approximation]\label{lemma:taylor-approximation}\normalfont
    Consider the joint linear model of the primal and dual variables,
    \begin{align}
        \begin{bmatrix} \hat{q}_+ \\ \mathcal{V}_i[\hat{\lambda}_{+,i}] \end{bmatrix} = \begin{bmatrix} \mathbf{A}_\kappa & \mathbf{B}_\kappa \\ \mathcal{V}_i[\mathbf{C}_{\kappa,i}] & \mathcal{V}_i[\mathbf{D}_{\kappa,i}] \end{bmatrix}\begin{bmatrix} \delta q \\ \delta u \end{bmatrix} + \begin{bmatrix} f_\kappa(\bar{q},\bar{u}) \\ \mathcal{V}_i[\lambda_{\kappa,i}(\bar{q},\bar{u})] \end{bmatrix}.
    \end{align}
    Then, this linear model satisfies \eqref{eq:unconstrained_staionarity_with_lambda} and \eqref{eq:perturbed-kkt}, the optimality conditions of the perturbed SOCP \eqref{eq:q_dynamics_log}, to first order:
    \begin{align}
        \begin{bmatrix}
        \hat{\mathbf{P}} \hat{q}_+ + \hat{b} - \sum^{n_c}_{i=1} \hat{\mathbf{J}}_i^\top \hat{\lambda}_{+,i} \\
        \mathcal{V}_i\left[(\hat{\mathbf{J}}_i \hat{q}_+ + \hat{c}_i)^\top \hat{\lambda}_{+,i} - 2\kappa^{-1}\right]
        \end{bmatrix} = \mathcal{O}\left((\delta q,\delta u)^2\right)
    \end{align}
    with $\hat{\mathbf{P}}\coloneqq \mathbf{P}(\bar{q},\bar{u}) + \frac{\partial \mathbf{P}}{\partial q}\delta q + \frac{\partial \mathbf{P}}{\partial u}\delta u$, and similarly for $\hat{b}$, $\hat{\mathbf{J}}$, and $\hat{c}_i$. 
\end{lemma}
\begin{proof}
    If we expand the Taylor-approximated equations and drop the terms above first-order, we are left with (i) the original equality conditions over nominal coordinates which must hold due to the nominal values being obtained at optimality and (ii) its first-order expansion, which must hold due to the definition of the sensitivity gradients \eqref{eq:kkt-sensitivity}. 
\end{proof}

\subsection{Ellipsoidal Trust Region (ETR)}
\Cref{lemma:taylor-approximation} implies that the linear model will no longer be accurate far away from $(\bar{q},\bar{u})$, even with the benefits of smoothing \cite{gqdp}. As such, many optimization methods are further concerned with where we can \emph{trust} the local model, and introduce the concept of a \emph{trust region} around ($\bar{q},\bar{u})$, which further qualifies the accuracy of the linear model \cite{trust-region, trust-region-generalization}. To simplify computation, these trust regions are often chosen to be simple geometric primitives. 

Following classical works, let us first consider an ellipsoidal trust region which can be described in quadratic form:
\begin{equation} \label{eq:ellpsoidal_trust_region}
\mathcal{E}_\mathbf{\Sigma}(\bar{q},\bar{u})\coloneqq \{\delta z = (\delta q,\delta u) | \delta z^\top \mathbf{\Sigma} \delta z \leq 1\}.
\end{equation}

The set of ${q}_+$ achievable with this ellipsoidal trust region can be written as
\begin{align}
\label{eq:ellpsoidal_trust_region_alternative}
\{\hat{q}_+ |\hat{q}_+ = \mathbf{A}_\kappa\delta q + \mathbf{B}_\kappa\delta u + f_\kappa(\bar{q},\bar{u}), (\delta q,\delta u)\in\mathcal{E}_\mathbf{\Sigma}(\bar{q},\bar{u})\}
\end{align}
We emphasize the set \eqref{eq:ellpsoidal_trust_region_alternative} is still an ellipsoid, as it is the image of the ellipsoid $\mathcal{E}_\mathbf{\Sigma}(\bar{q},\bar{u})$ under the linear map $(\mathbf{A}_\kappa, \mathbf{B}_\kappa)$ with a constant offset $f_\kappa(\bar{q},\bar{u})$.

In order for this linear model to be consistent with the true dynamics within the trust region, choosing an appropriate trust region parameter $\mathbf{\Sigma}$ is imperative. If $\mathbf{\Sigma}$ has too large of a volume, it is likely that there will be large errors in the behavior of the system within the trust region. On the other hand, if $\mathbf{\Sigma}$ is too small, iterative algorithms will have to take more iterations to converge. 

\subsection{Contact Trust Region (CTR)}
\label{sec:local-approximation:ctr}
However, contact is more special in that it is a \emph{unilaterally constrained} dynamical system\footnote{Bilateral constraints are often included in sensitivity analysis, so would be implicitly obeyed by the linear model.}. If points in the trust region predicts a behavior of primal and dual variables that are not feasible, this can also cause a large discrepancy in the approximated model within the trust region. In practice solvers would also try to find an ellipsoidal trust region such that all the elements within the trust region would result in a feasible prediction \cite{computing-trust-region-step, trust-region-constrained-solver}. However, this is often achieved by shrinking the ETR centered at the linearization point, resulting in overly conservative(small) trust regions.

Thus, we propose to separate the role of the ellipsoid in keeping Taylor approximation error low, and having to be inscribed within the feasible set. This is achievable by intersecting the ellipsoidal trust region and the primal and dual feasible set of the SOCP contact dynamics \eqref{eq:q_dynamic_socp} under linearized approximations of the primal and dual variables. We formalize this object in \Cref{def:ctr}.
\begin{definition}[\bfseries Contact Trust Region]\label{def:ctr}\normalfont
We define the Contact Trust Region (CTR) at $(\bar{q},\bar{u})$ as the set of all allowable perturbations that do not result in violation of the primal and dual feasibility constraints under a linear model,
\begin{subequations}
\label{eq:full_contact_trust_region}
\begin{align}
\mathcal{S}_{\mathbf{\Sigma},\kappa}(\bar{q},\bar{u}) \coloneqq \{ & (\delta q, \delta u) |  \delta z^\top\mathbf{\Sigma}\delta z \leq 1, \delta z = (\delta q, \delta u), \label{eq:full_contact_trust_region:etr}\\
\hat{q}_+ & = \mathbf{A}_\kappa\delta q + \mathbf{B}_\kappa\delta u + f_\kappa(\bar{q},\bar{u}), \label{eq:full_contact_trust_region:primal_linearization}\\
\hat{\lambda}_{+,i} & = \mathbf{C}_{\kappa,i} \delta q + \mathbf{D}_{\kappa,i} \delta u + \lambda_{\kappa,i}(\bar{q},\bar{u}), \label{eq:full_contact_trust_region:dual_linearization}\\ 
\hat{\mathbf{J}}_i\hat{q}_+ + \hat{c}_i & \in \mathcal{K}_i, \label{eq:full_contact_trust_region:primal_feasibility}\\
\hat{\lambda}_{+,i} & \in \mathcal{K}^\star_i \label{eq:full_contact_trust_region:dual_feasibility}
\}. 
\end{align}
\end{subequations}
\end{definition}

Although the linearizations \eqref{eq:full_contact_trust_region:primal_linearization} and \eqref{eq:full_contact_trust_region:dual_linearization} are taken under the smoothed, unconstrained dynamics \eqref{eq:q_dynamics_log}, enforcing the feasibility constraints of the SOCP dynamics \eqref{eq:q_dynamic_socp} is still justified. Firstly, the primal feasibility constraint \eqref{eq:full_contact_trust_region:primal_feasibility} is implied by the domain of the logarithm in \eqref{eq:q_dynamics_log}. Furthermore, although the nominal contact forces $\lambda_{\kappa, i}$ in are dual feasible (because of \eqref{eq:dual_feasible_point}), we still need to ensure the feasibility of the linearization $\hat{\lambda}_{\kappa, i}$ by including \eqref{eq:full_contact_trust_region:dual_feasibility}. 

Notably, the convex constraints \eqref{eq:full_contact_trust_region:primal_linearization}-\eqref{eq:full_contact_trust_region:dual_feasibility} are locally equivalent to the perturbed (smoothed) KKT conditions of contact dynamics \eqref{eq:q_dynamics_kkt} (which is equivalent to \eqref{eq:q_dynamic_socp}). By \Cref{lemma:taylor-approximation}, the primal and dual Taylor expansions, \eqref{eq:full_contact_trust_region:primal_linearization} and \eqref{eq:full_contact_trust_region:dual_linearization}, are equivalent to the stationarity condition \eqref{eq:unconstrained_staionarity_with_lambda} (which is equivalent to \eqref{eq:balance_single_equation}) and the perturbed complementarity condition \eqref{eq:perturbed-kkt} (which is the perturbation of \eqref{eq:friction_constraints:complementary_slackness}). Moreover, primal feasibility \eqref{eq:full_contact_trust_region:primal_feasibility} locally implies \eqref{eq:friction_constraints:primal_feasibility}, and dual feasibility \eqref{eq:full_contact_trust_region:dual_feasibility} implies \eqref{eq:friction_constraints:dual_feasibility}.

Oftentimes we are also interested in the effect of action perturbations for a \emph{fixed} nominal state. Therefore, we introduce a simplified case of the CTR that has zero perturbations in configuration $(\delta q=0)$.
\begin{definition}[\bfseries Action-only Contact Trust Region]\normalfont
\label{def:actr}
We define the Action-only Contact Trust Region (A-CTR) as 
\begin{subequations}\label{eq:action-contact-trust-region}
\begin{align}
\mathcal{S}^\mathcal{A}_{\mathbf{\Sigma},\kappa}(\bar{q},\bar{u}) \coloneqq \{ \delta u | & \delta u^\top\mathbf{\Sigma} \delta u \leq 1, \\
\hat{q}_+ & = \mathbf{B}_\kappa\delta u + f_\kappa(\bar{q},\bar{u}), \\
\hat{\lambda}_{+,i} & = \mathbf{D}_{\kappa,i} \delta u + \lambda_{\kappa,i}(\bar{q},\bar{u}), \label{eq:action-contact-trust-region:ellipsoidal}\\ 
\mathbf{J}_i \hat{q}_+ + c_i & \in \mathcal{K}_i, \label{eq:action-contact-trust-region:primal_feasibility}  \\
\hat{\lambda}_{+,i} & \in \mathcal{K}^\star_i \label{eq:action-contact-trust-region:dual_feasibility}\}.
\end{align}
\end{subequations}
\end{definition}

Both the CTR \eqref{eq:full_contact_trust_region} and the A-CTR \eqref{eq:action-contact-trust-region} are convex as they each can be expressed as an intersection of convex constraints.

\begin{figure*}[t]
\centering\includegraphics[width = 1.0\textwidth]{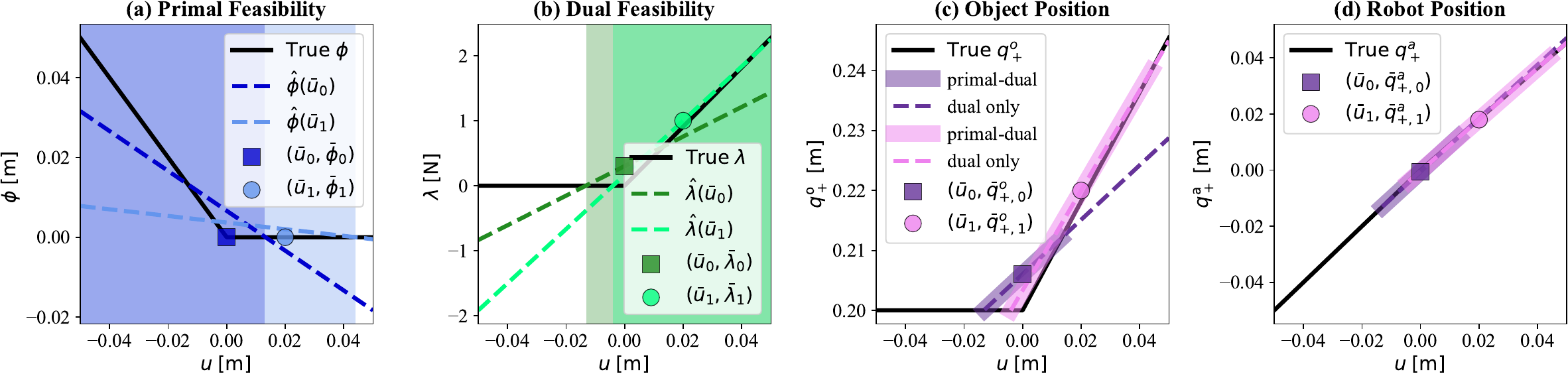}
\caption{
A-CTR illustrations for the system in \Cref{fig:1d-pushing-schematic}b at $(\bar{q}^\mathrm{o}, \bar{q}^\mathrm{a}) = (0.2, 0)$. Sub-figures show how (\textbf{a}) the signed distance function $\phi$, (\textbf{b}) the contact force $\lambda$, (\textbf{c}) the object configuration $\quplus$ and (\textbf{d}) the robot configuration $\qaplus$ change as a function of the action $u$.
In every sub-figure, the black solid line represents the true, non-smooth dynamics \eqref{eq:q_dynamic_f}.
The markers (squares and circles) represent linearization points.
In (\textbf{a}) and (\textbf{b}),
the deeper-colored dotted lines represent the linearization at $\bar{u}_0 = 0$, and the lighter-colored lines represent the linearization at $\bar{u}_1=0.02$. 
The shaded regions represent the feasible set of the corresponding color. For instance, the dark blue shaded region in (\textbf{a}) represents 
$\hat{\phi}\left(\bar{q}=(0.2, 0), \bar{u}=0\right) \geq 0$; 
the light green region in (\textbf{b}) represents
$\hat{\lambda}_+\left(\bar{q}=(0.2, 0), \bar{u}=0.02\right) \geq 0$.
In (\textbf{c}) and (\textbf{d}), the dotted lines show parts of the linearizations that satisfy only the dual constraints; the thick shaded lines around dotted lines satisfy both the primal and dual constraints.
} 
\label{fig:1d_pushing_primal_dual_feasibility}
\vskip -0.1 true in
\end{figure*}

\begin{example}[\bfseries A-CTR for 1D Pushing]\label{ex:pushing-actr}\normalfont
Consider the 1D system in \Cref{fig:1d-pushing-schematic} with two bodies of width $0.2 \mathrm{m}$, one actuated (red sphere) and one unactuated (gray box), both constrained to slide on a frictionless surface along the $x$ axis. 

In this example, we illustrate the A-CTRs for the nominal configuration in \Cref{fig:1d-pushing-schematic}b, where $\bar{q} = (\bar{q}^\mathrm{o}, \bar{q}^\mathrm{a}) = (0.2, 0)$: the ball is touching the left face of the box. 


Firstly, we note that the true, non-smooth contact dynamics (defined by \eqref{eq:q_dynamic_f} and shown as the black line segments in \Cref{fig:1d_pushing_primal_dual_feasibility}a-c) has two contact modes: (\textbf{i}) a \emph{no-contact} mode corresponding to the left linear piece with domain $u \leq 0$, and (\textbf{ii}) an \emph{in-contact} mode corresponding to the right piece with domain $u\geq 0$. The two modes are also present in the robot dynamics in \Cref{fig:1d_pushing_primal_dual_feasibility}d, but the slope difference between the two modes is barely noticeable.

Secondly, the dual feasibility constraints \eqref{eq:action-contact-trust-region:dual_feasibility} are crucial to prevent unphysical behaviors, such as pulling the object via contact forces. As shown in \Cref{fig:1d_pushing_primal_dual_feasibility}b, the linearizations (dotted lines) incorrectly suggest the box can pull the object with a negative $u$. However, once the system switches to no-contact mode according to the true dynamics, the box remains stationary. By enforcing dual feasibility on $\hat{\lambda}$, the vast majority of negative $u$ is restricted, as demonstrated by the green shaded regions in \Cref{fig:1d_pushing_primal_dual_feasibility}b. A small amount of negative $u$ is still allowed due to mismatch between the gradients of the smoothed and true dynamics.

Interestingly, we observe that although the smoothed gradient does not perfectly match the gradient of the true dynamics, the mismatch can be reduced by tweaking the linearization point. 
As shown in \Cref{fig:1d_pushing_primal_dual_feasibility}, the linearization at $\bar{u}_0 = 0.0$ under smoothed contact dynamics does not match the gradient of either mode well, as $\bar{u}_0$ is on the boundary between the two modes. As a result, the boundary of the feasible regions computed from the linearization (e.g. dark green shaded region in \Cref{fig:1d_pushing_primal_dual_feasibility}b) does not match the boundary of the in-contact mode perfectly. On the other hand, for $\bar{u}_1 = 0.02$, a linearization point further away from the boundary of the domain of the in-contact mode, both the linearization (light green dotted line) and the approximated boundary of the dual feasible set (light green shaded region) match the in-contact piece much more closely.

Lastly, we argue that enforcing the primal feasibility constraint \eqref{eq:action-contact-trust-region:primal_feasibility} can needlessly restrict robot motion. As shown in \Cref{fig:1d_pushing_primal_dual_feasibility}a, under the true signed-distance function, the entire in-contact piece satisfies the non-penetration constraint with $\phi = 0$. However, due to the slight gradient mismatch introduced by smoothing, the linear approximation $\hat{\phi}$ is never perfectly flat. Consequently, the primal-dual feasible regions (shaded regions in \Cref{fig:1d_pushing_primal_dual_feasibility}c and \Cref{fig:1d_pushing_primal_dual_feasibility}d) end just a few centimeters ahead of the linearization points, whereas the linear approximations alone remain valid over a much larger domain defined solely by the dual constraints (dashed lines in the same sub-figures).
\end{example}

However, imposing the primal and dual feasibility constraints in \eqref{eq:full_contact_trust_region} does not always reduce the trust region size, which we illustrate in \Cref{ex:squeezing-actr}.

\begin{figure}
\centering\includegraphics[width = 0.48\textwidth]{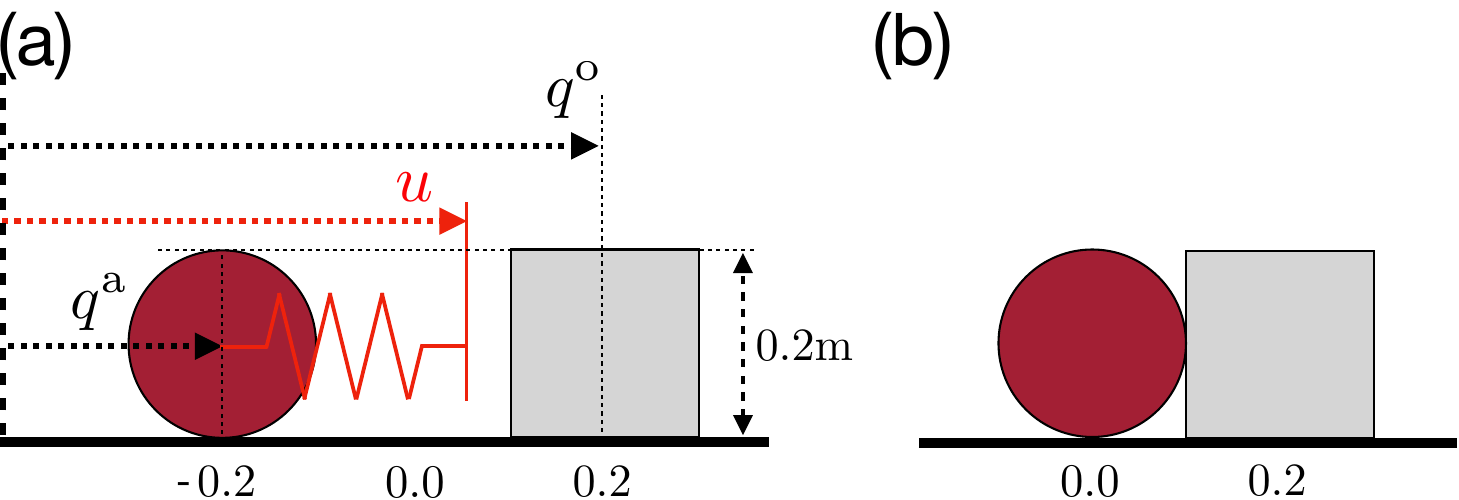}
\caption{
\textbf{(a)} Schematic of the 1-dimensional pusher system used in \Cref{ex:pushing-actr}. \textbf{(b)} corresponds to the configuration $(\qu, \qa) = (0.2, 0)$.
} 
\label{fig:1d-pushing-schematic}
\vskip -0.1 true in
\end{figure}

\begin{figure}[t]
\centering\includegraphics[width = 0.46\textwidth]{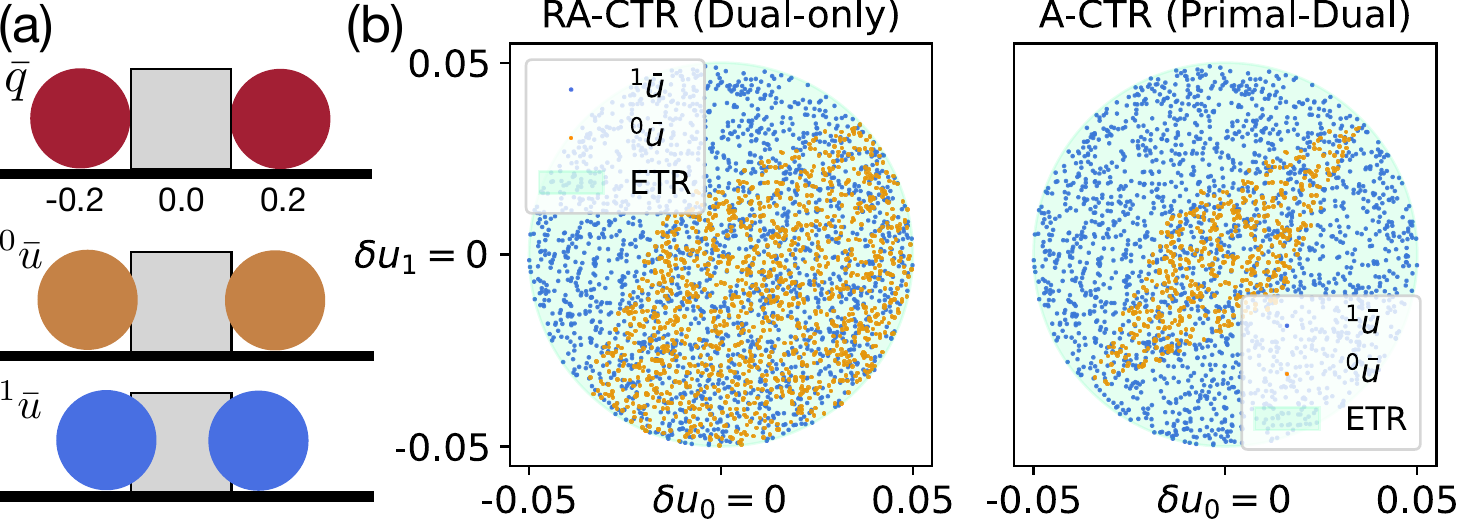}
\caption{
\textbf{(a)} Nominal configuration and actions for \Cref{ex:squeezing-actr}. $u_0$ and $u_1$ are the position commands of the left and right ball, respectively.
\textbf{(b)} Samples in the action space that satisfy the primal \eqref{eq:full_contact_trust_region:primal_feasibility} and dual \eqref{eq:full_contact_trust_region:dual_feasibility} feasibility constraints for different nominal actions. 
}
\label{fig:three_spheres_ctr}
\vskip -0.2 true in
\end{figure}

\begin{example}[\bfseries A-CTR for 1D Squeezing]\label{ex:squeezing-actr}\normalfont
Consider the 1D system in \Cref{fig:three_spheres_ctr}. Similar to \Cref{ex:pushing-actr}, both the robot and the object slide on a frictionless surface. But the robot now consists of two actuated spheres, one on each slide of the box. We illustrate the A-CTRs for the nominal configuration in \Cref{fig:three_spheres_ctr}a, where $\bar{q}^\mathrm{o} = 0.0$ and $\bar{q}^\mathrm{a}=(-0.2, 0.2)$: each side of the box has a ball touching it. The action space for this system is 2D. In the ETR constraint \eqref{eq:full_contact_trust_region:etr}, we set $\mathbf{\Sigma}=\mathrm{diag}([0.05^2, 0.05^2])^{-1}$, i.e. the ETR is a sphere of radius $0.05$ (blue spheres in \Cref{fig:three_spheres_ctr}b).

We sample 2000 points from the ETR, reject samples that violate the primal and/or dual feasibility constraints, and plot the remaining samples in \Cref{fig:three_spheres_ctr}b. For the nominal action ${}^0\bar{u} = (-0.19, 0.19)$ (orange dots), the feasibility constraints eliminate a fair amount of samples. Enforcing both primal and dual constraints eliminate more samples than enforcing only the dual constraints. 

However, as the nominal action gets deeper into penetration, e.g. ${}^1\bar{u} = (-0.16, 0.16)$ (blue dots), both the primal-dual and the dual-only feasible regions grow and completely contain the ETR, as shown by the blue dots covering the blue spheres entirely in \Cref{fig:three_spheres_ctr}b.

This is not a surprising result: as the spheres squeeze the box harder (deeper commanded penetration), there is more room to wiggle the nominal action without losing contact.
\end{example}

\subsection{Relaxed Contact Trust Region (R-CTR)}\label{sec:local-approximation:relaxed-ctr}
We believe the unnecessary restrictions on trust region sizes caused by the primal feasibility constraint in \Cref{ex:pushing-actr} are not an isolated phenomenon, and we hypothesize that imposing primal feasibility in general is too conservative. In particular, under a linear model of the smoothed dynamics, the sensitivity of the actuated bodies is often larger than the sensitivity of unactuated objects: this is a relaxation of the sensitivity of the unactuated body being $\mathbf{0}$ and the actuated body being $\mathbf{I}$ when not in contact. Therefore, the actuated body catches up to penetrate the unactuated object upon very small perturbations. 

Accordingly, we define variants of CTR and A-CTR that relax the primal feasibility constraint. 

\begin{definition}[\bfseries Relaxed Contact Trust Region]\label{def:relaxed-ctr}\normalfont
We define the Relaxed Contact Trust Region (R-CTR) at $(\bar{q},\bar{u})$ as the set of all allowable perturbations that satisfy \emph{only the dual feasibility constraints} under a linear model,
\begin{subequations}
\label{eq:relaxed_contact_trust_region}
\begin{align}
\tilde{\mathcal{S}}_{\mathbf{\Sigma},\kappa}(\bar{q},\bar{u}) \coloneqq \{ & (\delta q, \delta u) |  \delta z^\top\mathbf{\Sigma}\delta z \leq 1, \delta z = (\delta q, \delta u), \\
\hat{q}_+ & = \mathbf{A}_\kappa\delta q + \mathbf{B}_\kappa\delta u + f_\kappa(\bar{q},\bar{u}), \label{eq:relaxed_contact_trust_region:primal_linearization}\\
\hat{\lambda}_{+,i} & = \mathbf{C}_{\kappa,i} \delta q + \mathbf{D}_{\kappa,i} \delta u + \lambda_{\kappa,i}(\bar{q},\bar{u}), \label{eq:relaxed_contact_trust_region:dual_linearization}\\ 
\hat{\lambda}_{+,i} & \in \mathcal{K}^\star_i \label{eq:relaxed_contact_trust_region:dual_feasibility}
\}. 
\end{align}
\end{subequations}
\end{definition}

\begin{definition}[\bfseries Relaxed Action-only Contact Trust Region]\normalfont
\label{def:relaxed-actr}
We define the Relaxed Action-only Contact Trust Region (RA-CTR) as 
\begin{subequations}\label{eq:relaxed-action-contact-trust-region}
\begin{align}
\tilde{\mathcal{S}}^\mathcal{A}_{\mathbf{\Sigma},\kappa}(\bar{q},\bar{u}) \coloneqq \{ \delta u | & \delta u^\top\mathbf{\Sigma} \delta u \leq 1, \\
\hat{q}_+ & = \mathbf{B}_\kappa\delta u + f_\kappa(\bar{q},\bar{u}), \\
\hat{\lambda}_{+,i} & = \mathbf{D}_{\kappa,i} \delta u + \lambda_{\kappa,i}(\bar{q},\bar{u}), \\ 
\hat{\lambda}_{+,i} & \in \mathcal{K}^\star_i \label{eq:relaxed-action-contact-trust-region:dual_feasibility}\}.
\end{align}
\end{subequations}
\end{definition}
In both \Cref{def:relaxed-ctr} and \Cref{def:relaxed-actr}, we keep $\hat{q}^\mathrm{a}$ although they are not directly used in the trust region definitions: they are useful for imposing joint and torque limits, which are detailed in \Cref{sec:additional-constraints}.

Furthermore, we make the following definition for the predicted dynamics over the relaxed trust regions.
\begin{definition}[\bfseries Motion Set]\normalfont\label{def:motion-set}
We define the image of $\tilde{\mathcal{S}}(\bar{q},\bar{u})$ under the linearized primal solution map as the Motion Set,
\begin{subequations} \label{eq:motion_set}
\begin{align}
    \mathcal{M}_\mathbf{\Sigma,\kappa}(\bar{q},\bar{u})\coloneqq \{\hat{q}_+ | & \hat{q}_+ = \mathbf{A}_\kappa\delta q + \mathbf{B}_\kappa\delta u + f_\kappa(\bar{q},\bar{u}),\\
    & (\delta q,\delta u)\in \tilde{\mathcal{S}}_{\mathbf{\Sigma},\kappa}(\bar{q},\bar{u})\}.
\end{align}
\end{subequations}
\end{definition}

\begin{figure*}[t]
\centering\includegraphics[width = 1.0\textwidth]{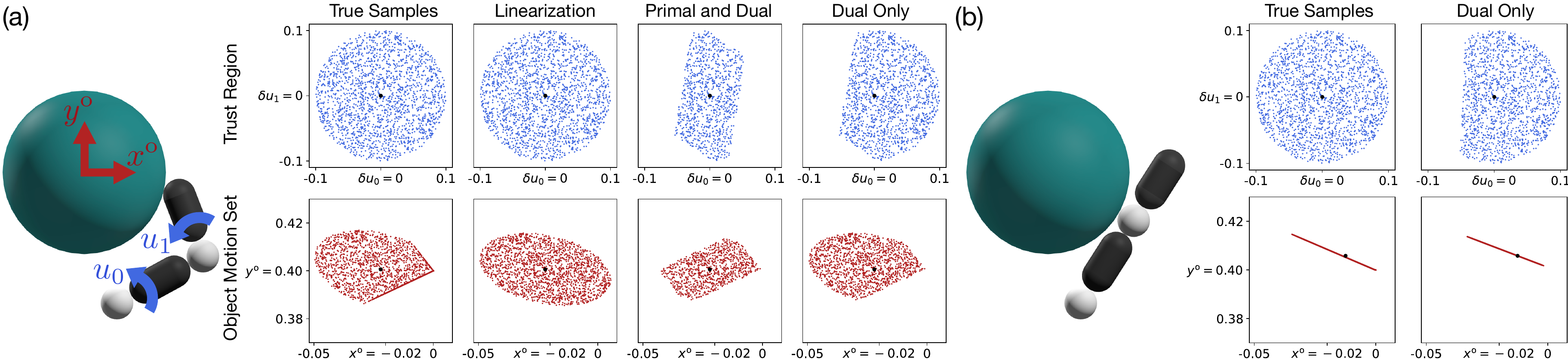}
\caption{Visualization of the A-CTR $\mathcal{S}^\mathcal{A}_{\mathbf{\Sigma},\kappa}$ (primal and dual), RA-CTR $\tilde{\mathcal{S}}^\mathcal{A}_{\mathbf{\Sigma},\kappa}$ (dual only) and the action-only object motion set $\mathcal{M}^{\mathcal{A},\mathrm{o}}_{\mathrm{\Sigma}, \kappa}$ under the different linearization points shown in the illustrations.
In the ``Trust Region" row, the ``True Samples" subplot is obtained by uniformly sampling 2000 points from the ball $\norm{\delta u} \leq 0.1$; samples in the following columns are obtained by rejecting samples that do not satisfy the respective constraints. 
In the ``Object Motion Set" row, the ``True Samples" subplot is obtained by passing the $\delta u$ samples from the trust region subplot above through the true contact dynamics \eqref{eq:q_dynamic_f}; samples in the following columns are obtained by mapping the $\delta u$ samples from the corresponding trust region subplot through the linear map defined in \eqref{eq:action-only-motion_set}.
}  
\label{fig:motion-set-visualization}
\vskip -0.15 true in
\end{figure*}

\begin{definition}[\bfseries Action-only Motion Set]\normalfont\label{def:action-only-motion-set}
We define the image of $\tilde{\mathcal{S}}^\mathcal{A}_{\Sigma, \kappa}(\bar{q},\bar{u})$ under the linearized primal solution map as the Action-only {Motion Set},
\begin{subequations} \label{eq:action-only-motion_set}
\begin{align}
    \mathcal{M}^\mathcal{A}_\mathbf{\Sigma,\kappa}(\bar{q},\bar{u})\coloneqq \{\hat{q}_+ | & \hat{q}_+ = \mathbf{B}_\kappa\delta u + f_\kappa(\bar{q},\bar{u}),\\
    & \delta u \in \tilde{\mathcal{S}}^\mathcal{A}_{\mathbf{\Sigma},\kappa}(\bar{q},\bar{u})\}.
\end{align}
\end{subequations}
\end{definition}

Similar to how we partition $q$ into $\qu$ and $\qa$, we use superscripts $\cdot^\mathrm{o}$ and $\cdot^\mathrm{a}$ to denote subsets of $\mathcal{M}_\mathbf{\Sigma,\kappa}$ corresponding to the object and robot configurations, respectively. For example, $\mathcal{M}^{\mathcal{A}, \mathrm{o}}_{\mathbf{\Sigma},\kappa}(\bar{q},\bar{u})\coloneqq\{\hat{q}^\mathrm{o}_+ | \hat{q}_+\in\mathcal{M}^\mathcal{A}_{\mathbf{\Sigma},\kappa}(\bar{q},\bar{u})\}$. In addition, as a linear map of a convex set, the motion set \eqref{eq:motion_set} is also convex.

\input{body/motion-set-visualization}

\subsection{Mechanics Derivation of the Motion Set}\label{sec:mechanics-derivation}
We can further motivate the idea that relaxing the primal feasibility constraints improves local approximations of feasible object motions by connecting the RA-CTR to classical constructs in manipulation, such as the motion cone \cite{mechanics,motion-cones} and the wrench set \cite{ferrari-canny,lynchpark}. In this section, we establish these connections by providing an alternative derivation of the action-only motion set (\Cref{def:action-only-motion-set}), starting from contact forces instead of Taylor expansions.

\textbf{The Contact Force Set.}
For a single contact pair $i$, the predicted linear model of how the contact force changes as we vary $\delta u$ is given by $\hat{\lambda}_{+,i} = \mathbf{D}_{\kappa,i}\delta u + \lambda_{\kappa,i}(\bar{q},\bar{u})$, as long as the prediction $\hat{\lambda}_{+,i}$ lies within the friction cone $\mathcal{K}^\star_i$. Thus, the set of allowable contact forces for this pair is given by 
\begin{subequations}
\label{eq:contact_force_set}
\begin{align}
    \mathcal{C}^\mathcal{A}_{\mathbf{\Sigma},\kappa,i}(\bar{q},\bar{u})\coloneqq \big\{\hat{\lambda}_{+,i} |& \hat{\lambda}_{+,i} = \mathbf{D}_{\kappa,i}\delta u + \lambda_{\kappa,i}(\bar{q},\bar{u}) \\
    &\hat{\lambda}_{+,i}  \in \mathcal{K}_i^\star, \\
    &\delta u^\top \mathbf{\Sigma}\delta u  \leq 1 \big\}.
\end{align}
\end{subequations}

\textbf{The Generalized Friction Cone.} The generalized friction cone \cite{generalized-friction-cone} is a linear mapping of the friction cone to the space of object wrenches. This can be obtained by applying the contact Jacobian that linearizes the kinematics of the contact point with respect to object coordinates,
\begin{subequations}
\label{eq:generalized_friction_cone}
    \begin{align}
        \mathcal{JC}^\mathcal{A}_{\mathbf{\Sigma},\kappa,i}(\bar{q},\bar{u})\coloneqq \left\{w_i | w_i = \Ju[i]^\top\lambda_{+,i}, \; \lambda_{+,i} \in \mathcal{C}^\mathcal{A}_{\mathbf{\Sigma},\kappa,i}\right\}.
    \end{align}
\end{subequations}

\textbf{The Wrench Set.} The wrench set \cite{ferrari-canny, lynchpark} is defined as the set of all achievable wrenches that can be applied to an object given all possible contact forces that can be applied from a given configuration. This quantity has classically served as an important metric in grasping analysis \cite{ferrari-canny,han-grasp-analysis,hongkai-phd-thesis}. While the classical wrench set has bounds on the contact force, we instead apply bounds on the actuator input  ($\|\delta u\|_\mathbf{\Sigma}^2\leq 1$). 

Our version of the wrench set is defined by taking Minkowski sums of each generalized friction cone $\mathcal{JC}^\mathcal{A}_{\mathbf{\Sigma},\kappa, i}$ for all contact pairs,

\begin{subequations}
\label{eq:wrench_set}
\begin{align}
\mathcal{W}^\mathcal{A}_{\mathbf{\Sigma}, \kappa}(\bar{q},\bar{u}) \coloneqq \{w | w & = \tau^\mathrm{o} + \textstyle\sum_i w_i\\
w_i & \in \mathcal{JC}^\mathcal{A}_{\mathbf{\Sigma},\kappa, i}(\bar{q},\bar{u})\},
\end{align}
\end{subequations}
where $h\tau^\mathrm{o}$ is some external impulse on the object (e.g. by gravity). The wrench set adequately describes the set of all wrenches that can be applied on the object from a given nominal configuration and input $(\bar{q},\bar{u})$. 

\textbf{The Motion Set.} The motion set is analogous to motion cones \cite{mechanics, motion-cones}. While classical motion cones describe the set of feasible object velocity under a single patch contact, we define motion cone as the set of all feasible displacements achievable by multiple contacts in one time step. 

We construct this object by utilizing a quasistatic relation between the object movement and the applied impulse \eqref{eq:unactuatedbalance},
\begin{subequations}\label{eq:motion-sets-classical-derivation}
\begin{align}
\mathcal{M}^{\mathcal{A},\mathrm{o}}_{\mathbf{\Sigma},\kappa}(\bar{q},\bar{u}) = \{\hat{q}_+^\mathrm{o} |&  \epsilon\Mu(\bar{q})\left(\hat{q}^\mathrm{o}_+ - \bar{q}^\mathrm{o}\right)/h = h w,\\
& w\in\mathcal{W}^\mathcal{A}_{\mathbf{\Sigma},\kappa}(\bar{q},\bar{u})\}.
\end{align}
\end{subequations}
We are now ready to draw the connection between the construction of motion sets in this section, and the definition of the motion set as the image of the trust region in \Cref{def:action-only-motion-set}.
\begin{lemma}\label{lemma:motion-wrench}\normalfont
The object motion set derived in \eqref{eq:motion-sets-classical-derivation} is equivalent to object subset of the Action-only Motion Set $\mathcal{M}^\mathcal{A}_\mathbf{\Sigma,\kappa}(\bar{q},\bar{u})$ defined in \Cref{def:action-only-motion-set}.
\end{lemma}
\begin{proof}
\Cref{app:proof}.
\end{proof}

\subsection{Additional Constraints}\label{sec:additional-constraints}
We list some additional constraints and costs that may be added to this formulation to further restrict the CTR.
\subsubsection{Joint Limits}
Consider joint limits $q^\mathrm{a}_{lb}$ and $q^\mathrm{a}_{lb}$ that must be enforced by the robot. Then, we can add the constraint that limits the position command that we send to the robot,
\begin{equation}
    q^\mathrm{a}_{lb}\leq \bar{u} + \delta u \leq q^\mathrm{a}_{ub}.
\end{equation}

\subsubsection{Torque Limits}
We often want to limit the torque applied by the robot to the object by imposing lower and upper bounds $\tau_{lb}$ and $\tau_{ub}$ on joint torque. Although the quasi-dynamic formulation cannot account for dynamic transient torques, we can compute the steady-state torque experienced by the robot using the difference in the sent position command and the predicted position \cite{pang2022easing}. Thus, the set of achievable $\delta u$ that does not violate this steady-state torque limit can be written as
\begin{equation}
   \tau_{lb}\leq \mathbf{K}_\mathrm{a}^{-1}\left(\bar{q}^\mathrm{a}_+ + \mathbf{B}^\mathrm{a} \delta u - (\bar{u} + \delta u)\right)\leq \tau_{ub}.
\end{equation}
where $\tau^\mathrm{a}$ accounts for gravitational torques on the robot.

%% file: body/motion-set-visualization.tex
\begin{example}[\bfseries A-CTRs and Motion Sets]\label{ex:planar_hand_one_finger}\normalfont
Consider the simple robotic system in \Cref{fig:motion-set-visualization}, where a 2-joint robot arm attempt to move a sphere which has two translational DOFs $(x^\mathrm{o}, y^\mathrm{o})$ and does not rotate. The motion sets attempt to approximate the 1-step reachable set at the nominal configurations in the figure. Nominal actions are chosen so that the robot is slightly in penetration with the object.

Firstly, the motion set with only the linearization and no feasibility constraints cannot capture the unilateral nature of contact, as ellipsoids are bi-directional.

Secondly, imposing both primal and dual feasibility constraints yields an overly conservative representation of the object’s possible motion. In contrast, enforcing only dual feasibility provides a closer match to the true object motion set. This is consistent with our observations from \Cref{ex:pushing-actr}.

Lastly, from \Cref{fig:motion-set-visualization}b, we note the proposed action-only motion set, $\mathcal{M}^{\mathcal{A},\mathrm{o}}_{\mathrm{\Sigma}, \kappa}$, respects singular robot configurations, as the sensitivity analysis procedure factor in the configuration of the manipulator. 
\end{example}

%% file: body/inverse-dynamics.tex
\section{Local Planning and Control}\label{sec:local-planning-control}
\noindent We now study local gradient-based planning and control, which is one of the core applications of our proposed contact trust region. Among possible formulations of this problem, we focus on computing a configuration and input sequence that moves the system towards a goal configuration $q_\text{goal}$. The full nonlinear form of this problem can be written as
\begin{subequations}\label{eq:nonlinear-to}
\begin{align}
    \min_{q_{0:T},u_{0:T-1}}\quad & \|q_\text{goal} - q_T\|^2_\mathbf{Q} + \sum^{T-1}_{t=0} \|u_t - u_{t-1}\|^2_\mathbf{R}, \label{eq:nonlinear-to:cost} \\
    \text{s.t.} \quad & q_{t+1} = f(q_t, u_t) \quad \forall t, \label{eq:nonlinear-to:dynamics}\\
                      & |u_t - u_{t-1}| \leq \eta \quad \forall t,  \label{eq:nonlinear-to:input_constraint}\\
                      & q_0 = \bar{q}_0,  \label{eq:nonlinear-to:initial_condition}
\end{align}
\end{subequations}
where \eqref{eq:nonlinear-to:dynamics} enforces dynamics constraints, \eqref{eq:nonlinear-to:input_constraint} enforces input limits (recall that $u_t$ is a position command, thus input limits are enforced in relative form), and \eqref{eq:nonlinear-to:initial_condition} enforces the initial condition.

One of the biggest challenges in solving \eqref{eq:nonlinear-to} is handling the non-smooth contact dynamics constraint \eqref{eq:nonlinear-to:dynamics}. Although MIP-based methods \cite{tobia-recovery, marcucci2019mixed} have struggled to scale up to complex, contact-rich problems, the MIP formulation reveals why the problem is hard: the search through the exponentially many contact modes.

In this section, we present a method for solving \eqref{eq:nonlinear-to} by incorporating contact dynamics smoothing and the contact trust region into an iLQR-like \cite{li2004iterative} trajectory optimization scheme. Through two toy problems, we show that the proposed method can iteratively approach an advantageous contact mode for reaching the given goals, even when the initial guess is in a contact mode with non-informative gradient. In addition, we present how the method can be used as a model predictive controller.  

\subsection{Trajectory Optimization with R-CTR} \label{sec:ctr_trajectory_optimization}
Consider a trajectory optimization scheme, where we first obtain some guess of the nominal input $\bar{u}_{0:T-1}$, then roll it out under the dynamics to get the nominal configuration trajectory $\bar{q}_{0:T}$. Under a local approximation of \eqref{eq:nonlinear-to} around this nominal trajectory $(\bar{q}_{0:T}, \bar{u}_{0:T-1})$ utilizing gradients of smoothed dynamics and R-CTR, we search for optimal perturbations $(\delta q_{0:T},\delta u_{0:T-1})$ by solving the following local trajectory optimization problem with linear dynamics constraints:
\begin{subequations}\label{eq:linear_trajopt}
\begin{align}
&\mathrm{SubTrajOpt}(\bar{q}_{0:T}, \bar{u}_{0:T-1}, q_\text{goal}) = \delta u^\star_{0:T-1}, \text{where} \\
&\min_{\delta q_{0:T},\delta u_{0:T-1}} \quad  \|q_\text{goal} - q_T\|^2_\mathbf{Q} + \sum^{T-1}_{t=0} \|u_t - u_{t-1}\|^2_\mathbf{R}, \label{eq:linear_trajopt:cost}\\
& \text{s.t.} \quad \delta q_{t+1}= \mathbf{A}_{\kappa, t} \delta q_t + \mathbf{B}_{\kappa, t} \delta u_t, \; t = 0\dots T-1, \label{eq:linear_trajopt:linear_dynamics_constraint} \\
& \quad \quad (\delta q_t, \delta u_t)\in \tilde{\mathcal{S}}_\mathbf{\Sigma, \kappa}(\bar{q}_t,\bar{u}_t), \; t = 0\dots T-1,\label{eq:linear_trajopt:ctr}\\
& \quad \quad q_t = \bar{q}_t + \delta q_t, \; t = 0\dots T, \\
& \quad \quad u_t = \bar{u}_t + \delta u_t, \; t = 0\dots T-1,\\
& \quad \quad |u_t - u_{t-1}| \leq \eta, \; t = 1\dots T-1,\\
& \quad \quad \delta q_0 = 0, \label{eq:linear_trajopt:initial_condition}
\end{align}
\end{subequations}

\noindent Here \eqref{eq:linear_trajopt:linear_dynamics_constraint} is the standard linear dynamics constraint in linear MPC, where $\delta q_{t+1} \coloneqq q_{t+1} - f(\bar{q}_t, \bar{u}_t)$; 
$\tilde{\mathcal{S}}_\mathbf{\Sigma,\kappa}$ in \eqref{eq:linear_trajopt:ctr} is the R-CTR \eqref{eq:relaxed_contact_trust_region}; 
\eqref{eq:linear_trajopt:linear_dynamics_constraint} and \eqref{eq:linear_trajopt:ctr} constitute our local approximation of the contact dynamics constraint \eqref{eq:nonlinear-to:dynamics};
\eqref{eq:linear_trajopt:initial_condition} is due to the initial condition constraint \eqref{eq:nonlinear-to:initial_condition}. 
The sub trajectory optimization problem \eqref{eq:linear_trajopt} is a standard SOCP that can be solved by off-the-shelf conic solvers. We note that other constraints as joint limits and torque limits in \Cref{sec:additional-constraints} can be added trivially to this formulation by incorporating them into $\tilde{\mathcal{S}}_{\mathbf{\Sigma},\kappa}$. 

After obtaining the optimal values $\delta q_t^\star, \delta u_t^\star$, we update the nominal input trajectory with $\bar{u}_t\leftarrow\bar{u}_t + \delta u_t^\star$, and repeat the process of rolling out to obtain nominal configuration trajectory, and searching for local improvements. This iterative scheme is summarized in \Cref{alg:trajopt_ctr}. 

Notably, the rollout step (\Cref{alg:trajopt_ctr:rollout}) is analogous to the forward pass in iLQR; the $\mathbf{SubTrajOpt}$ step (\Cref{alg:trajopt_ctr:rollout}) is analogous to the backward pass. Unlike iLQR's backward pass, which solves for the optimal action perturbations $\delta u_t$ iteratively using Bellman recursion, $\mathbf{SubTrajOpt}$ solves for all $\delta u_t$ jointly with convex optimization.

Lastly, we impose the relaxed CTR constraint in \eqref{eq:linear_trajopt:ctr} instead of the full CTR, as the full CTR often leads to under approximation of the 1-step reachable set of the true dynamics. We have shown this via simple examples in \Cref{sec:local-approximation}, and will show more experimental support for the choice of using R-CTR in \Cref{sec:local-planning-results}.
\vskip -0.1 true in
\begin{algorithm}
\caption{\small{CTR Trajectory Optimization-\textbf{CtrTrajOpt}}}\label{alg:trajopt_ctr}
\textbf{Input:} Initial state $q_0$, goal state $q_\text{goal}$, input trajectory guess $\bar{u}_{0:T-1}$, iterations limit $n_\mathrm{max}$\;
\textbf{Output:} Optimized input trajectory $\bar{u}_{0:T-1}^\star$\;
$n \leftarrow 0$\;
\While {not converged $\mathrm{and}$ $n < n_\mathrm{max}$} { \label{alg:trajopt_ctr:while}
    $\bar{q}_{0:T}\leftarrow$ Rollout $f$ from $q_0$ with $\bar{u}_{0:T-1}$ \label{alg:trajopt_ctr:rollout}\;
    $\delta u_{0:T-1}^\star \leftarrow \mathbf{SubTrajOpt}(\bar{q}_{0:T}, \bar{u}_{0:T-1})$ \label{alg:trajopt_ctr:}\;
    $\bar{u}_t \leftarrow \bar{u}_t + \delta u_t^\star \quad \forall t $\;
    $n \leftarrow n + 1$
}
\algorithmicreturn $\; \bar{u}_{0:T-1}$
\end{algorithm}

\subsection{Initial Guess Heuristic}\label{sec:penetration-finder}
As the trajectory optimization problem \eqref{eq:nonlinear-to} is nonconvex, the solution that \Cref{alg:trajopt_ctr} converges to can be sensitive to the choice of initial guesses for the input trajectory $\bar{u}_{0:T-1}$. The most uninformed initial guess would be to set the position command to be the current configuration of actuated bodies ($\bar{u}_t=\bar{q}^\mathrm{a}_0, \forall t$). However, this can be problematic especially when the robots are not in contact with the object in $q_0$, as $\mathbf{B}_\kappa^\mathrm{o}$, the gradient of object motion w.r.t. robot actions, can get close to $0$ even with dynamics smoothing when the distance between them are far away. 

To alleviate this problem, we can compute a position command that would result in the closest contacting configuration from $q_0$, and pass this command as the initial guess for $\bar{u}_{0:T-1}$.

To compute this initial guess, we utilize the intuition that the log-barrier smoothed dynamics in \eqref{eq:q_dynamics_log} results in a force-from-a-distance like effect that pushes away objects even when not in contact. We simply reverse this force field by sending the following torque to the robot, 
\begin{equation}
    \tau = - \sum_i {\mathbf{J}^\mathrm{o}_i}^\top \lambda_i
\end{equation}
and simulate forward using Drake \cite{sap} until the robot is in contact with the object. We use a generous amount of smoothing (e.g. $\kappa\in [10, 100]$) to ensure that the contact forces $\lambda_i$ are large enough to pull the robot links towards the object even if they are far away at $q_0$.


\subsection{Examples}
To elucidate the role of R-CTR in \Cref{alg:trajopt_ctr}, we show how a single-horizon ($T=1$) trajectory optimizer solves two toy problems in \Cref{ex:pushing-example} and \Cref{ex:2d-pushing-example}. We note that the R-CTR reduces to the RA-CTR (\Cref{def:relaxed-actr}) when $T=1$ as the initial condition constraint \eqref{eq:linear_trajopt:initial_condition} eliminates the perturbations on $q$.

\begin{figure}[t]
\centering\includegraphics[width = 0.48\textwidth]{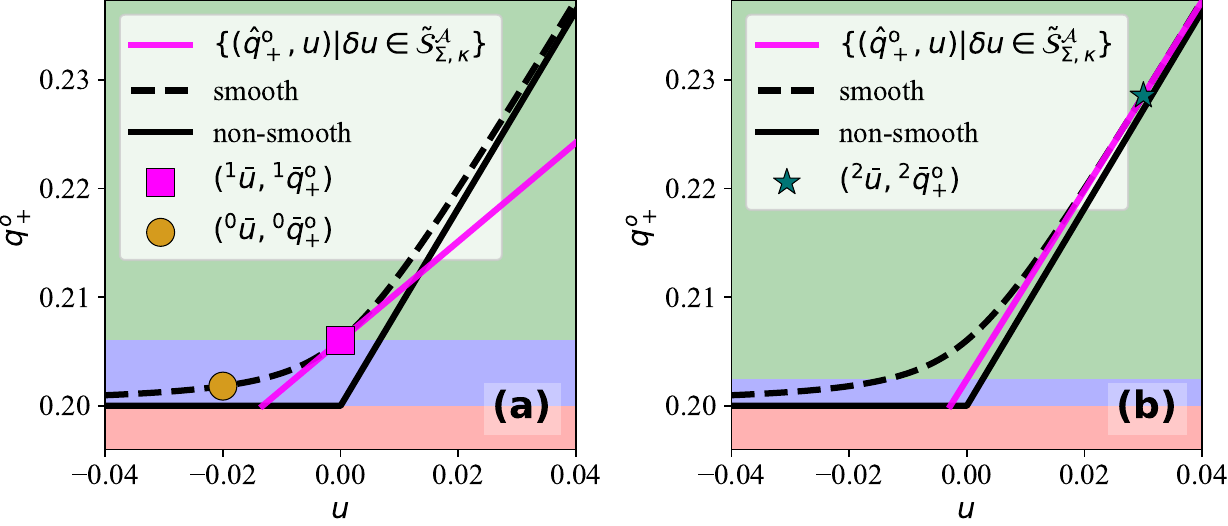}
\caption{
Running \Cref{alg:trajopt_ctr} on the 1D sphere-box system in \Cref{fig:1d-pushing-schematic} from the nominal configuration $(\bar{q}_0^\mathrm{o}, \bar{q}_0^\mathrm{a}) = (-0.02, 0.2)$.
In both subplots, the black solid line represents true non-smooth dynamics, and the black dashed line represents the smoothed dynamics. 
The yellow circle marks the naive nominal action ${}^0\bar{u} = \bar{q}^\mathrm{a}_0$ and its resulting object configuration ${}^0\bar{q}^\mathrm{o}_0$.
Similarly, the magenta square indicates the next nominal action ${}^1\bar{u}$ after applying the initial guess heuristics, and the green star shows the final nominal action ${}^2\bar{u}$ after running \Cref{alg:trajopt_ctr}.
The magenta lines depict the set 
$\{(\hat{q}_+, \bar{u} + \delta u) | \hat{q}_+ = \mathbf{B}_\kappa\delta u + f_\kappa(\bar{q},\bar{u}), \delta u\in \tilde{\mathcal{S}}^{\mathcal{A}}_{\mathbf{\Sigma}, \kappa}(\bar{q},\bar{u})\}$ for the different nominal actions. 
The green, blue, red colored zones denote zones of $q^\mathrm{o}_\text{goal}$ where the inverse dynamics controller behaves similarly.
}
\label{fig:1d_pushing}
\vskip -0.1 true in
\end{figure}

\begin{figure*}[t]
\centering\includegraphics[width = 0.98\textwidth]{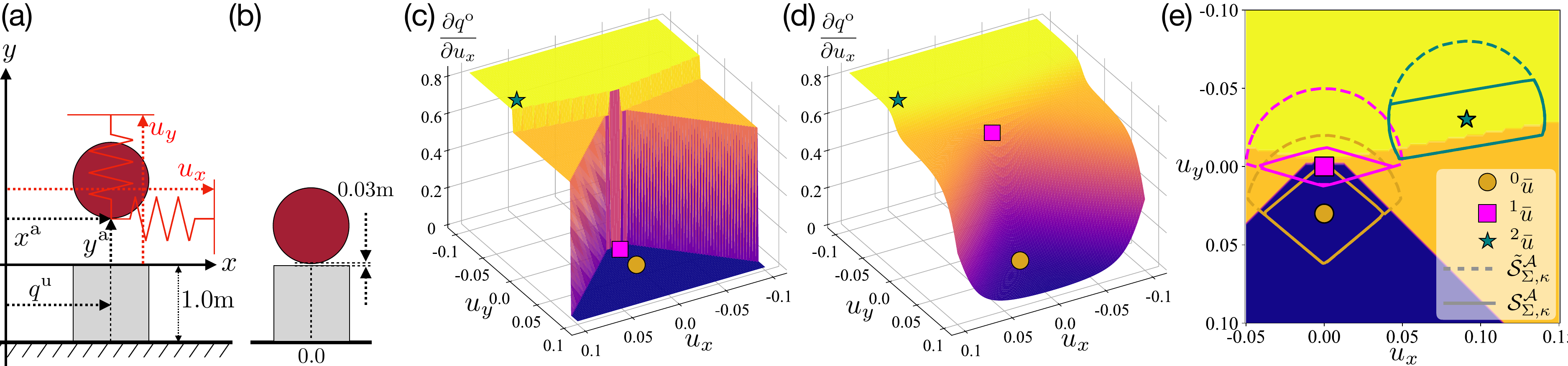}
\caption{
(\textbf{a}) A schematic of the system in \Cref{ex:2d-pushing-example}. 
(\textbf{b}) The nominal configuration $q^\mathrm{o}_0 = 0$ and $q^\mathrm{a}_0 = (x^\mathrm{a}_0, y^\mathrm{a}_0) = (0, 0.03)$: the ball is hovering above but not touching the box.
(\textbf{c}) Gradient landscape of the non-smooth dynamics at the nominal configuration in (\textbf{b}). The blue region at the bottom corresponds to the separation contact mode, the two orange regions the sliding modes, and the yellow region the sticking mode.
(\textbf{d}) Gradient landscape of the smoothed dynamics at the same nominal configuration. The dynamics is smoothed with $\kappa=1000$.
(\textbf{e}) A-CTR ($\mathcal{S}^\mathcal{A}_{\Sigma, \kappa}$) and RA-CTR($\tilde{\mathcal{S}}^\mathcal{A}_{\Sigma, \kappa}$) for different $\bar{u}$'s, overlaid on the non-smooth gradient landscape in (\textbf{c}). 
}
\label{fig:box_ball_graze}
\vskip -0.1 true in
\end{figure*}

\input{body/pushing-example}

\input{body/2d_pushing_example}

\subsection{Model Predictive Control}\label{sec:mpc}
The CTR-based trajectory optimizer in \Cref{sec:ctr_trajectory_optimization} can be readily turned to a controller through Model Predictive Control (MPC). As shown in \Cref{alg:mpc}, we obtain the optimal trajectory $u^\star_t$ by solving \eqref{eq:nonlinear-to} with \Cref{alg:trajopt_ctr} (\Cref{alg:mpc:trajopt}), only use the first input $u^\star_0$ to deploy on the true CQDC dynamics \Cref{alg:mpc:dynamics_rollout}, then re-plan after observing the next configuration. Moreover, we initialize $\mathbf{CtrTrajOpt}$ from the initial guess heuristics (\Cref{sec:penetration-finder}) at the first iteration (\Cref{alg:mpc:first_initialization}), and from the solution of the previous iteration at later iterations (\Cref{alg:mpc:later_initialization}). We refer to the length of the MPC episode as $H$, as opposed to planning horizon $T$.

\begin{algorithm}
\caption{MPC Rollout ($\mathbf{MPC}$)}\label{alg:mpc}
\textbf{Input:} Initial state $q_0$, goal state $q_\text{goal}$, planning horizon $T$, iterations limit $n_\mathrm{max}$, MPC rollout horizon $H$\;
\textbf{Output:} Lists of visited states $L_q$, applied inputs $L_u$\;
$L_q \leftarrow [q_0]$, $L_u \leftarrow []$\;
\For {$t = 0$ to $H - 1$} {
    \eIf{$t == 0$} { \label{alg:mpc:begin_if}
        $\bar{u}_{0:T-1}\leftarrow$ Apply initial guess heuristics to $q_t^\mathrm{a}$\; \label{alg:mpc:first_initialization}
    } {
        $\bar{u}_{0:T-1}\leftarrow$ Initialize from the previous $u_{0:T-1}^\star$\; \label{alg:mpc:later_initialization}
    }
    
    $u_{0:T-1}^\star \leftarrow \mathbf{CtrTrajOpt}(q_t, q_\text{goal}, \bar{u}_{0:T-1}, n_\mathrm{max}$)\; \label{alg:mpc:trajopt}
    $q_{t+1} = f(q_t, u_0^\star)$\; \label{alg:mpc:dynamics_rollout}
    $L_q$.\texttt{append}($q_{t+1}$), $L_u$.\texttt{append}($u_0^\star$)\;
}
\algorithmicreturn $\; L_q, \; L_u$
\end{algorithm}

In \Cref{ex:bimanual-iiwa}, we demonstrate the effectiveness and scalability of the proposed MPC on the \code{IiwaBimanual} example, which has many more contact modes compared to the simple examples we have considered so far.

\input{body/id-example-iiwa}

%% file: body/pushing-example.tex
\begin{example}[\bfseries 1D Pushing]\label{ex:pushing-example}\normalfont
Revisiting the 1D sphere-box system in \Cref{fig:1d-pushing-schematic}, we would like to push the box to a goal position $q^\mathrm{o}_\text{goal}$ by running \Cref{alg:trajopt_ctr}. 
We begin with the nominal configuration $\bar{q}_0 = (\bar{q}_0^\mathrm{o}, \bar{q}_0^\mathrm{a}) = (-0.02, 0.2)$, where the ball nearly touches the box, and a nominal action $^0\bar{u} = \bar{q}^\mathrm{a}_0$. The nominal action\footnote{
We use the left superscript to denote how the nominal values change through (i) applying the initial guess heuristics, and (ii) iterations in \Cref{alg:trajopt_ctr}. The left superscript is intended to be different from the right subscript which indexes time steps.
}
is depicted by the yellow circle in \Cref{fig:1d_pushing}a. Even with the help of dynamics smoothing, the slope at this point is still small and thus provides limited guidance for planning.

Applying the initial guess heuristics in \Cref{sec:penetration-finder} brings the robot into contact with the object, yielding ${}^1\bar{u}=0.0$ (magenta square in \Cref{fig:1d_pushing}). The updated ${}^1\bar{u}$ produces the new linearized dynamics (magenta line in \Cref{fig:1d_pushing}a), whose slope more closely matches the ``in contact" piece of the non-smooth contact dynamics. Meanwhile, the dual feasibility constraint \eqref{eq:relaxed-action-contact-trust-region:dual_feasibility} keeps the ball from pulling the box back.

However, the slope at ${}^1\bar{u}$ is still quite different from the slope of the in-contact mode. Under this somewhat inaccurate linear approximation of the true dynamics, we further analyze the behavior of the optimizer depending on the given goal configuration $q^\mathrm{o}_\text{goal}$:
\begin{itemize}
    \item Green zone: If $q^\mathrm{o}_\text{goal}$ is in the green zone, the optimizer will move in a positive $u$ direction, correctly moving the box towards the goal. However, as the slope is smaller than the slope of the non-smooth dynamics, the commanded $u$ will overshoot $q^\mathrm{o}_\text{goal}$. 
    \item Blue zone: If $q^\mathrm{o}_\text{goal}$ is in the blue zone, then the optimizer will move backwards even though the optimal action is to move forward. This is due to the non-physical behavior caused by smoothing: according to the smoothed linearization, when $u=0$ is commanded, the object will be pushed away by the barrier forces to a location further than $q^\mathrm{o}_\text{goal}$. To lessen this effect, the optimizer chooses to move backwards.
    \item Red zone: If $q^\mathrm{o}_\text{goal}$ is in the red zone, the optimizer would move backwards but stop at the boundary of the red region, which corresponds to the leftmost point in the magenta line segments in \Cref{fig:1d_pushing}. Without dual feasibility \eqref{eq:relaxed-action-contact-trust-region:dual_feasibility}, a naive linear model with an ETR will move backwards into the red region to attempt pulling the object.
\end{itemize}

The analysis above suggests that ${}^1\bar{u}$, obtained by applying the initial guess heuristics to ${}^0\bar{u}$, is still not ideal for planning: as ${}^1\bar{u}$ is located at the boundary between two contact modes which have distinct gradients, a single linear model at ${}^1\bar{u}$ approximates neither mode well. 

Interestingly, \Cref{alg:trajopt_ctr} can iteratively improve the local linear approximation with the help of dynamics smoothing. Let us consider running \Cref{alg:trajopt_ctr} to reach $q^\mathrm{o}_\text{goal} = 0.22$ from $\bar{q}_0$ and ${}^1\bar{u}$. At the first iteration, \eqref{eq:linear_trajopt} is solved with the RA-CTR in \Cref{fig:1d_pushing}a. Due to overshooting caused by underestimating the slope of the true dynamics, ${}^2\bar{u} \leftarrow 0.03$. At the second iteration, the optimizer constructs the RA-CTR at ${}^2\bar{u}$, which is shown in \Cref{fig:1d_pushing}b. As ${}^2\bar{u} $ is deeper in penetration, the RA-CTR aligns more closely with $\{u|u\geq 0\}$, the domain of the in-contact piece of the non-smooth dynamics. Moreover, the blue zone under the new RA-CTR, representing the non-physical artifacts caused by smoothing, also shrinks. Under this better local approximation of the non-smooth dynamics, the second iteration returns ${}^3\bar{u} = 0.0202$, bringing $q^\mathrm{o}_+$ to within $0.001\mathrm{m}$ of $q^\mathrm{o}_\text{goal}$.

\end{example}

%% file: body/2d_pushing_example.tex
\begin{example}[\bfseries 2D Ball Pushing a Box]\label{ex:2d-pushing-example}\normalfont
To further illustrate how \Cref{alg:trajopt_ctr} iteratively improves the local linear approximation of the non-smooth contact dynamics, we study a slightly more complex system shown in \Cref{fig:box_ball_graze}a. The system consists of an unactuated box constrained to slide frictionlessly along the $x$ axis, and an actuated ball that can move in the $xy$ plane and make frictional contact with the box's top surface.

The non-smooth contact dynamics of this system consists of 4 contact modes: sticking, sliding left, sliding right, and separation. Each mode can be identified with an affine piece in the Piecewise-Affine (PWA) contact dynamics, which is continuous but has discontinuous gradients. Specifically, we examine $\partial q^\mathrm{o}_+ / \partial u_x$, the component of $\mathbf{B}_\kappa^\mathrm{o} = \partial q^\mathrm{o}_+ / \partial u$ that describes how much the object(box) moves along the $x$ axis relative to how much the ball(robot) moves along the same axis.
As shown in \Cref{fig:box_ball_graze}c, $\partial q^\mathrm{o}_+ / \partial u_x$ is constant within each mode, but is separated from nearby modes by cliffs. The gradient can be made continuous with dynamics smoothing, which is shown in \Cref{fig:box_ball_graze}d.


Consider moving the box to $q^\mathrm{o}_\text{goal} = 0.2$ from the initial configuration in \Cref{fig:box_ball_graze}b: $q^\mathrm{o}_0 = 0$ and $q^\mathrm{a}_0 = (x^\mathrm{a}_0, y^\mathrm{a}_0) = (0, 0.03)$.
Starting with ${}^0\bar{u} = q^\mathrm{a}_0$ (yellow circle in \Cref{fig:box_ball_graze}) puts the system in the ``separation" mode, but accomplishing the task requires the ball to push down and drag the object to the right with friction, which is best done in the ``sticking" mode.

We now show that \Cref{alg:trajopt_ctr} is able to iteratively nudge $\bar{u}$ to the ``sticking" mode. Similar to \Cref{ex:pushing-example}, we first apply the initial guess heuristics, bringing ${}^0\bar{u}$ to ${}^1\bar{u}=(0, 0)$ (magenta square). 
Looking at the smoothed gradient landscape in \Cref{fig:box_ball_graze}d, we see that the smoothed $\partial q^\mathrm{o}_+ / \partial u_x$ climbs up a lot from $\bar{u} = q^\mathrm{a}_0$ to $\bar{u} = (0, 0)$. 
Furthermore, after running \Cref{alg:trajopt_ctr} for 2 iterations, the nominal action ${}^3\bar{u}$ (green star) reaches the yellow plateau. This sequence of $\bar{u}$'s is plotted in \Cref{fig:box_ball_graze}d. 

Plotting the same $\bar{u}$ sequence over the non-smooth gradient landscape in \Cref{fig:box_ball_graze}c reveals that $\bar{u}$ jumps from the ``separation" mode at the bottom to the ``sticking" mode at the top. Without smoothing, mode switches can also be achieved by encoding modes with integer variables and solving the resulting MIP. However, for systems with more contact modes, this approach becomes prohibitively expensive.

Lastly, we plot the A-CTR and the RA-CTR for the above sequence of $\bar{u}$ in \Cref{fig:box_ball_graze}e. Here we use $\mathbf{\Sigma} = \mathrm{diag}(0.05^{-2}, 0.05^{-2})$, which bounds $\delta u$ in a circle of radius $0.05$. 
The curved part of the trust region boundaries arises from the ellipsoidal constraint \eqref{eq:action-contact-trust-region:ellipsoidal}, whereas the straight part follows from the feasibility constraints \eqref{eq:action-contact-trust-region:primal_feasibility} and \eqref{eq:action-contact-trust-region:dual_feasibility}.
Compared with the RA-CTR at ${}^0\bar{u}$ (dark yellow), the RA-CTR at ${}^1\bar{u}$ (magenta), obtained from applying the initial guess heuristics to ${}^0\bar{u}$, approximates the sticking region much better. Furthermore, as ${}^2\bar{u}$ (dark green) gets deeper into the ``sticking" mode, the straight edge of its RA-CTR aligns more closely with the boundary between sticking and sliding. Finally, compared to the RA-CTRs, the non-relaxed A-CTRs are always smaller. As a result, \Cref{alg:trajopt_ctr} would need more iterations to converge if the RA-CTR constraint in \eqref{eq:linear_trajopt:ctr} were replaced by A-CTRs. 

\end{example}

%% file: body/id-example-iiwa.tex
\begin{example}[\bfseries MPC on \code{IiwaBimanual}]\label{ex:bimanual-iiwa}\normalfont
Consider the planar \code{IiwaBimanual} system where each iiwa arm has 4 of the 7 available DOFs locked, constraining its motion to the $xy$ plane. The system is tasked with moving a cylindrical object to a desired location. From the initial configuration $q^\mathrm{o}_0=(q^\mathrm{o}_x,q^\mathrm{o}_y,q^\mathrm{o}_\theta)=(0.65,0,0)$, we command a relative large rotation to reach $q^\mathrm{o}_\text{goal}=(0.65,0.1,5\pi/6)$. The controller is rolled out in closed-loop for $35$ steps, and the resulting trajectories are displayed in \Cref{fig:inverse-dynamics-modes}.

\begin{figure}[t]
\centering\includegraphics[width = 0.48\textwidth]{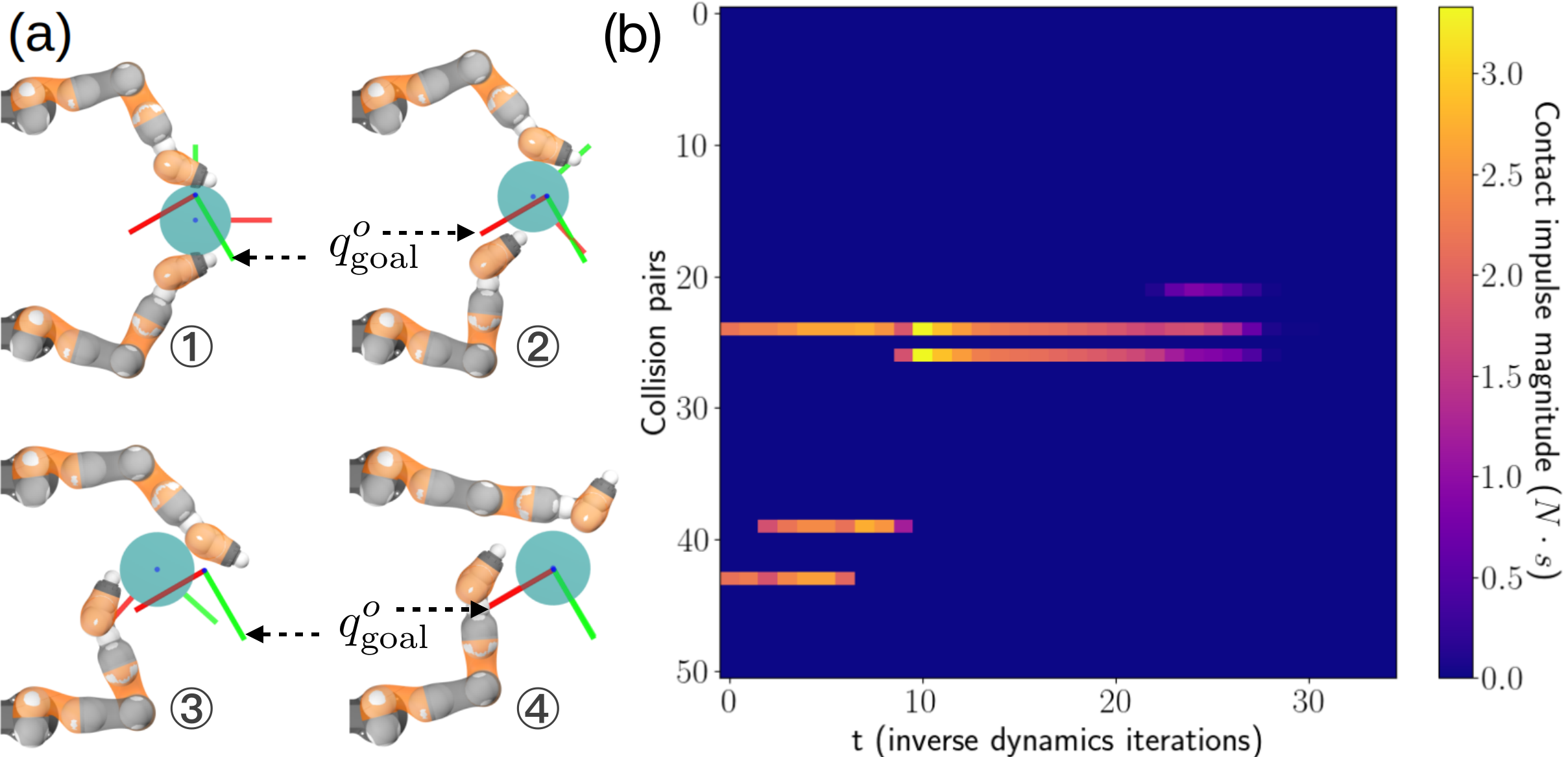}
\caption{
Visualization for \Cref{ex:bimanual-iiwa}. (\textbf{a}) Visualization of key frames of the closed-loop sequence, with the current $q^\mathrm{o}$ and the goal $q^\mathrm{o}_\text{goal}$ displayed as frames. (\textbf{b}) Visualization of contact forces of every potential collision pair in the \code{IiwaBimanual} system.} 
\label{fig:inverse-dynamics-modes}
\vskip -0.1 true in
\end{figure}
\end{example}

The contact forces in \Cref{fig:inverse-dynamics-modes}b tells us that the controller is able to scalably search across local changes in contact modes and arrive at a different contact sequence than the initial configuration.

%% file: body/mpc-qsim-results.tex
\section{Local MPC Experiments}\label{sec:local-planning-results}
In this section, we evaluate the performance of the CTR-based MPC in \Cref{alg:mpc} under the CQDC dynamics constraints. We are particularly interested in answering the following questions:
\begin{itemize}
\item How does R-CTR compare with CTR and ETR?
\item Does the method successfully reach a diverse set of goals on complex systems such as dexterous hands? 
\end{itemize}

To answer these questions and demonstrate the scalability of our method, we conduct statistical analysis on two contact-rich robotic systems: 
\begin{itemize}
    \item the planar \code{IiwaBimanual} system (\Cref{fig:tasks_collision_geometry}-top and \Cref{ex:bimanual-iiwa}), which comprises 3 unactauted DOFs, 6 actuated DOFs and 29 collision geometries. Both the arms and the object are constrained to move in the $xy$ plane, with gravity pointing along the negative $z$ direction. The object is a cylinder measuring $0.28\mathrm{m}$ in diameter. The task is to rotate the object to target $SE(2)$ poses. Each robot arm's collision geometries consist of 14 spheres; the object collision geometry is a single cylinder. 
    \item the 3D \code{AllegroHand} system (\Cref{fig:tasks_collision_geometry}-bottom), which comprises 6 unactauted DOFs, 16 actuated DOFs and 39 collision geometries. The object, a $6\mathrm{cm}$ cube, is unconstrained. The hand's wrist is welded to the world frame. The task is to reorient the cube to target $SE(3)$ poses. The hand has 30 collision geometries, which consist of a collection of boxes, spheres and cylinders; the object collision geometries include a cube and 8 small spheres at the cube's corners. 
\end{itemize}

\begin{figure}
\centering\includegraphics[width = 0.45\textwidth]{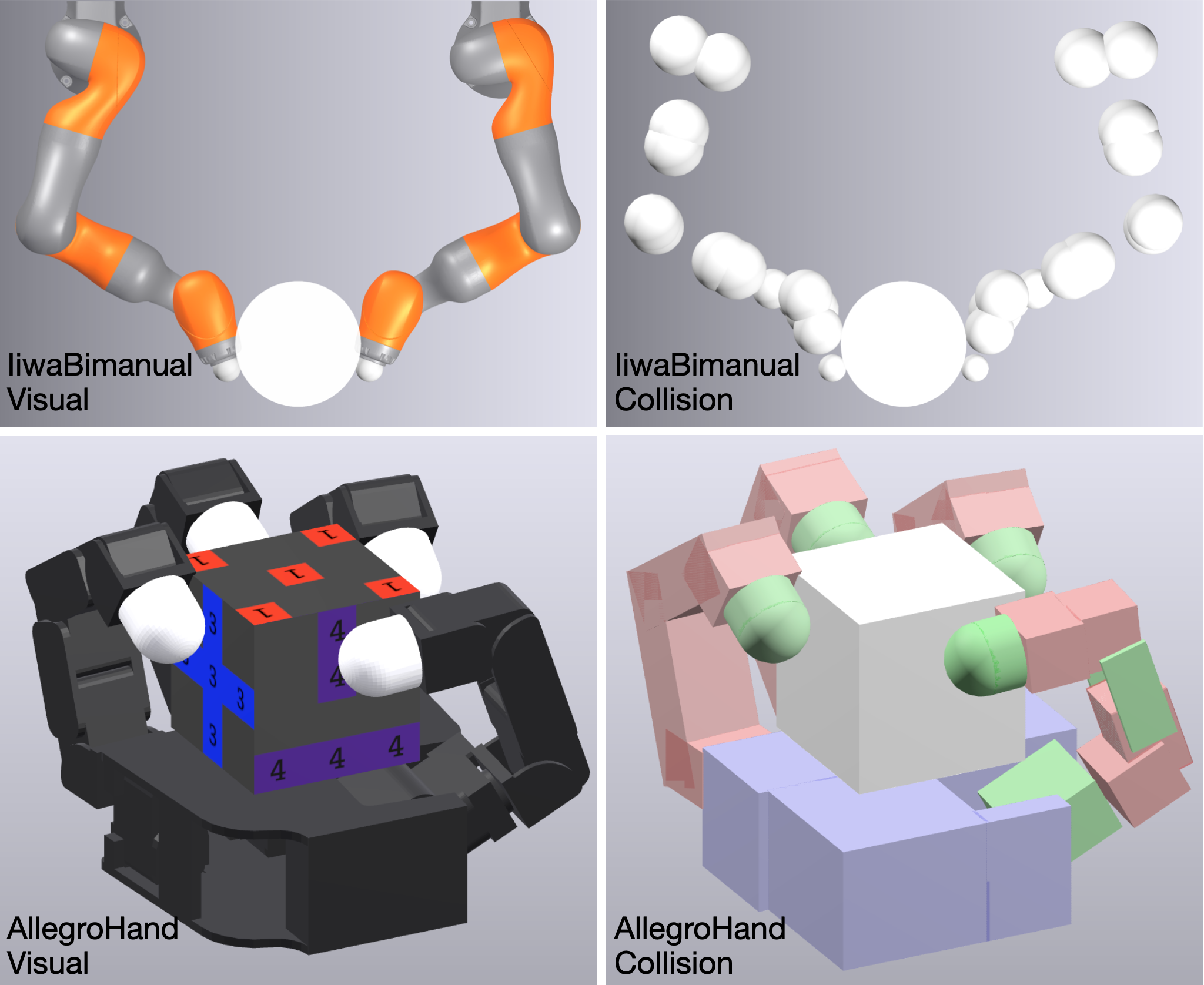}
\caption{Visual and collision geometries of \code{IiwaBimanual} (top) \code{AllegroHand} (bottom). Every \code{IiwaBimanual} geometry is shown as white, while the robot collision geometries in \code{AllegroHand} are colored to increase contrast.} 
\label{fig:tasks_collision_geometry}
\vskip -0.2 true in
\end{figure}

\subsection{Experiment Setup} 
\subsubsection{Goal Generation \label{sec:local-planning-results:goal_selection}} 
When evaluating the proposed MPC statistically, the success rate depends on \emph{which goals} are commanded from \emph{which initial configurations}. 
For local optimization, selecting these pairs is nontrivial: goals that are too easy are uninformative, while those requiring highly non-local movements are beyond the scope of local stabilization.

For both the \code{IiwaBimanual} and \code{AllegroHand} systems, we generate about 1000 pairs of initial conditions and goals which are locally reachable yet far enough to be challenging for MPC. Details about the goal generation scheme can be found in \Cref{appendix:goal_selection}.

The sampled pairs are shown in the scatter plots in \Cref{fig:planning-comparison}, which exhibit a broad distribution, indicating a high level of diversity in the samples. Moreover, as we only care about moving the object to desired configurations, we do not generate robot goal configurations.

\subsubsection{Evaluation Metrics} \label{sec:local-planning-results:evaluation_metrics}
We evaluate the performance of MPC by comparing the difference between $q^\mathrm{o}_\text{final}$, the final object configuration reached by the by \Cref{alg:mpc}, and $q^\mathrm{o}_\text{goal}$, the goal object configuration. We split the object configuration $\qu$ into a quaternion $Q$ and a position $p$: $\qu \coloneqq (Q, p)$, both expressed relative to the world frame. We divide the error in $\qu$ into a \emph{translation error} $\|p_\text{goal} - p_\text{final}\|$, and an \emph{rotation error} $\Delta \theta (Q_\text{goal}, Q_\text{final})$, where $\Delta \theta(\cdot, \cdot)$ returns the angular difference between two unit quaternions.

\subsubsection{Implementation Details}
The numerical experiments are run on a M2 Max Macbook Pro with 64GB of RAM. We use the same open-sourced implementation of the CQDC dynamics as in \citet{gqdp}. The convex subproblem \eqref{eq:linear_trajopt} in \Cref{alg:trajopt_ctr} is solved with \citet{mosek}.

\subsection{Effect of Trust Region Radius $\mathbf{\Sigma}$ and Rollout Horizon $H$}
\label{sec:local-planning-results:hyperparameters}
In this section, we explore how the two parameters affect the performance of our MPC scheme on both the \code{IiwaBimanual} and \code{AllegroHand} systems. We also study how these parameters interplay with different trust region formulations, including R-CTR, CTR and ETR. We measure performance in terms of average tracking errors for the goals selected in \Cref{sec:local-planning-results:goal_selection}.

As we sweep the two parameters, we fix the planning horizon $T = 1$ in \Cref{alg:mpc}, since increasing $T$ does not improve performance (see \Cref{appendix:local-planning-results:planning_horizon} for details). 
Furthermore, we use $n_\mathrm{max}= 2$ \code{IiwaBimanual} and $n_\mathrm{max}=3$ for \code{AllegroHand}. 

\subsubsection{Trust Region Radius}
\label{sec:trust_region_radius}
How does the trust region radius $\mathbf{\Sigma}$ affect the performance of R-CTR, CTR and ETR? To answer this question, we run \Cref{alg:mpc} on both systems with different trust region radii. We set $\mathbf{\Sigma}$ to a constant scaling of the identity matrix, i.e. $\mathbf{\Sigma} = r^{-2} \mathbf{I}$ where $r \in \R_+$ is the radius. The results are illustrated in \Cref{fig:total_errors_plot}a and b, where we measure performance by the average tracking errors for the goals generated in \Cref{sec:local-planning-results:goal_selection}.

On both systems, we observe less performance difference between the trust region variants
as the radius $r$ shrinks. This is because the feasibility constraints \eqref{eq:full_contact_trust_region:primal_feasibility} and \eqref{eq:full_contact_trust_region:dual_feasibility} would less likely become active, simply due to the ellipsoid in \eqref{eq:full_contact_trust_region:etr} being small. However, smaller $r$ also leads to larger tracking errors, as the MPC lacks sufficient authority to progress towards the goal. As $r$ increases, performance initially improves but then starts degrade as $r$ becomes too large, allowing the MPC solution may drift too far from the linearization point. 

For \code{IiwaBimanual}, we find that R-CTR performs the best among the trust region variants: ETR can be overly relaxed, while CTR can be somewhat restrictive. In contrast, on \code{AllegroHand}, all trust region variants behave more similarly, suggesting that \code{AllegroHand} may activate the feasibility constraints less frequently than \code{IiwaBimanual}.
We hypothesize that \code{IiwaBimanual} more frequently encounters the unilateral contact scenario in \Cref{fig:motion-set-visualization}a, whereas \code{AllegroHand} typically operates in the bilateral contact regime depicted in \Cref{fig:three_spheres_ctr}a.

\begin{figure}
\centering\includegraphics[width = 0.48\textwidth]{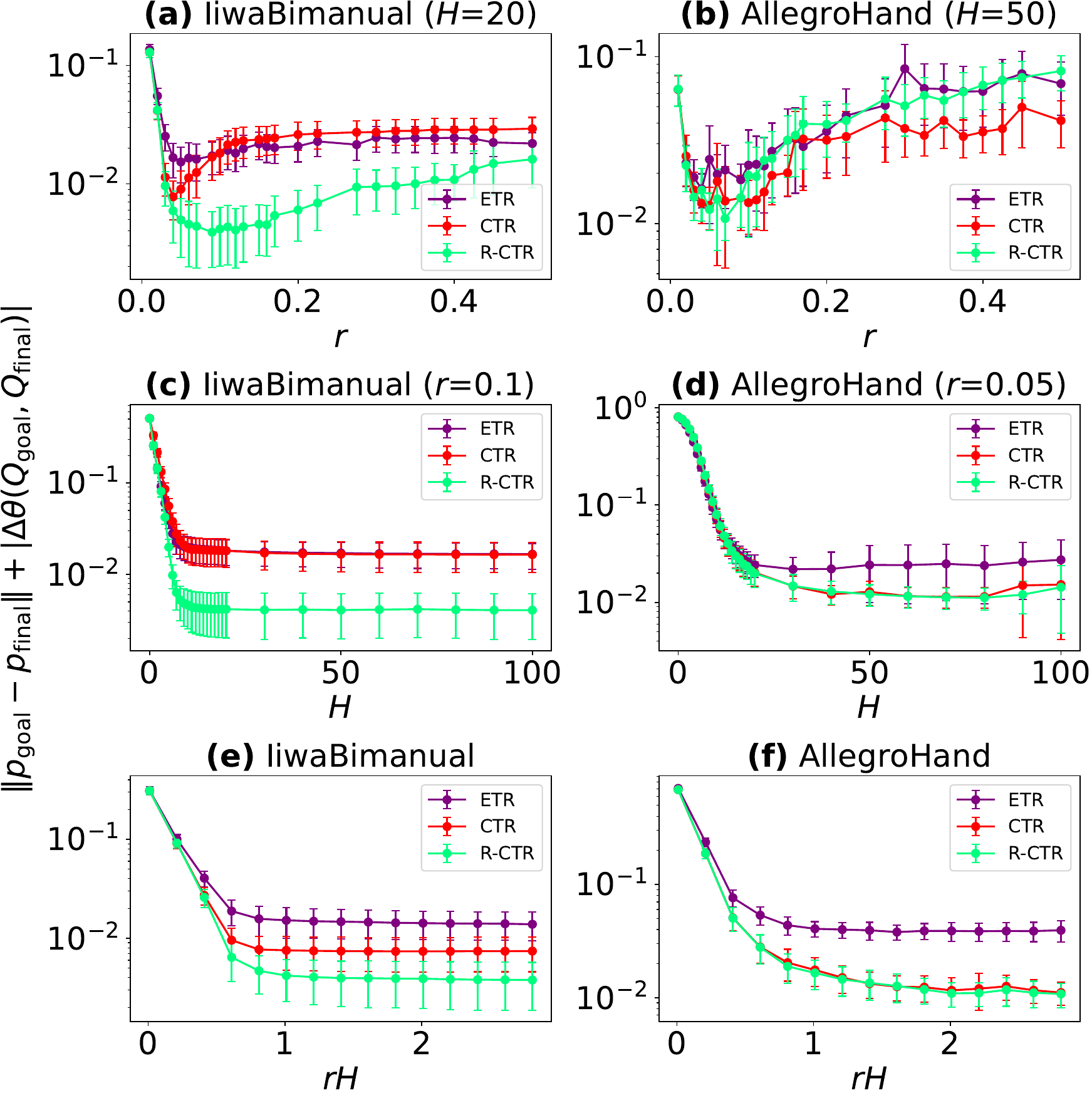}
\caption{How tracking errors, defined as the sum of translation and rotation errors, change with different trust region radius $r$ ($\mathbf{\Sigma} = r^{-2}\mathbf{I}$) and rollout horizon $H$. The plotted curves denote the average error for the goals generated in \Cref{sec:local-planning-results:goal_selection}, and the error bars denote $0.1\sigma$ where $\sigma$ is the standard deviation. For plots (e) and (f), the plotted curves denote the best average error across all combinations of $r$ and $H$ with constant $r \cdot H$.}
\label{fig:total_errors_plot}
\vskip -0.2 true in
\end{figure}

\subsubsection{Rollout Horizon $H$} In \Cref{fig:total_errors_plot}(c) and (d), we examine how performance changes as the MPC rollout horizon $H$ increases. As expected, the tracking errors decrease initially and then plateau for both systems. 
For \code{IiwaBimanual}, R-CTR again outperforms ETR and CTR in terms of average tracking error. On \code{AllegroHand}, both CTR-based approaches display comparable performance; although they still outperform ETR, their respective error bars overlap significantly.

For large $H$ ($\geq 80$), while \code{IiwaBimanual}'s error curves remain stable, \code{AllegroHand}'s curves trend slightly up with larger variances. We attribute this to numerical instability: since \code{AllegroHand} has an order of magnitude more collision pairs than \code{IiwaBimanual}, numerical issues in collision detection and CQDC dynamics evaluation are more likely to occur. Such issues sometimes destabilize MPC, terminating a run with large final errors.

\subsubsection{$r \cdot H$: ``Effective Lookahead''} \label{sec:constant_rh}
The product of the trust region radius $r$ and the planning horizon $H$  represents the \emph{effective lookahead} of MPC. A short, aggressive plan (low $H$, high $r$) might change the system state as much as a long, cautious one (high $H$, low $r$), giving them a similar reach into the future state of the system. For a fixed $r\cdot H$, we expect a more ``informative'' trust region can better guide the system towards the goals, resulting in lower final tracking errors.

In \Cref{fig:total_errors_plot}(e) and (f), we examine how the errors change as a function of the effective lookahead. 
To fairly measure the average error at a fixed value of $r \cdot H$, each data point on these plots represents the best performance achieved for a given $r \cdot H$ value, found by varying $r$ and $H$ while holding $r \cdot H$ constant.

For sufficiently large $r\cdot H$, R-CTR and CTR significantly outperform ETR, suggesting that the incorporation of the feasibility constraints provides better guidance to MPC, even when ETR is used cautiously (low $r$ and high $H$). 

\begin{figure*}[t]
\centering
\subfloat[\centering Iiwa Bimanual]{
    \includegraphics[width=0.98\textwidth]{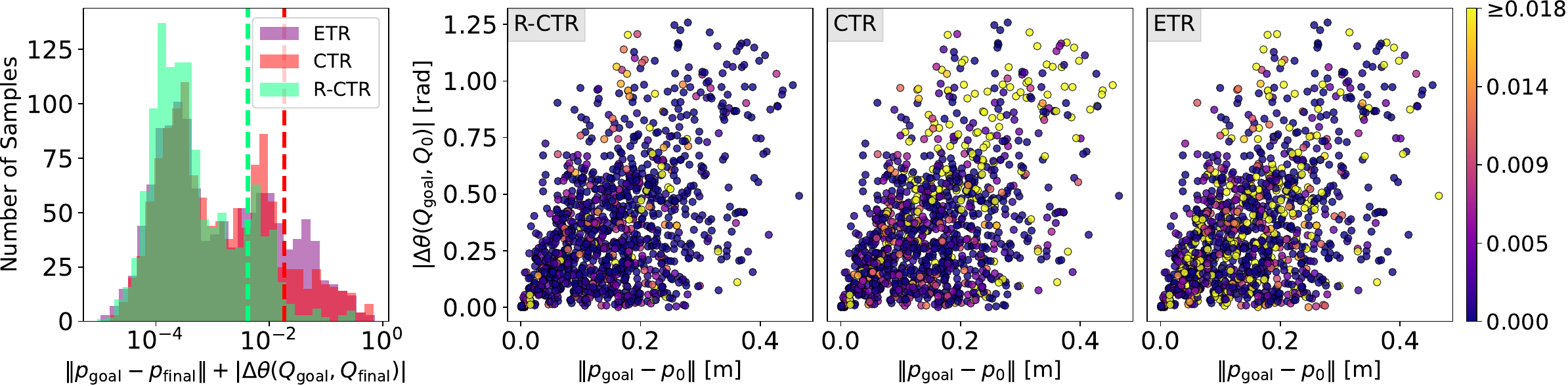}
    \label{fig:planning-comparison:iiwa}
}
\hfill
\subfloat[\centering Allegro Hand]{
    \includegraphics[width=0.98\textwidth]{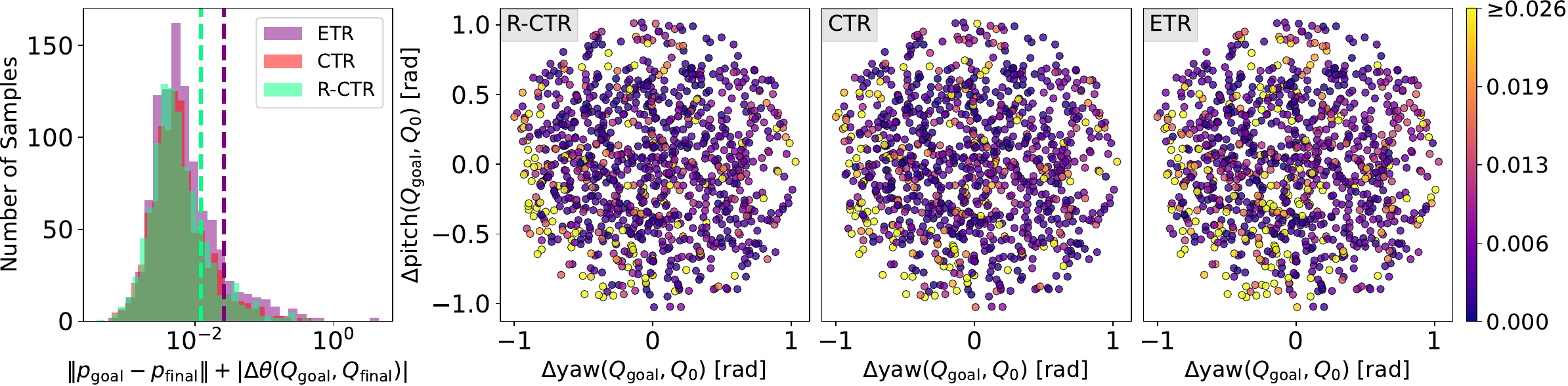}
    \label{fig:planning-comparison:allegro}
}

\caption{Results for running \Cref{alg:mpc} with R-CTR, CTR and ETR constraints on the pairs of goals and initial conditions generated in \Cref{sec:local-planning-results:goal_selection}.
The histogram shows how many samples reached a specified error in the $x$ axis, where error is defined as $\|p_\text{goal} - p_\text{final}\| + |\Delta \theta (Q_\text{goal}, Q_\text{final})|$. 
The dotted lines in the histogram represent the average errors.
Some dotted lines are barely visible due to overlapping (their values are available in \Cref{tab:planning-errors}).
Each dot in the scatter plot corresponds to a pair of an initial condition $q^\mathrm{o}_0$ and a goal $q^\mathrm{o}_\text{goal}$.
Color of the dots is determined by the error: warmer colors indicate bigger errors.
The three scatter plots in each row share the same color scale, whose maximum value is the average of the errors for reaching the goals with ETR. 
(\textbf{a}) For \code{IiwaBimanual}, the $x$ axis of the scatter plots is the position difference between $q^\mathrm{o}_0$ and $q^\mathrm{o}_\text{goal}$, the $y$ axis the angular difference.
(\textbf{b}) For \code{AllegroHand}, as the position difference is small for all goals, the $x$ axis shows the yaw angle difference between $q^\mathrm{o}_0$ and $q^\mathrm{o}_\text{goal}$, and the $y$ axis the pitch angle difference. The roll, pitch and yaw axes are defined in \Cref{fig:allegro_rollout_failures}a.}
\label{fig:planning-comparison}
\vskip -0.15 true in
\end{figure*}

\subsection{Can Our Method Reach Goals?}
Finally, for all initial condition and goal pairs generated in \Cref{sec:local-planning-results:goal_selection}, we run \Cref{alg:mpc} and analyze the results using the metrics in \Cref{sec:local-planning-results:evaluation_metrics}.

\Cref{tab:planning-hyperparameters-allegro-and-iiwa} summarizes a subset of MPC hyperparameters which are chosen based on the discussion in \Cref{sec:local-planning-results:hyperparameters}. Additional hyperparameters are summarized in \Cref{app:hyperparameters_for_lcoal_mpc}.

\begin{table}[h]
\centering
\small
\begin{tabular}{@{}lcccc@{}}
\toprule
System & $T$ & $n_{\max}$ & $r$ & $H$ \\
\midrule
\code{IiwaBimanual} & 1 & 2 & 0.10 & 20 \\
\code{AllegroHand}  & 1 & 3 & 0.05 & 50 \\
\bottomrule
\end{tabular}
\caption{MPC hyperparameters: $T$ and $n_\mathrm{max}$ are defined in \Cref{alg:trajopt_ctr}; $r$ is the trust region radius ($\mathbf{\Sigma} = r^{-2}\mathbf{I}$); $H$ is the MPC rollout horizon in \Cref{alg:mpc}.}
\label{tab:planning-hyperparameters-allegro-and-iiwa}
\vskip -0.15 true in
\end{table}


\vskip -0.2 true in
\subsubsection{Goal Tracking Performance}
As shown in \Cref{fig:planning-comparison}, MPC with R-CTR achieves the lowest average error on both \code{IiwaBimanual} and \code{AllegroHand}. As \Cref{fig:planning-comparison} combines rotation and translation errors, we present them separately in \Cref{tab:planning-errors} to better illustrate their relative magnitudes. Notably, MPC with R-CTR also achieves the lowest variance. 
\begin{table}[h]
\centering
\small                 
\begin{tabular}{@{}lccccc@{}} 
\toprule
 & \multicolumn{2}{c}{\code{IiwaBimanual}} &
   \multicolumn{2}{c}{\code{AllegroHand}} \\
 \cmidrule(lr){2-3}\cmidrule(lr){4-5}
 & Trans.\,[mm] & Rot.\,[mrad] & Trans.\,[mm] & Rot.\,[mrad] \\
\midrule
R-CTR & 2.0\,(11.5) & 2.1\,(10.1) & 2.2\,(5.7) & 9.8\,(26.9) \\
CTR   & 9.4\,(30.9) & 8.9\,(33.3) & 2.2\,(5.6) & 9.9\,(28.3) \\
ETR   & 8.6\,(23.8) & 9.6\,(37.9) & 4.5\,(53.8) & 21.4\,(136.4) \\
\bottomrule
\end{tabular}
\caption{Mean translation and rotation errors for the combined errors in \Cref{fig:planning-comparison}. Each cell displays the mean (std).}
\label{tab:planning-errors}
\vskip -0.1 true in
\end{table}

Upon a closer examination of the \emph{distribution} of errors in \Cref{fig:planning-comparison}, a key advantage of R-CTR is that it produces far fewer trajectories with larger errors. 
For \code{IiwaBimanual}, this advantage is significant: the R-CTR distribution is more tightly clustered at low error values and lacks the heavy tail in CTR and ETR's error distributions.
For \code{AllegroHand}, the result is more nuanced: R-CTR and CTR perform almost identically, but both are a small yet discernible shift towards the left compared to ETR.

We believe the benefit of R-CTR is the most significant when the dual feasibility constraints substantially restrict their corresponding larger ETR. 
As indicated by the examples in \Cref{sec:local-approximation}, this restriction is pronounced in the unilateral contact scenarios (\Cref{ex:pushing-actr} and \Cref{ex:planar_hand_one_finger}), but subtle or even non-existent when the contact points form a stable, force-closure grasp (e.g. \Cref{ex:squeezing-actr}). 

This hypothesis is corroborated by the example trajectories in \Cref{app:local_mpc_examples}.
In the ETR failure cases (\Cref{fig:iiwa_planning_examples:etr_failure} and \Cref{fig:allegro_planning_examples:etr_failure}),
the robots lack a force-closure grasp and cannot move the object in certain directions. The dual feasibility constraint in R-CTR correctly identifies these limitations and prevents the extraneous actions that cause the ETR baseline to fail.
In contrast, when the robots have an effective force-closure grasp on the object (\Cref{fig:iiwa_planning_examples:etr_success} and \Cref{fig:allegro_planning_examples:etr_success}), including the dual feasibility constraints is less crucial for success.

Another key question is why the primal feasibility constraint makes CTR overly restrictive on \code{IiwaBimanual} (hence the larger tracking errors), while having almost no negative effect on \code{AllegroHand} when compared to R-CTR.
We hypothesize that the impact of this constraint is highly dependent on the system's kinematics and the specific contact configuration. As our examples illustrate, the primal feasibility constraint can sometimes severely reduce the trust region volume (e.g. \Cref{fig:box_ball_graze}e) but at other times barely alters it (e.g. the blue nominal configuration in \Cref{fig:three_spheres_ctr}). This suggests the grasping configurations on \code{AllegroHand} are less affected by this particular constraint. A deeper analysis of the role of primal feasibility in different contact scenarios could be an interesting topic for future research.

\subsubsection{Infeasibility}
Encountering infeasible constraints is a common issue that plagues iterative methods for solving the non-convex trajectory optimization problem \eqref{eq:nonlinear-to}. 
However, \Cref{alg:mpc} rarely suffers from this problem.
As shown in \Cref{tab:infeasile_trust_retions}, despite each goal requiring the computation of tens ($T \cdot H \cdot n_\mathrm{max}$) of unique trust regions, the empirical rate of infeasibility is very low. As a result, when generating the tracking results in \Cref{fig:planning-comparison}, we simply terminate MPC prematurely if an infeasible trust region is encountered.
\begin{table}[h]
\centering
\small
\begin{tabular}{lcccc}
\toprule
     &  ETR & CTR & R-CTR & Number of Goals\\
\midrule
\code{IiwaBimanual} & 0 & 0 & 0 & 1233\\
\code{AllegroHand} & 5 & 2 & 3 & 1000\\
\bottomrule
\end{tabular}
\caption{Number of infeasible trust regions encountered while running \Cref{alg:mpc} to reach the goals from \Cref{sec:local-planning-results:goal_selection}.}
\label{tab:infeasile_trust_retions}
\end{table}

\subsubsection{Computation Time}
\label{sec:local-planning-results:reaching_goals:computation_time}
We report wall clock time for the most time-consuming steps in \Cref{alg:mpc}, shown in \Cref{tab:planning_time}. For \code{AllegroHand}, $\mathbf{CtrTrajOpt}$ is slower due to its high-dimensional state and action space, as well as larger number of collision pairs. Among the three trust region variants, ETR is the fastest due to the absence of feasibility constraints. In contrast, CTR is the slowest, as it enforces both primal and dual feasibility. 

There is significant room for improvement for the computation time, which we leave as future work. For instance, we can (i) exploit common sparsity patterns in MPC programs \cite{crocoddyl}, and (ii) prune the trust region constraints that are too far from making contact (our current trust region construction is very conservative, see \Cref{app:hyperparameters_for_lcoal_mpc} for a more detailed discussion). 
\vskip -0.1 true in
\begin{table}[h]
\centering
\small
\setlength{\tabcolsep}{3pt} 
\renewcommand{\arraystretch}{1.0} 
\begin{tabular}{@{}llcc@{}}
\toprule
\multicolumn{2}{l}{Time unit [ms]} &
\code{IiwaBimanual} & \code{AllegroHand} \\
\midrule
Per IGH &                      & 16.8\,(0.05) & 32.0\,(0.6) \\
\midrule
\multirow{3}{*}{Per $\mathbf{CtrTrajOpt}$}
          & R-CTR              &  5.1\,(0.4)  &  99.2\,(28.6) \\
          & CTR                &  7.0\,(0.2)  & 205.7\,(32.5) \\
          & ETR                &  2.4\,(0.2)  &  52.9\,(21.8) \\
\midrule
Per $q_{t+1}=f(q_t,u_0^\star)$ &
                           &  0.9\,(0.05) &   5.92\,(0.38) \\
\bottomrule
\end{tabular}
\caption{Breakdown of the average runtime (wall clock) for \Cref{alg:mpc}. IGH stands for Initial Guess Heuristic. Each cell displays the mean (standard deviation).}
\label{tab:planning_time}
\vskip -0.15 true in
\end{table}

%% file: body/mpc-hardware-results.tex
\section{Stabilization under Second-Order Dynamics}\label{sec:stabilization-results}
In \Cref{sec:local-planning-results}, we presented results for running MPC on the CQDC dynamics \eqref{eq:q_dynamic_f}. However, although the CQDC dynamics removes the force-at-a-distance effect from smoothing, it is still different from real contact physics: not only are second-order effects ignored, it also introduces a small gap between objects undergoing sliding friction \cite[Section IV-A2]{pangsimulator}, which is an artifact known as ``hydroplaning" \cite[Appendix B]{underactuated} shared by contact dynamics formulations that utilize Anitescu's convex relaxation for contact dynamics \cite{anitescu}. 

In this section, we would like to understand how the gap between the CQDC and real physics affect MPC performance in more realistic settings. Specifically, we focus on answering the question: can the MPC controller using the CQDC dynamics perform closed-loop stabilization (i) on high-fidelity simulated second-order dynamics, and (ii) on hardware?

\subsection{Implementation Heuristics for Second-order Dynamics}
Fast forwarding to the goal-tracking results in \Cref{fig:second-order-mpc-comparison}, we can see the performance degradation caused by the physics gap is significant. Rather than re-implementing \Cref{alg:mpc} with second-order dynamics and a contact model without the hydroplaning artifact, we would like to understand how far we can push CQDC-based MPC through practical adjustments.
Specifically, we introduce two modifications to \Cref{alg:mpc} which have empirically boosted MPC performance on second-order dynamics (which we denote by the symbol $f_\text{2nd}$)
\footnote{With the term ``second-order dynamics" and the symbol $f_\text{2nd}$, we refer to both simulated second-order dynamics and hardware dynamics.}. 
The modified MPC algorithm is summarized in \Cref{alg:mpc_real_dynamics}.

Firstly, instead of running MPC as quickly as possible, we compute a longer trajectory (usually 1-2 seconds) using \Cref{alg:mpc} for multiple steps ($H > 1$) (\Cref{alg:mpc_real_dynamics:call_mpc}), roll out $L_u$ in its entirety on the second-order dynamics (\Cref{alg:mpc_real_dynamics:rollout}), and repeat. 
This resembles the ``action chuncking'' scheme ubiquitous in robot learning, which has been shown to reduce compounding errors during rollouts when the policy (or controller), rather than the open-loop-stable dynamics of the robot-object system, is the primary contributor to compounding errors~\cite{zhang2025imitation}. 
In our case, the MPC controller indeed injects compounding errors due to the gap between the CQDC model used for planning and the second-order nature of $f_\text{2nd}$.
See \Cref{app:2nd_order_mpc:hyperparameters} for more details about the effect of $H$ on closed-loop tracking performance.

Secondly, we apply the initial guess heuristics more frequently to compensate for the side effects of the hydroplaning artifact.
When running trajectories generated by \Cref{alg:mpc} open-loop on second-order dynamics, we observed that the robot sometimes loses contact with the object. We hypothesized that this is caused by the hydroplaning artifact of the CQDC dynamics, which allows the robot to exert sliding friction forces on the object without actually touching it. As the hydroplaning artifact does not exist for sticking friction \cite{anitescu}, we can alleviate hydroplaning by regularly pulling the robot back into contact with object using the initial guess heuristics (\Cref{sec:penetration-finder}). Specifically, in \Cref{alg:mpc}, instead of only applying the initial guess heuristics at the first iteration (\Cref{alg:mpc:begin_if} to \Cref{alg:mpc:later_initialization}), we apply it at every iteration to the corresponding $q_t^\mathrm{a}$.

We denote this modified version of MPC with the symbol $\mathbf{MPC}_\text{Proj}$ in \Cref{alg:mpc_real_dynamics:call_mpc} of \Cref{alg:mpc_real_dynamics}, where the subscript denotes the projection back to the contact manifold, which is the effect of applying the initial guess heuristics.
\vskip -0.15 true in
\begin{algorithm}
\caption{MPC with 2nd-order Dynamics $f_\text{2nd}$}\label{alg:mpc_real_dynamics}
\textbf{Input:} Initial state $q_0$, goal state $q_\text{goal}$, planning horizon $T$, iterations limit $n_\mathrm{max}$, MPC rollout horizon $H$, number of re-plans $N$\;
$i \leftarrow 0$\;
\While {$q_\mathrm{goal}$ not reached $\mathrm{and}$ $i < N $} {
    $q_{0:H}^\star, u_{0:H-1}^\star \gets \mathbf{MPC}_\text{Proj}(q_0, q_\text{goal}, T, n_\text{max}, H)$\; \label{alg:mpc_real_dynamics:call_mpc}
    $q_\text{final} \gets$ Apply $u_{0:H-1}^\star$ to $f_\text{2nd}$\; \label{alg:mpc_real_dynamics:rollout}
    $q_0 \gets q_\text{final}$\;
    $i \leftarrow i + 1$
}
\end{algorithm}
\vskip -0.25 true in

\begin{figure*}[t]
\centering
\subfloat[\centering \code{IiwaBimanual}: implementation heuristics ablation, all using R-CTR.]{
    \includegraphics[width=0.98\textwidth]{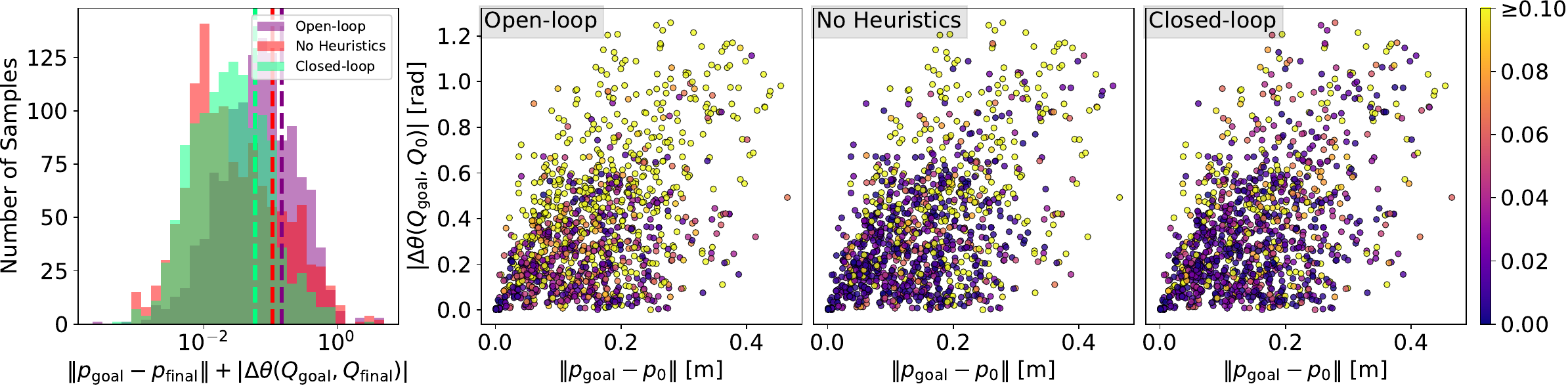}
    \label{fig:second-order-mpc-comparison:iiwa}
}
\hfill
\subfloat[\centering \code{IiwaBimanual}: closed-loop trust region comparison. R-CTR results are the same as the closed-loop results in (\textbf{a}).]{
    \includegraphics[width=0.98\textwidth]{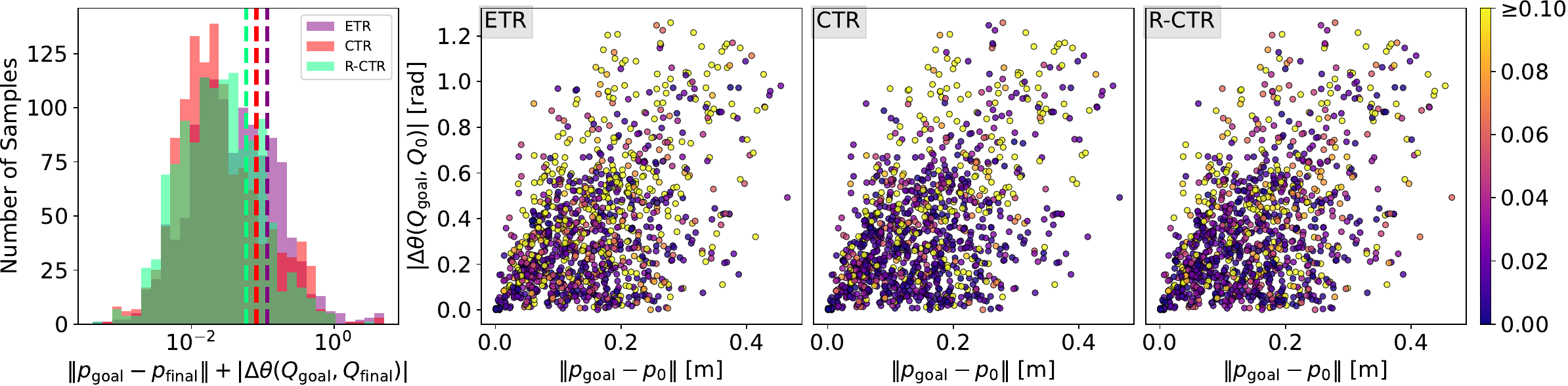}
    \label{fig:second-order-mpc-comparison:iiwa-trust-regions}
}
\hfill
\subfloat[\centering \code{AllegroHand}: implementation heuristics ablation, all using R-CTR.]{
    \includegraphics[width=0.98\textwidth]{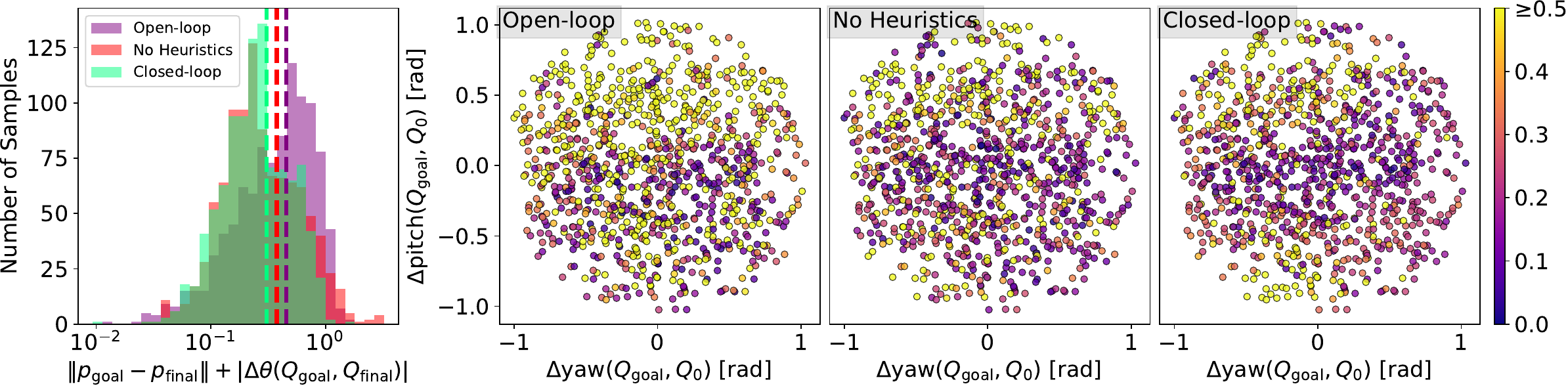}
    \label{fig:second-order-mpc-comparison:allegro}
}
\hfill
\subfloat[\centering \code{AllegroHand}: closed-loop trust region comparison. R-CTR results are the same as the closed-loop results in (\textbf{c}).]{
    \includegraphics[width=0.98\textwidth]{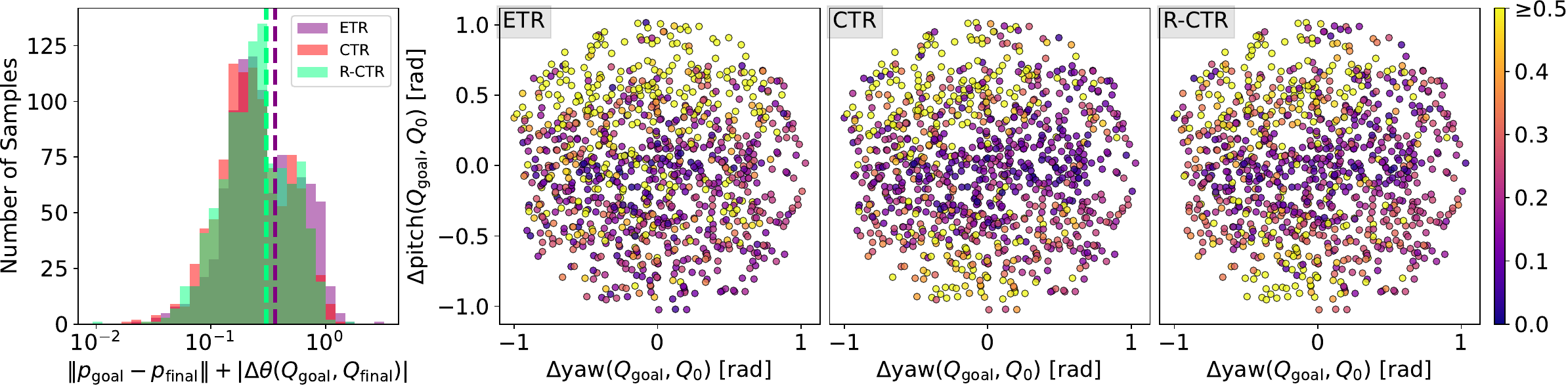}
    \label{fig:second-order-mpc-comparison:allegro-trust-regions}
}
\caption{Results for running the five variants of \Cref{alg:mpc_real_dynamics}, namely open-loop, no heuristics, close-loop, and closed-loop with CTR and ETR, on the pairs of goals and initial conditions generated in \Cref{sec:local-planning-results:goal_selection}. All figure elements follow the same definitions used in \Cref{fig:planning-comparison}. }
\label{fig:second-order-mpc-comparison}
\vskip -0.1 true in
\end{figure*}

\subsection{Experiment Setup}
We evaluate \Cref{alg:mpc_real_dynamics} using the same final-distance-to-goal metric proposed in \Cref{sec:local-planning-results:evaluation_metrics} on both \code{IiwaBimanual} and \code{AllegroHand}. Hyperparameter choices are summarized in \Cref{app:2nd_order_mpc:hyperparameters}. 
For each system, we run four ablation variants of our algorithm, each omitting a crucial component.
\begin{itemize}
    \item \textbf{Closed-loop}. This is running \Cref{alg:mpc_real_dynamics} as is with multiple re-plans ($N > 1$) and R-CTR. 
    \item \textbf{No Heuristics}. We keep using R-CTR, but replace $\mathbf{MPC}_\text{proj}$ with the $\mathbf{MPC}$ defined in \Cref{alg:mpc}. In other words, we apply the initial guess heuristics only at the first iteration in MPC, instead of at every iteration.
    \item \textbf{Open-loop}. We keep using R-CTR, but reduce feedback rate to the extreme: there is no re-planning at all. Instead, we plan a long trajectory that gets the object all the way to the goal and run it open-loop. 
    \item \textbf{Closed-loop with CTR/ETR}. We study how the other trust region variants affect closed-loop performance. 
\end{itemize}

\textbf{Simulation Setup}.
We incorporate the \code{IiwaBimanual} and \code{AllegroHand} systems presented in \Cref{sec:local-planning-results} into Drake \cite{drake}, which employs a complete second-order dynamics model along with state-of-the-art contact solvers \cite{castro2023theory} that do not have the hydroplaning effect. We maintain identical collision geometries, robot controller stiffness, and friction coefficients between the CQDC dynamics and Drake.

\textbf{Hardware Setup}.
To minimize the sim-to-real gap, we carefully aligned key parameters between our simulation model in Drake and the physical system. Specifically, we ensured a close correspondence for: (i) the collision geometries used for planning, (ii) the feedback gains of the robot's low-level controllers, and (iii) the inertial and friction coefficients. This alignment was crucial for the successful transfer of our method to hardware. 


As the focus of this work is state-based planning and control for contact-rich manipulation, we measure the object pose utilizing the OptiTrack motion capture system. Passive spherical reflectors are sufficient for the bucket in \code{IiwaBimanual}. For the cube in \code{AllegroHand}, in order to not alter the collision geometry with external markers, we embed infrared light-emitting active markers in the cube
In order to minimize the effect of marker occlusion from the hand, we adopt a similar marker placement as the cube used by \citet{openai-hand-demo}. 
In addition, we post-process the cube pose by solving an inverse kinematics problem that projects the raw OptiTrack measurement out of penetration with the Allegro hand. In practice, our pose estimation setup proved reliable enough to forgo additional measures for addressing uncertainty.

\subsection{Results \& Discussion}
We plot the results of our experiments in \Cref{fig:second-order-mpc-comparison}, and display the translation and rotation errors separately in \Cref{tab:mpc-errors-2nd-order}. We discuss some of our findings from the experiments.

\begin{table}
\centering
\small
\setlength{\tabcolsep}{3pt}  
\renewcommand{\arraystretch}{1.05} 
\begin{tabular}{@{}lcccc@{}}
\toprule
 & \multicolumn{2}{c}{\code{IiwaBimanual}}
 & \multicolumn{2}{c}{\code{AllegroHand}} \\
\cmidrule(lr){2-3}\cmidrule(lr){4-5}
 & Trans.\,[m] & Rot.\,[rad] & Trans.\,[m] & Rot.\,[rad] \\
\midrule
\multirow{2}{*}{Open-loop Sim}
   & 0.033 & 0.115 & 0.011 & 0.443 \\
   & (0.152) & (0.191) & (0.007) & (0.286) \\
\addlinespace
\multirow{2}{*}{No-Heuristics Sim}
   & 0.041 & 0.066 & 0.015 & 0.358 \\
   & (0.156) & (0.153) & (0.017) & (0.335) \\
\addlinespace
Closed-loop Sim & 0.020 & 0.039 & 0.014 & 0.290 \\
\; \; with R-CTR & (0.065) & (0.079) & (0.009) & (0.197) \\
\midrule
Closed-loop Sim & 0.040 & 0.076 & 0.013 & 0.349 \\
\; \; with ETR & (0.167) & (0.214) & (0.010) & (0.267) \\
\addlinespace
Closed-loop Sim & 0.031 & 0.051 & 0.013 & 0.289 \\
\; \; with CTR & (0.122) & (0.164) & (0.010) & (0.211) \\
\midrule
\multirow{2}{*}{Open-loop HW}
   & 0.016 & 0.056 & 0.014 & 0.323 \\
   & (0.019) & (0.064) & (0.007) & (0.256) \\
\addlinespace
\multirow{2}{*}{Closed-loop HW}
   & 0.013 & 0.024 & 0.018 & 0.258 \\
   & (0.016) & (0.038) & (0.011) & (0.215) \\
\bottomrule
\end{tabular}
\caption{Translation and rotation errors for simulation (Sim) and hardware (HW) experiments. Each cell displays the mean
(standard deviation).}
\label{tab:mpc-errors-2nd-order}
\vskip -0.2 true in
\end{table}

\subsubsection{\textbf{Closed-loop} vs. \textbf{Open-loop}} 
Closed-loop MPC clearly outperforms open-loop on both systems. This suggests that feedback is still crucial to reduce object tracking errors despite the gap between CQDC and second-order dynamics.

\subsubsection{\code{AllegroHand} vs. \code{IiwaBimanual}}
The average errors from closed-loop MPC is much smaller on \code{IiwaBimanual}, which we attribute to the Allegro task's inherently difficulty. As reaching goal poses on the Allegro often requires lifting the cube (\Cref{fig:allegro_rollout_failures}b), any slip can cause the cube to slide back to the palm, resetting progress and resulting in large errors. In contrast, on \code{IiwaBimanual}, the bucket remains on a tabletop, so slipping merely slows progress without eliminating it.

Challenges due to slipping is also evident in the distribution of tracking errors in the scatter plots in \Cref{fig:second-order-mpc-comparison:allegro}. Under the closed-loop MPC, Allegro's tracking error remains low for large yaw angles when pitch angle is near zero, because the cube's bottom face stays close the palm for such goals. In other goal orientations requiring more lift, the cube is less supported by the palm, increasing the risk of losing grip. 

\subsubsection{\textbf{Closed-loop} vs. \textbf{No Heuristics}}
Although the initial guess heuristics only slightly reduces the mean error (see the histograms in \Cref{fig:second-order-mpc-comparison}), it significantly lower the incidence of large errors (i.e., those close to or exceeding 1). Such large errors usually happen when the object falls off the supporting surface (\Cref{fig:iiwa_rollout_failures}d and \Cref{fig:allegro_rollout_failures}d) after contact is lost during MPC. By regularly applying the initial guess heuristics, the robot is constantly pulled back to the object, greatly reducing the chances of losing contact. As demonstrated in \Cref{fig:iiwa_rollout_failures}b and \Cref{fig:allegro_rollout_failures}c, although lost contact occurred and caused some error, the object did not fall, thanks to the initial guess heuristics.

Moreover, applying the initial guess heuristics results in cleaner, shorter trajectories. Once contact is lost during the execution of \Cref{alg:mpc_real_dynamics}(e.g. \Cref{fig:iiwa_rollout_failures}), the subsequent $\mathbf{MPC}$ step (\Cref{alg:mpc_real_dynamics:call_mpc}) typically starts by commanding a large robot movement to reestablish contact. However, under the \textbf{closed-loop} variant of \Cref{alg:mpc_real_dynamics} where the initial guess heuristics is applied more frequently, this movement is significantly smaller.

In fact, for each pair of initial and goal configurations, we can compute the path length traveled by the robot as follows:
\begin{equation}
    \text{Robot Path Length} \;=\; \int_{0}^{t_\text{final}} \|\dot{q}^\mathrm{a}(t)\|_2\, \mathrm{d}t,
\end{equation}
where $t_\text{final}$ is the duration of the experiment and ${q}^\mathrm{a}(t): \R \rightarrow \R[n_\mathrm{a}]$ the robot trajectory. As shown in \Cref{fig:robot_path_length_heuristics_vs_no}, \textbf{closed-loop} MPC results in shorter robot path lengths.

\subsubsection{\textbf{R-CTR} vs. \textbf{ETR/CTR}}
The performance comparison between the different trust region formulations under second-order dynamics follows the same general trend observed in our local MPC experiments (\Cref{fig:planning-comparison}). R-CTR remains the top performer on both systems, achieving the lowest average error and exhibiting a tighter error distribution with fewer large-error outliers.

On \code{IiwaBimanual}, the advantage of R-CTR is more significant, showing a larger improvement over both the more restrictive CTR and the baseline ETR. For \code{AllegroHand}, the result is more nuanced: while R-CTR and CTR perform almost identically, both offer a clear improvement over ETR, which produces a visibly wider error distribution.

\subsubsection{Hardware Experiments}
Using voxel-grid down-sampling, we selected a subset of 50 uniformly-spaced pairs of initial states and goals for each system from the pairs used for simulation experiments. 
As shown in the final two rows of \Cref{tab:mpc-errors-2nd-order}, the average errors of hardware experiments are highly comparable to our simulation results. 
This strong sim-to-real correspondence is expected, given our careful alignment of the simulation and hardware environments.
All hardware and simulation videos are available in an interactive plot on our project website.

\begin{figure}
\centering\includegraphics[width = 0.48\textwidth]{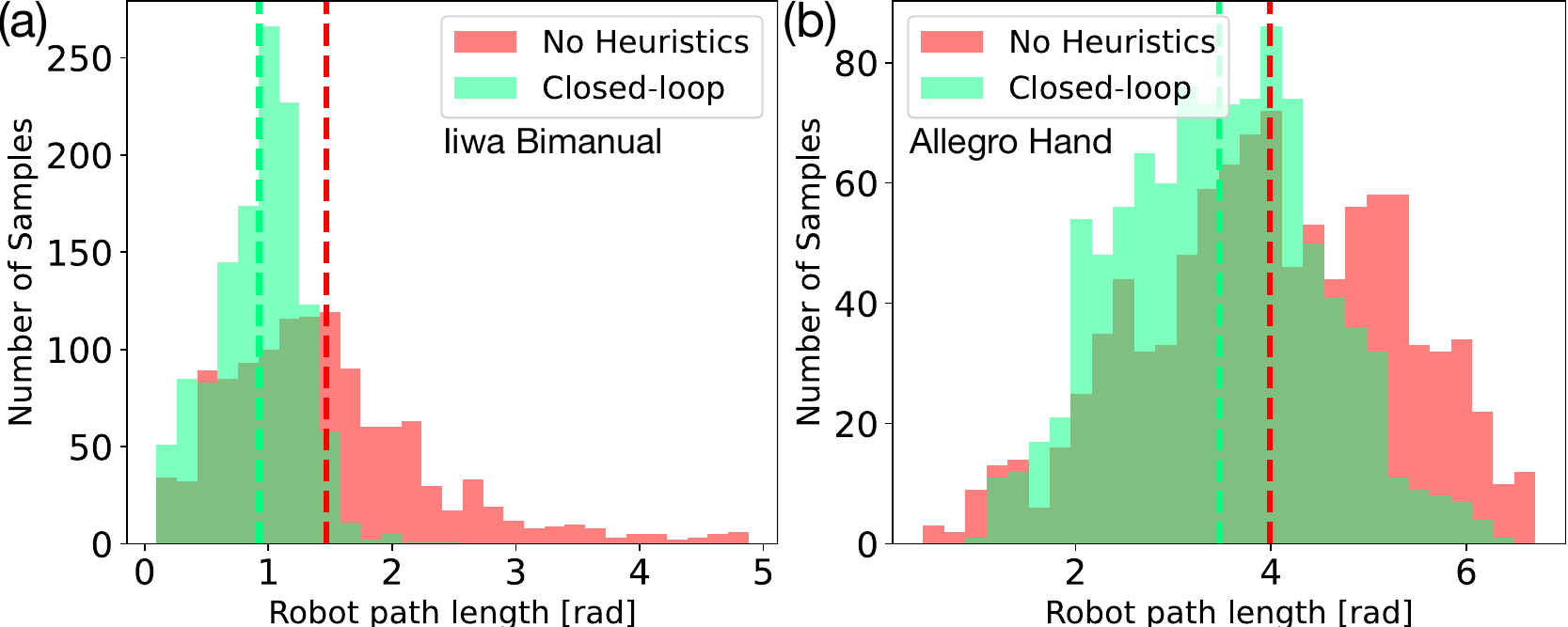}
\caption{Robot Path Length Comparison between the \textbf{Closed-loop} and \textbf{No Heuristics} variants of \Cref{alg:mpc_real_dynamics}. Dotted vertical lines indicate the mean values of each color-coded sample set.} 
\label{fig:robot_path_length_heuristics_vs_no}
\vskip -0.15 true in
\end{figure}

\begin{figure}
\centering\includegraphics[width = 0.40\textwidth]{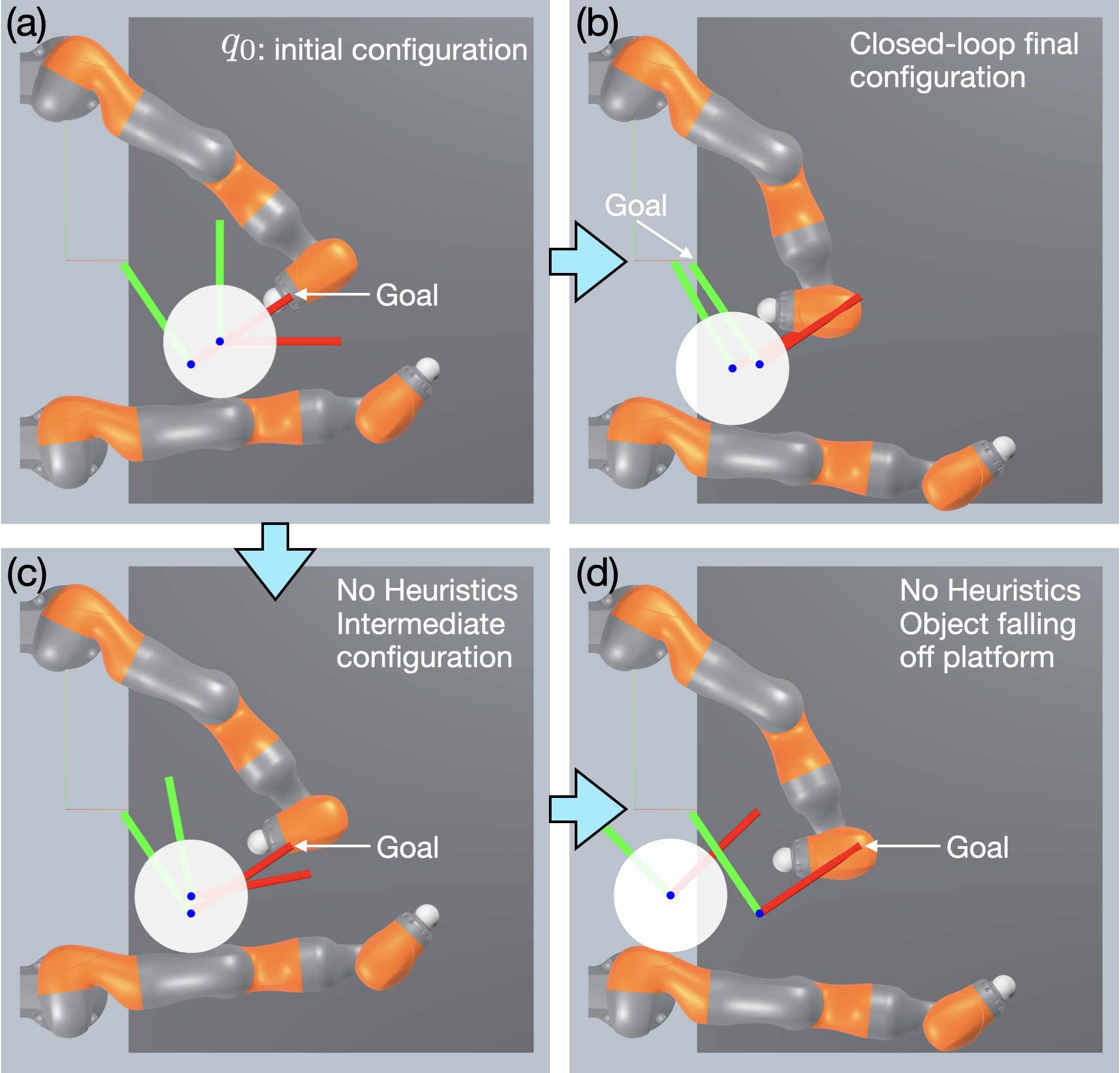}
\caption{Failures of \Cref{alg:mpc_real_dynamics} on \code{Iiwabimanual}. The dark gray box represents the table top which supports the bucket. The light gray area represents empty space.} 
\label{fig:iiwa_rollout_failures}
\vskip -0.15 true in
\end{figure}

\begin{figure}
\centering\includegraphics[width = 0.45\textwidth]{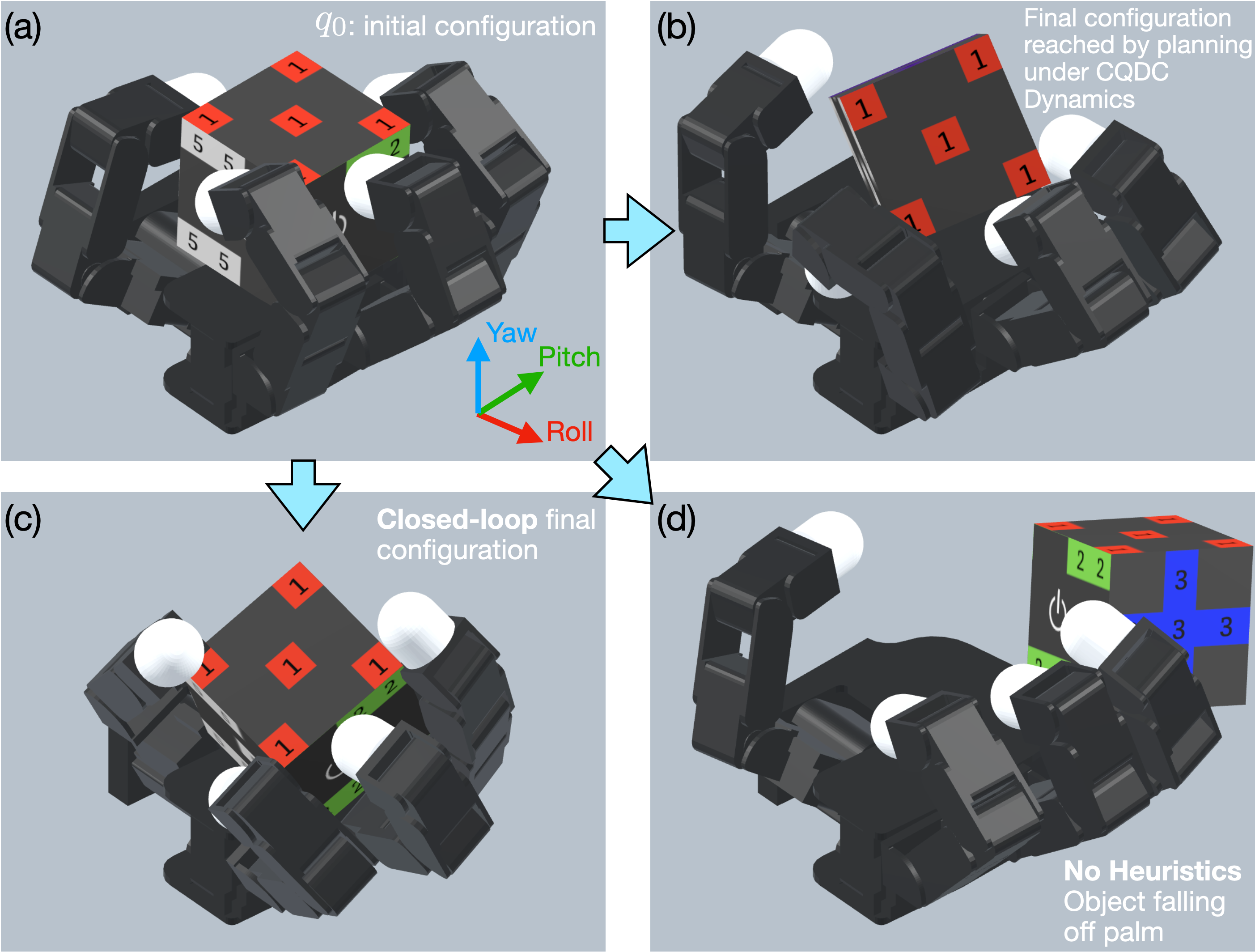}
\caption{Failures of \Cref{alg:mpc_real_dynamics} on \code{AllegroHand}.} 
\label{fig:allegro_rollout_failures}
\vskip -0.15 true in
\end{figure}

%% file: body/grasp-sampling.tex
\section{Generating Goal-conditioned Contact Configurations}\label{sec:actuator-placement}
Through our method presented in \Cref{sec:local-planning-control}, as well as the results in \Cref{sec:local-planning-results} and \Cref{sec:stabilization-results}, we have shown how our proposed MPC allows successful stabilization to \emph{local} goals. 

However, relying solely on finite-horizon MPC formulations can be limiting if the goal is sufficiently ``far away''. Specifically, goals that are more difficult to reach require the controller to momentarily make suboptimal actions in order to be optimal in the long run. Such problems inherently cannot be addressed solely by greedy finite-horizon MPC, as illustrated by key examples in \Cref{fig:exploration-difficulty}. To tackle such problems, classical motion planning and RL literature either resort to sophisticated exploration strategies or utilize dynamic programming to estimate a long-horizon value function.

To efficiently conduct long-horizon planning for contact-rich manipulation, we utilize two key insights. The first insight comes from dynamic programming: knowledge of key intermediate states allows us to decompose a difficult long-horizon problem into multiple shorter subproblems in which greedy strategies can be more effective. For the planar-pushing example in \Cref{fig:exploration-difficulty}, consider the configuration in which the ball is at the bottom of the box as opposed to the top. If the planner had (i) seen this configuration before and (ii) known that this configuration is significantly more advantageous for pushing the box upward using local control, it can attempt to go towards this advantageous configuration first before attempting to establish contact with the box. 

Next, we use the fact that because we have full control of the actuated DOFs, it is possible to leverage collision-free motion planning to move from one configuration to another without reasoning about contact dynamics (provided that a feasible path exists). This allows us to abstract away the details of collision-free motion planning at when we plan contact interactions, and assume that the robot can teleport from one configuration when the object is at a stable configuration. For instance, in the planar-pushing example in \Cref{fig:exploration-difficulty}, we can freely search for a good pusher configuration to push the box upwards without having to worry about the details of which path the robot has to take.

In this section, we leverage the CTR and our proposed MPC method to first search for key configurations for reaching given goals. As these configurations are almost always on the contact manifold, we refer to them as contact configurations. In the next section, we will show how we can chain these contact configurations together.

\begin{figure}
\centering\includegraphics[width = 0.40\textwidth]{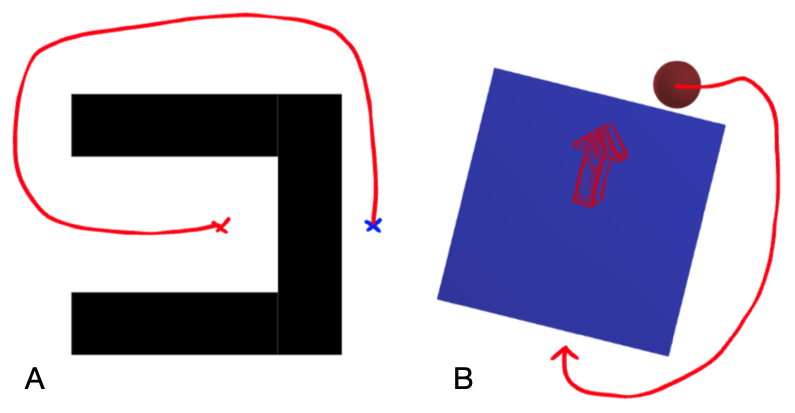}
\caption{Difficult cases for greedy finite-horizon MPC that require long-horizon exploration capabilities. A. A classical example for path-finding where the robot needs to go from the red X to the blue X. Note that the robot needs go back first, momentarily suffering an increase in Euclidean distance, before it can proceed towards the blue X and reach the goal. B. A similar planar-pushing example in manipulation, where the goal is to push the box upwards. From this configuration, the agent must break contact and travel all around to the back of the box, causing no effect in making progress towards the goal in terms of box movement, in order to push the box up.} 
\label{fig:exploration-difficulty}
\vskip -0.2 true in
\end{figure}

\subsection{Problem Specification}
Given a current configuration of the object $\qu$ and some goal configuration $q^\mathrm{o}_g$, we ask: what is a good robot configuration $q^\mathrm{a}$, if we want the resulting full configuration $q=(q^\mathrm{a},q^\mathrm{u})$ to be \emph{advantageous} for driving $\qu$ to $q^\mathrm{o}_g$ with local MPC? Formally, we write this optimization problem as 
\begin{subequations}\label{eq:nonlinear-grasp-sampling}
\begin{align}
    \min_{q^\mathrm{a}}\;\; & C(q^\mathrm{a}; \qu ,q^\mathrm{o}_g)\label{eq:nonlinear-grasp-sampling:cost} \\
    \text{s.t.}\;\; & q_{lb}^\mathrm{a} \leq q^\mathrm{a} \leq q_{ub}^\mathrm{a}, \label{eq:nonlinear-grasp-sampling:joint-limits} \\
    & \phi_i(q^\mathrm{a}, q^\mathrm{o}) \geq 0 \quad \forall i\label{eq:nonlinear-grasp-sampling:non-penetration},
\end{align}
\end{subequations}
where \eqref{eq:nonlinear-grasp-sampling:joint-limits} are joint-limit constraints, \eqref{eq:nonlinear-grasp-sampling:non-penetration} enforce non-penetration constraints for every collision pair indexed by $i$, and $\eqref{eq:nonlinear-grasp-sampling:cost}$ is our cost criteria for judging how fit $q^\mathrm{a}$ is in driving $\qu$ to $q^\mathrm{o}_{g}$. Our cost $C$ consists of two terms: the finite-horizon value function of the MPC policy (\Cref{sec:finite-horizon-value-function}), and a regularization term for robustness (\Cref{sec:robustness-regularizer}). 

\subsubsection{Finite-Horizon Value Function of the MPC Policy}
\label{sec:finite-horizon-value-function}
A natural way to query the fitness of $q^\mathrm{a}$ is to utilize the cost of our finite-horizon MPC problem \eqref{eq:nonlinear-to:cost} incurred from the MPC rollout defined in \Cref{alg:mpc},
\begin{subequations}\label{eq:value-function}
\begin{align}
    V(q^\mathrm{a};\qu,&q^\mathrm{o}_g)= \|q^\mathrm{o}_g - q^\mathrm{o}_T\|^2_\mathbf{Q} + \sum^{T-1}_{t=0} \|u_t - u_{t-1}\|^2_\mathbf{R},\\
    \text{s.t.} \quad     & q_{0:T},u_{0:T-1} = \textbf{MPC}(q_0, q^\mathrm{o}_g, T, n_\mathrm{max}, H),\\ & q_0 = (q^\mathrm{a},q^\mathrm{o}).
\end{align}
\end{subequations}
We note that although MPC rollout in \Cref{alg:mpc} accepts a full configuration $q_g$ as a goal, we can give it unactuated configurations only ($q^\mathrm{o}_g$) by setting the cost terms of the actuated objects to zero, $\mathbf{Q}_a=\mathbf{0}$, as we did in \Cref{sec:local-planning-results:evaluation_metrics}. In addition, due to the efficiency of the MPC controller, the finite-hozion value function can be quickly queried online. However, the landscape of this value function, as visualized for a simple problem in \Cref{fig:pusher-t-visualization}, is multi-modal with many local minima and maxima. This hints at the necessity of global optimization when we search for minimizers of \eqref{eq:value-function}.  

\begin{figure*}[t]
\centering\includegraphics[width = 0.98\textwidth]{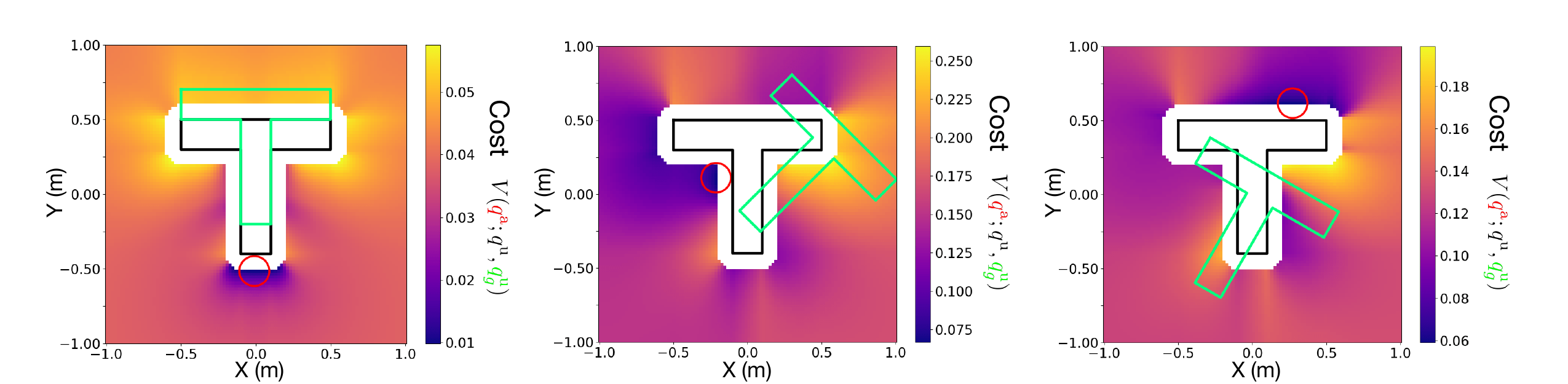}
\caption{Visualization of the landscape of \eqref{eq:value-function} for the planar pusher-T example inspired by \citet{chi2024diffusionpolicy}. The black T corresponds to the current object configuration $\qu$, and the green T the goal object configuration $q^\mathrm{o}_{g}$. The landscape corresponds to the position of the round pusher $q^\mathrm{a}$, colored with the cost function $V(\qa;\qu,q^\mathrm{o}_g)$. The red pusher configuration corresponds to the global optima of the landscape.}
\label{fig:pusher-t-visualization}
\vskip -0.25 true in
\end{figure*}

\subsubsection{Robustness Regularizer}\label{sec:robustness-regularizer}

Is the value function in \eqref{eq:value-function} sufficient as a cost? Although it correctly evaluates the fitness of a given $\qa$ in terms of closed-loop goal reaching, we found that there may be cases where multiple configurations are equally fit, yet one configuration provides more robustness compared to others. 

For example, consider the task of pushing the ball slightly to the right in \Cref{fig:robustness}. Both configurations are nearly equal in the goal-reaching cost (see \Cref{tab:robustness}); yet, the configuration in \Cref{fig:robustness}a would be preferred in practice, as both fingers can be used to reject small disturbances.

To formalize this notion of robustness, we take inspiration from classical grasping metrics \cite{ferrari-canny, hongkai-sdp, frogger}, as well as the connection between the RA-CTR and the classical wrench set in \Cref{sec:mechanics-derivation}. Specifically, we adopt a worst-case metric \cite{ferrari-canny} that reasons about the maximum wrench a grasp can resist along any direction. Geometrically, this quantity corresponds to the maximum-inscribed sphere in the wrench set.

Formally, consider a unit vector $v\in\mathbb{R}^{n_{q_\mathrm{u}}}$. Then, the radius of the maximum-inscribed sphere within the wrench set can be described as the following mini-max problem,
\begin{subequations}
\begin{align}
    r(\qa, \qu) \coloneqq \min_{v}&\max_{r\in\mathbb{R}} \; r \\
    \text{s.t.}  \;& \|v\| = 1 \\
                   & rv \in \mathcal{W}_{\mathbf{\Sigma},\kappa}^\mathcal{A}(\bar{q}=(\qa, \qu),\bar{u}=q^\mathrm{a}).
\end{align}
\end{subequations}

To evaluate this quantity, we first make a polytopic approximation of the wrench set $\mathcal{W}_{\mathbf{\Sigma},\kappa}^\mathcal{A}$ by sampling from the RA-CTR using a rejection sampling scheme. This scheme (i) first samples from the ETR ellipsoid and (ii) rejects samples that do not obey feasibility constraints (this is the same procedure which generated the samples in \Cref{fig:motion-set-visualization}). Then, we fit a convex hull to these samples using the \texttt{Qhull} library. 

Once we have a polytopic representation of the convex set in $H$-rep, (i.e. $\{z|\mathbf{a}_{\mathcal{W},i}^\top  z + b_{\mathcal{W},i} \leq 0\;\forall i\}$), we can find $r$ by solving a variant of the Chebyshev center problem,
\begin{subequations}
\label{eq:chebyshev_center}
\begin{align}
    \max_{r\geq 0} \;& r \\
    \text{s.t.}\; & \|a_i\|_2 r + b_i \leq 0 \quad \forall i.
\end{align}
\end{subequations}

\subsubsection{Total Cost}
We combine the two costs in \Cref{sec:finite-horizon-value-function} and \Cref{sec:robustness-regularizer} use a weighting term $\alpha\in\mathbb{R}_{\geq 0}$: 
\begin{equation}
    C(\qa;\qu,q^\mathrm{o}_g) = V(\qa;\qu,q^\mathrm{o}_g) - \alpha r(\qa;\qu)^2,
\end{equation}
where we note that the radius is subtracted as it is a reward. 

\begin{figure}[h]
\centering\includegraphics[width = 0.48\textwidth]{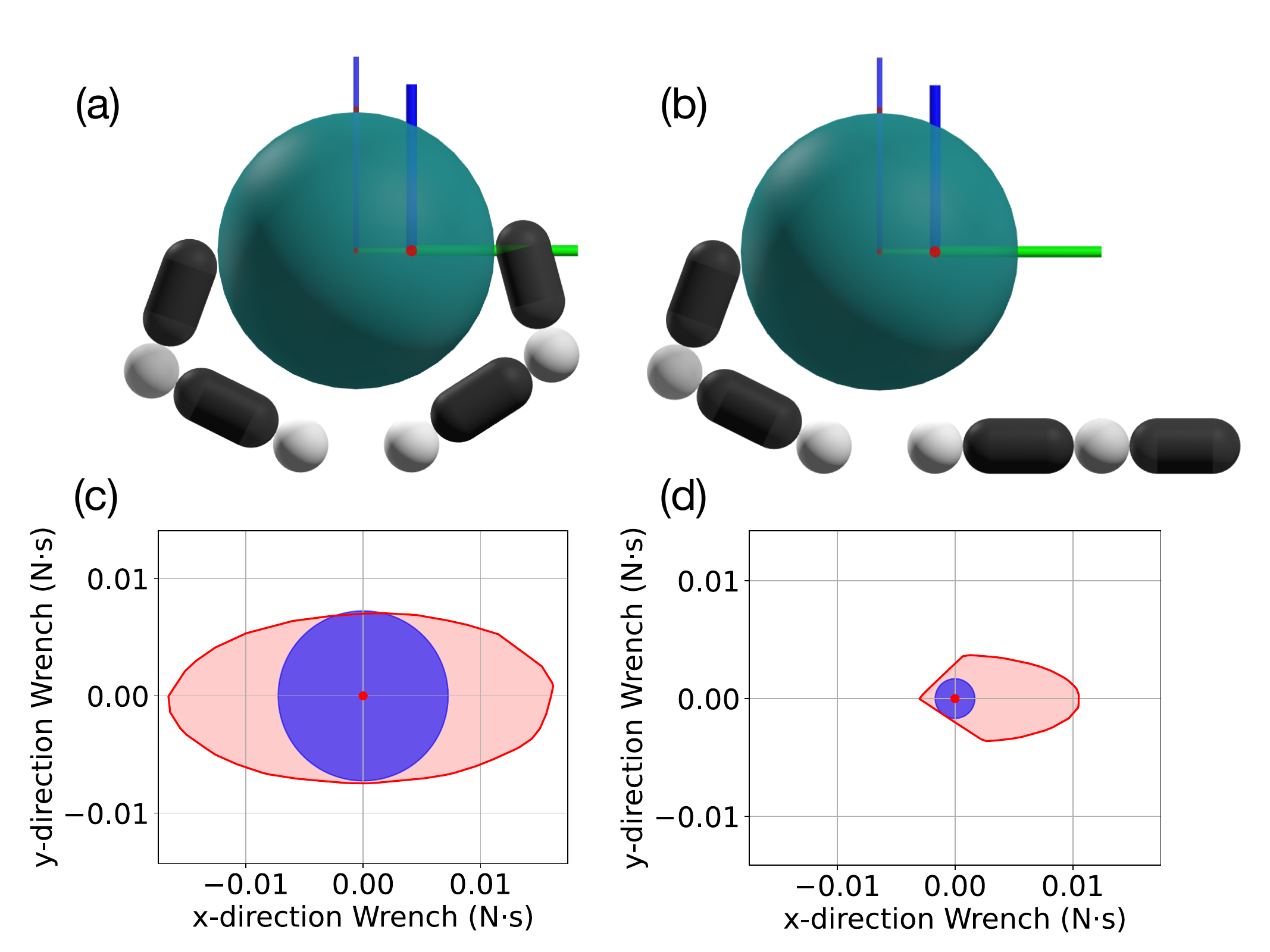}
\caption{Top Row: Comparison of two configurations (\textbf{a}) and (\textbf{b}) for the 2D planar hand system in \Cref{fig:motion-set-visualization}. Although both configurations are advantageous for moving the object to the goal, they share the same MPC value function cost (\Cref{tab:robustness}). Bottom Row: Plot of the wrench set (illustrated in red) and the maximum inscribed sphere (blue), where (\textbf{c}) corresponds to configuration (\textbf{a}) and (\textbf{d}) corresponds to (\textbf{b}). Note that configuration (\textbf{a}) has a much larger inscribed sphere due to the anti-podal grasp.} 
\label{fig:robustness}
\end{figure}

\begin{table}[h]
\centering
\small
\begin{tabular}{@{}lcc@{}}
\toprule
Configuration & (a) & (b) \\
\midrule
MPC value function $V$          & $2.82\times10^{-4}$ & $2.16\times10^{-4}$ \\
Max inscribed sphere radius     & $7.12\times10^{-3}$ & $1.68\times10^{-3}$ \\
\bottomrule
\end{tabular}
\caption{Comparison of MPC value function cost and maximum inscribed sphere radius for two configurations in \Cref{fig:robustness}.}
\label{tab:robustness}
\vskip -0.2 true in
\end{table}

\subsection{Solving \eqref{eq:nonlinear-grasp-sampling} by Sampling-based Optimization}
Solving \eqref{eq:nonlinear-grasp-sampling} is challenging, as (i) the gradient of the cost function is quite difficult to obtain \cite{frogger}, and (ii) it requires global search, as evidenced by the nonconvex cost landscape in \Cref{fig:pusher-t-visualization}. As a result, we resort to a simple sampling-based search which samples from the feasible set of \eqref{eq:nonlinear-grasp-sampling} by rejection sampling, then chooses the best sample.

However, due to the high variance of this process in high-dimensional spaces, as well as the complexity of navigating high-dimensional configuration spaces, we found that directly applying this strategy is not very effective beyond simple planar problems. To scale the sampling-based search to \code{AllegroHand}, we introduce a few heuristic changes to make this optimization more tractable, which we detail in \Cref{appendix:allegro_contact_sampling_heurisitics}.

\begin{figure}
\centering\includegraphics[width = 0.48\textwidth]{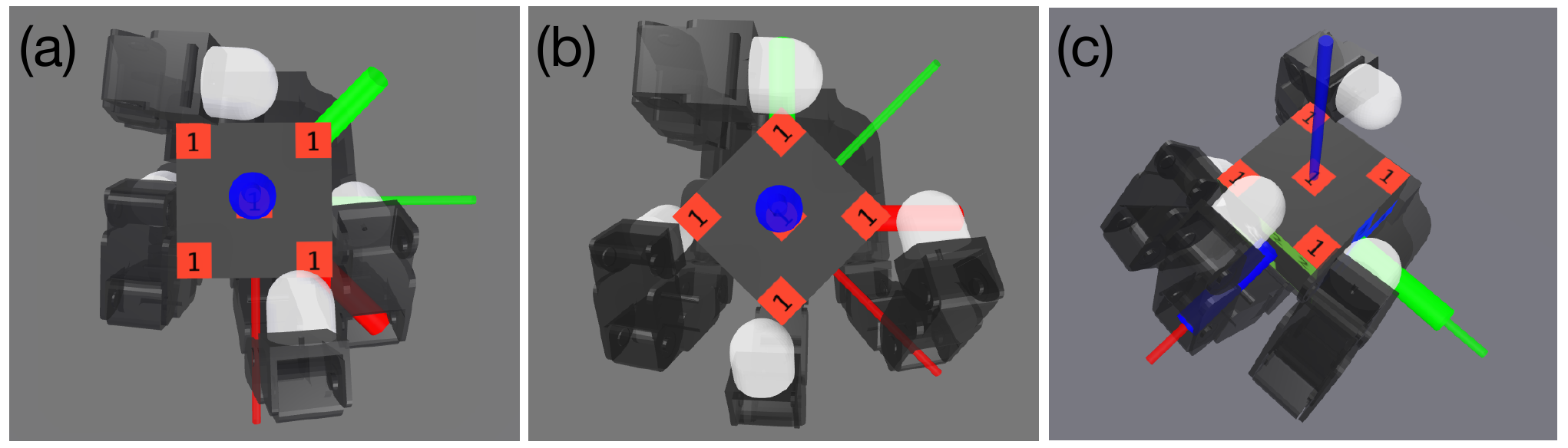}
\caption{Initial actuator configurations found by our method for (\textbf{a}) yaw $0^\circ$ to $45^\circ$, (\textbf{b}) yaw $45^\circ$ to $90^\circ$, and (\textbf{c}) pitch $0^\circ$ to $90^\circ$. The initial configuration of the cube is marked with a long and thin triad, while the goal configuration of the cube a short and thick triad.} 
\label{fig:synthesis-results}
\vskip -0.1 true in
\end{figure}

\begin{table}
\centering
\small
\begin{tabular}{@{}lccc@{}}
\toprule
Task & Position\,[mm] & Rotation\,[mrad] & Time\,[s] \\
\midrule
Yaw $0^{\circ}\!\to\!45^{\circ}$  & 12.05\,(3.49) & 32.41\,(27.14) & 57.83\,(2.08) \\
Yaw $45^{\circ}\!\to\!90^{\circ}$ &  9.85\,(10.88) & 15.72\,(21.08) & 57.28\,(0.72) \\
Pitch $0^{\circ}\!\to\!90^{\circ}$& 21.42\,(8.91) & 32.12\,(32.03) & 57.84\,(0.73) \\
\bottomrule
\end{tabular}
\caption{Performance of our algorithm for finding initial actuator configurations. Results are obtained with $10$ repeated runs. Each run is optimized using $1000$ samples. Each cell displays the mean (standard deviation).}
\vskip -0.2 true in
\label{tab:synthesis-results}
\end{table}

\subsection{Results}
We test the performance of our method on the \code{AllegroHand} system. We first set up three representative pairs of initial and goal configurations for the cube, then run our algorithm to solve \eqref{eq:nonlinear-grasp-sampling}, whose solution corresponds to the optimal initial configuration of the Allegro hand. Then, we evaluate the fitness of the solution by rolling out MPC (\Cref{alg:mpc}) from the found initial configuration, and record the error between the final rollout and the goal configurations. 

Results in \Cref{tab:synthesis-results} indicate that our method achieves error on the order of $10$mm in position, and $30$ mrad ($1.7^\circ$) in orientation. Moreover, the discovered initial hand configurations in \Cref{fig:synthesis-results} agree with our intuition on how the cube should be grasped if we want to move the cube towards the goals.

We note that the given result was obtained with $1000$ samples on the reduced-order model. Reducing the number of samples for lower computation time is possible, but would result in higher variance in performance. In addition, we found that other types of motions, such as pitch $0^\circ$ to $-90^\circ$, or rotation along the $x$-axis (roll), are physically very difficult on the Allegro hand.

%% file: body/prm.tex
\section{Global Planning with Roadmaps}\label{sec:roadmap}
Using the method for generating contact configurations in \Cref{sec:actuator-placement}, we present a simple recipe for global search, in which we chain local plans together to efficiently reach \emph{global} goals which are challenging for local MPC. Our method, inspired by the Probabilistic Roadmap (PRM)\cite{choset2005principles}, consists of an offline phase in which the roadmap is constructed, and an online phase where the same roadmap is reused to reach any new goals. 

\subsection{Roadmap Construction}
In the offline phase, we build a roadmap in which the vertices are grasping configurations and the edges are local plans that transition between these configurations. 

We present the roadmap construction method in \Cref{alg:roadmap}. 
The set of grasping configurations $\mathcal{G}$ in \Cref{alg:roadmap:generate_contact_configurations} are generated by solving \eqref{eq:nonlinear-grasp-sampling} for $q^\mathrm{o}$'s from a set of stable object configurations which sufficiently cover the object workspace.
For each pair of configurations $(q_i, q_j)$ in $\mathcal{G}$, we first try to reach $q^\mathrm{o}_j$ from $q_i$ by running MPC (\Cref{alg:roadmap:mpc}), and then reach $q^\mathrm{a}_j$ by standard collision-free motion planning (\Cref{alg:roadmap:collision_free}). This procedure is illustrated in \Cref{fig:roadmap}. We repeat this procedure and add all successful connections to the roadmap.

\begin{algorithm}
\caption{Roadmap Construction}\label{alg:roadmap}
\textbf{Input:} MPC Parameters: $T$, $n_\mathrm{max}$ and $H$\; \label{alg:roadmap:input}
\textbf{Output:} Sets of vertices $V$ and edges $E$\;
$V \gets \emptyset$, $E \gets \emptyset$\;
Generate contact configurations $\mathcal{G} \coloneqq \{(q^\mathrm{o}_i, q^\mathrm{a}_i)\}_{i=1}^m$ \label{alg:roadmap:generate_contact_configurations}\;
\For {ordered pair $q_i \coloneqq (q^\mathrm{o}_i, q^\mathrm{a}_i), q_j \coloneqq (q^\mathrm{o}_j, q^\mathrm{a}_j)$ in $\mathcal{G}$}{
    $q_{0:H}, u_{0:H-1} \gets \mathbf{MPC}(q_i, q^\mathrm{o}_j, T, n_\mathrm{max}, H)$\; \label{alg:roadmap:mpc}
    $q_{H:M}, u_{H:M-1} \gets \mathbf{CollisionFree}(q_H, q_j)$\; \label{alg:roadmap:collision_free}
    \If {$\mathbf{MPC}$ and $\mathbf{CollisionFree}$ both successful}{
        $V$.\texttt{add}($q_i$), $V$.\texttt{add}($q_j$)\;
        $E$.\texttt{add}($q_i$, $q_j$, $(q_{0:M}, u_{0:M-1})$)\;
    }
}
\algorithmicreturn $\; V, \; E$
\end{algorithm}

\begin{figure}[t]
\centering\includegraphics[width = 0.40\textwidth]{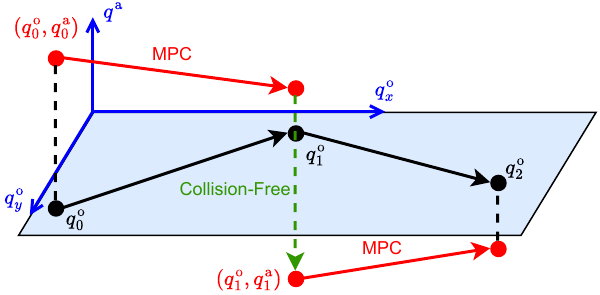}
\caption{Illustration of roadmap construction described by \Cref{alg:roadmap}. The vertical axis represents the actuator configuration $\qa$, while the horizontal plane the space of object configurations $\qu$. To connect the vertices $q_0 := (q^\mathrm{o}_0, q^\mathrm{a}_0)$ and $q_1 := (q^\mathrm{o}_1, q^\mathrm{a}_1)$, we first plan from $q_0$ to $q^\mathrm{o}_1$ with MPC, reaching the top red dot. As the top red dot and $q_1$ share the same object configuration $q^\mathrm{o}_1$, we can connect them with a collision free planner.} 
\label{fig:roadmap}
\vskip -0.2 true in
\end{figure}

We further note that in the special case of manipulating objects with geometric symmetries (such as \code{AllegroHand}), a single grasping configuration can represent multiple equivalent configurations. 
For example, a single grasping configuration for a cube can be expanded into $24$ distinct configurations, due to the cube's $24$ rotational symmetries. 
By exploiting this property, the same grasp and action sequence can be used to generate multiple edges in the roadmap. 

Consider the cube configuration in \Cref{fig:roadmap_example}a, which corresponds to the identity rotation (i.e. no rotation). By leveraging symmetry, we can reach its 24 rotational symmetries using just three basic operations: (i) yaw from $0^\circ$ to $90^\circ$, (ii) yaw from $0^\circ$ to $-90^\circ$ and (iii) pitch from $0^\circ$ to $90^\circ$. These three operations require just five grasps in total: 2 for each yaw and 1 for the pitch (see \Cref{fig:synthesis-results}). The grasps take about 5 minutes to generate (\Cref{tab:synthesis-results}). Moreover, connecting each pair of grasping configurations with MPC and collision-free motion planning takes a few seconds (\Cref{tab:planning_time}), keeping the total roadmap construction time under 10 minutes. Parallelizing these steps can further reduce computation time. 

\begin{figure*}[t]
\centering\includegraphics[width = 0.99\textwidth]{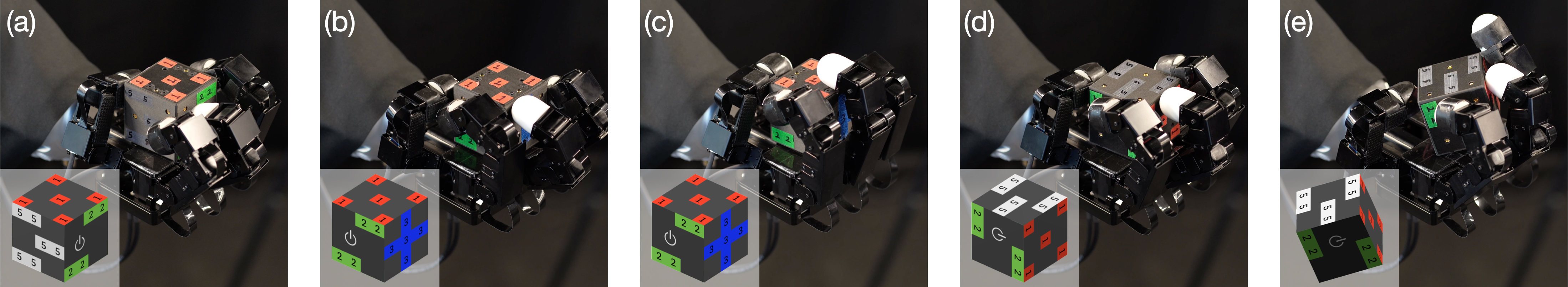}
\caption{A complete path generated by our roadmap-based global planner. The object configuration in each frame is highlighted at the lower-left corner.
(\textbf{a}) shows the starting configuration, which in this example is a vertex already in the roadmap.
(\textbf{b}) can be reached from (\textbf{a}) with a $-90^\circ$ yaw. 
(\textbf{c}) has the same object configuration as (\textbf{b}), but the hand has repositioned for a $90^\circ$ pitch.
(\textbf{d}) is system configuration after the pitch. 
(\textbf{e}) shows the system reaching the goal configuration from (\textbf{d}). 
The path (\textbf{a})-(\textbf{d}) is part of the roadmap and generated offline. 
(\textbf{d})-(\textbf{e}) is generated online using collision-free planning and local MPC. 
}
\label{fig:roadmap_example}
\vskip -0.15 true in
\end{figure*}

To verify the robustness of the constructed roadmap, we conducted a random walk on the roadmap on hardware, and recorded 150 successful consecutive edge transitions before hardware failure occurred (the hand overheated). Recording of this experiment can be found in the supplementary video. The high robustness of the roadmap can be attributed to (i) the simplicity of the cube's geometry, whose symmetry makes connecting nodes with MPC a lot easier; (ii) the caging property of the grasps (i.e. nodes), which consistently guides the cube back to the center of the palm.

\subsection{Inference on a Roadmap}
After the roadmap $(V, E)$ is constructed, we can synthesize plans connecting any starting configuration $q_0$ to any goal object configuration $q^\mathrm{o}_\text{goal}$. To do this, we first connect $q_0$ and $q^\mathrm{o}_\text{goal}$ to their respective nearest vertices in the vertex set $V$, which can be done using the same procedure in \Cref{alg:roadmap:mpc} and \Cref{alg:roadmap:collision_free} of \Cref{alg:roadmap}. Then, the problem reduces to finding the shortest path between two vertices on a graph, which can be solved with standard methods. An example path generated using this approach is shown in \Cref{fig:roadmap_example}. More roadmap planning examples can be found in the supplementary video.

\subsection{Limitations}
The primary limitation of our global planner is that its success depends on a connection recipe that is not yet general-purpose. Our current implementation heavily leverages the specific kinematics of the \code{AllegroHand} and the symmetries of the cube.

Our method for connecting two roadmap nodes, $(q^\mathrm{o}_i, q^\mathrm{a}_i) \in \mathcal{G}$ and $(q^\mathrm{o}_j, q^\mathrm{a}_j) \in \mathcal{G}$, involves two steps: (i) an \textbf{MPC} plan moves the object from $q^\mathrm{o}_i$ to $q^\mathrm{o}_j$, and (ii) a collision-free motion plan repositions the hand from the hand configuration reached by \textbf{MPC} to $q^\mathrm{a}_j$. Both steps are made easy by carefully selecting nodes that correspond to the 24 stable, symmetric rotations of the cube resting flat on the palm.

Our choice of nodes trivializes the first step, as connecting any $q^\mathrm{o}_i$ to $q^\mathrm{o}_j$ only needs to roll, pitch or yaw the cube by 90 degrees. 
However, for generic object shapes without exploitable symmetry, connectivity of $q^\mathrm{o}_j$ from $(q^\mathrm{o}_i, q^\mathrm{a}_i)$ needs to be computed more explicitly. One possible strategy is to build a convex inner approximation of the set of object poses reachable from $(q^\mathrm{o}_i, q^\mathrm{a}_i)$ using a method similar to \citet{petersen2023growing}, and attempt to connect only if $q^\mathrm{o}_j$ belongs to the approximated reachable set.

The second step critically assumes that the object remains stable and stationary at $q^\mathrm{o}_j$ while the hand moves.
This assumption, however, would likely fail for objects without nice symmetries or for tasks requiring grasps on unstable object poses, as the collision-free planner could no longer assume a static object. A more general method would need to explicitly reason about maintaining object stability during these hand-repositioning phases, such as a finger-gaiting strategy where a subset of fingers hold the object still while the other fingers reposition.

%% file: body/conclusion.tex
\section{Discussion and Conclusion}
\subsection{Synergies with Learning-based Methods}
Our contact-rich planners have strong connections to modern robot learning, particularly in addressing key bottlenecks in (i) large-scale imitation learning (IL) and (ii) reinforcement learning (RL).

\subsubsection{Data Synthesis and Augmentation for IL}
Recent advances in large-scale imitation learning are impressive \cite{black2024pi_0, lbmtri2025}, but their performance hinges on large-scale, embodiment-specific datasets. This data is typically collected via human teleoperation, which presents a significant bottleneck for two reasons: not only is teleoperation time-consuming and expensive, teleoperation interfaces are also often limited to controlling robot end effectors, making it challenging for humans to provide high-quality demonstrations for tasks involving whole-body or in-hand contact.

Our global planner directly addresses this challenge by serving as an automated expert data generator. By leveraging a model-based planner, we can create vast datasets of successful trajectories for complex, contact-rich tasks that are hard to teleoperate. For instance, we can generate synthetic expert demonstrations for tasks like \code{IiwaBimanual} and \code{AllegroHand} \citep{huaijiang-ral} to train a goal-conditioned diffusion policy \citep{chi2024diffusionpolicy}. The resulting policy could achieve a high success rate with zero human demonstrations and would be more computationally efficient at runtime than using the planner, as it wouldn't need to replan from scratch.

Beyond replacing human data, our planners can also be used to augment it. Human demonstrations often cover a narrow distribution of successful trajectories. A policy trained solely on this data can be brittle and fail when faced with small perturbations. Our local MPC planner can significantly enhance policy robustness. Following the approach in \citet{lu-rss}, we can take a human demonstration, inject noise at various points, and use a local MPC planner to generate a corrective action that brings the robot back toward the goal. This process creates a rich dataset of near-failure states and successful recoveries, teaching the learned policy how to handle perturbations far more effectively than a policy trained only on the original, unperturbed data.

\subsubsection{Efficient Exploration for RL}
Our global planner can also serve as efficient exploration mechanisms for RL. Despite many algorithmic advances, efficient exploration still remains a significant challenge for RL. For problems in \Cref{fig:exploration-difficulty}, for example, a naive Gaussian sampling exploration scheme would have a really hard time finding an effective policy. In order to bypass this difficulty, many practical RL algorithms leverage expert demonstrations either as rewards \cite{rajeswaran2018learningcomplexdexterousmanipulation, deepmimic, peng2021amp} or as random resets during training \cite{song2023reaching} to guide RL exploration towards relevant states that are more beneficial to accomplish the task.

Both \cite{khandate} and \cite{barreiros2025example-guided-rl} have successfully explored the effectiveness of using model-based planners such as \cite{gqdp} to vastly improve the sample efficiency of exploration in RL. We hypothesize that the stable states in our roadmap, for example, can serve as an extremely efficient exploration distribution in RL.

\subsection{Concluding Remarks and Limitations}
Have we solved the problem of planning and control through contact dynamics? Locally—on CQDC dynamics—the proposed MPC with Contact Trust Region constraints can achieve small tracking errors for the vast majority of the goals we sampled based on local reachability criteria. The local MPC can be readily integrated with a sampling-based planner to enable global search, although more work is necessary to generalize our global planner beyond the cube. 

We demonstrated our planner's efficacy on two distinct platforms: \code{IiwaBimanual} and \code{AllegroHand}. These examples were intentionally chosen to highlight the method's generalizability, as they represent significantly different scales, kinematic structures, and contact scenarios—from whole-body manipulation to dexterous, in-hand motion. Since our method only needs robot and object URDFs, it can be readily extended to other contact-rich systems, including different hands, grippers, or even legged robots (for which we may need to add balancing constraint to trust regions).

However, we do not understand this problem nearly as well as we do the simple pendulum: not only do we lack solid explanations for the small but non-zero number of planner failures, but we also have yet to fully understand the different roles played by feasibility constraints on \code{IiwaBimanual} vs \code{AllegroHand}.

Moreover, much more remains to be understood for second-order dynamics, both in simulation and on hardware. In particular, keeping the robot in contact with the object, without exploiting the artifact of CQDC dynamics, remains one of the biggest unsolved challenges. Incorporating the CTR constraints in MPC greatly alleviates the problem of lost contact, and the initial guess heuristics empirically helps a bit more. However, we occasionally still observe loss of contact, and its root cause remains unclear. 

Nevertheless, the tools presented in this paper already enable capabilities that were previously out of reach for model-based methods.
By accounting for the unilateral nature of contact, our contact trust region makes it possible to apply a broad range of robotics algorithms to contact-rich manipulation problems. 
We hope the MPC, grasp synthesis, and roadmap-based global planning methods introduced here are only a small sample of the many contact-rich manipulation algorithms yet to come.

%% file: body/acknowledgement.tex
\begin{acks}
We thank Stephen Proulx, Xinpei Ni, and Velin Dimitrov for building the cube with active OptiTrack markers, Giovanni Remigi for making the project website, and Bridget Hogan for building visualization tools. 
\end{acks}

\begin{dci}
The author(s) declared no potential conflicts of interest with respect to the research, authorship, and/or publication of this article
\end{dci}

\begin{funding}
This work is funded by the Defense Science \& Technology Agency Award DST00OECI20300823, Amazon Award PO\#2D-15694085, and the Boston Dynamics AI Institute Award Agmd Dtd 8/1/2023.
\end{funding}

%% file: body/appendix.tex
\section{Appendix}
\subsection{Derivation of \eqref{eq:kkt-sensitivity}}\label{proof:kkt-sensitivity}
The optimality conditions of \eqref{eq:q_dynamics_log} consists of the stationarity condition \eqref{eq:unconstrained_staionarity_with_lambda} and the 
relaxed complementary slackness \eqref{eq:perturbed-kkt}, which we reproduce below as
\begin{subequations}
\begin{align}
    \mathbf{P}q_+ + b - \sum_i \mathbf{J}_i^\top \lambda_i & = 0, \\
    \nu^\top_i\lambda_i & = 2 \kappa^{-1}, \quad \forall i.
\end{align}
\end{subequations}
The structure of the optimization problem parameters $(\mathbf{P},b,\mathbf{J}_i,c_i)$ tells us that $b$ is the only variable dependent on $u$. As the primal and dual variables $q_+,\lambda_i$ are also dependent variables, differentiating both equations w.r.t. $u$ gives us 
\begin{subequations}
\begin{align}
    \mathbf{P}\frac{\partial q_+}{\partial u} + \frac{\partial b}{\partial u} - \sum_i \mathbf{J}_i^\top \frac{\partial\lambda_i}{\partial u} & = 0 \\
    \left(\mathbf{J}_i\frac{\partial q_+}{\partial u}\right)^\top\lambda_i   + \frac{\partial \lambda_i}{\partial u}^\top (\mathbf{J}_i q_+ + c_i)   & = 0 \quad \forall i
\end{align}
\end{subequations}
Rewriting in matrix form gives us \eqref{eq:kkt-sensitivity}. 

We note that the matrix equation \eqref{eq:kkt-sensitivity} is not invertible unless the primal and dual cones defined in \eqref{eq:friction_constraints} were one-dimensional, e.g. $\mathcal{K} = \mathbb{R}_+$ and $\mathcal{K^\star} = \mathbb{R}_+$. This holds in the quadratic programming formulation of contact dynamics \cite{pangsimulator} where constraints are simple inequalities instead of conic inequalities. For higher-dimensional conic constraints such as \eqref{eq:friction_constraints}, differentiating \eqref{eq:unconstrained_staionarity_with_lambda} and \eqref{eq:dual_feasible_point} (instead of \eqref{eq:perturbed-kkt}) together would produce an invertible systems of equation to which the implicit function theorem can be applied. However, since deriving the derivatives involves complicated matrix differential calculus \cite{magnus2019matrix}, we omit its derivation in this paper.

\subsection{Proof of \Cref{lemma:motion-wrench}}\label{app:proof}
Writing the wrench set gives us 
\begin{subequations}
\begin{align}
    \mathcal{W}(\bar{q},\bar{u}) = \{w & | w = \sum_i \Ju[i]^\top \hat{\lambda}_i \\ 
    \hat{\lambda}_i & = \mathbf{D}_i\delta u + \lambda(\bar{q},\bar{u}) \\
    \delta u^\top\mathbf{\Sigma}\delta u & \leq 1 \\    
    \hat{\lambda}_i & \in \mathcal{K}^\star_i\}.
\end{align}
\end{subequations}
and we want to prove that the set, 
\begin{equation}
    \mathcal{M}_{\Sigma, \kappa}^{\mathcal{A},\mathrm{o}}(\bar{q},\bar{u}) \coloneqq \{ \mathbf{B}^\mathrm{o}\delta u + f^\mathrm{o}(\bar{q},\bar{u}) | \delta u \in \tilde{\mathcal{S}}^\mathcal{A}_{\Sigma, \kappa}(\bar{q},\bar{u})\}
\end{equation}
is equivalent to 
\begin{equation}
    \{\quplus | 
    \frac{\epsilon}{h}\Mu(\bar{q})(q^\mathrm{o}_+ - \bar{q}^\mathrm{o}) = h\tau^\mathrm{o} + w, w\in\mathcal{W}\}.
\end{equation}

To prove this, we utilize \Cref{lemma:taylor-approximation} to argue that for any given $\delta u$, the next configuration $\hat{q}_+^\mathrm{o}$ and contact impulses $\bar{\lambda}_i$ that are defined by a linear map on $\delta u$, 
\begin{align}
    \hat{q}_+^\mathrm{o} & = \mathbf{B}^\mathrm{o}\delta u +f^\mathrm{o}(\bar{q},\bar{u}), \\
    \hat{\lambda}_i & = \mathbf{D}_i \delta u + \lambda_i(\bar{q},\bar{u}),
\end{align}
must jointly satisfy
\begin{align}
    \frac{\varepsilon}{h}\Mu\left(\hat{q}_+^\mathrm{o} - \bar{q}^\mathrm{o}\right) = h\tau^\mathrm{o} + \sum_i \Ju[i]^\top \hat{\lambda}_i.
\end{align}

\subsection{Goal Generation for Local MPC}
\label{appendix:goal_selection}
To generate goals that are considered more locally reachable yet far enough to be challenging for the planner, we utilize the following schemes. For the \code{IiwaBimanual} environment, we first prescribe some initial state $\bar{q}$ by sampling contact points on the surface of the object and solving inverse kinematics; then, the goals are sampled from the boundary of the action-only object motion set $\mathcal{M}^{\mathcal{A},\mathrm{o}}_{\mathbf{\Sigma}, \kappa}(\bar{q},\bar{u})$ with $\bar{u}=\bar{q}^\mathrm{a}$. To make the goals sufficiently challenging, we use large ellipsoid radius for the purpose of sampling goals. The resulting goals have maximum $120^\circ$ rotation and $0.4m$ translation from the initial state (see \Cref{fig:planning-comparison} for the distribution of goals). Similarly, we prescribe some grasp for \code{AllegroHand} and sample orientations that have about a $60^\circ$ difference from the initial state according to the axis-angle metric. With the above procedure, we sampled $1233$ pairs of initial states and goals for \code{IiwaBimanual} and $1000$ pairs for \code{AllegroHand}.

We summarize in \Cref{tab:goal_average_movement} the average amount of translation and rotation between the initial and goal object configurations generated here. These would be the object pose tracking errors if the planning algorithm being evaluated did nothing.
\begin{table}[h]
\centering
\small
\begin{tabular}{@{}lcc@{}}
\toprule
 & \code{IiwaBimanual} & \code{AllegroHand} \\
\midrule
Mean $\lVert p_{\text{goal}} - p_{0}\rVert$ [mm]            & 152 & 16  \\
Mean $\lvert\Delta\theta(Q_{\text{goal}},Q_{0})\rvert$ [mrad] & 356 & 788 \\
\bottomrule
\end{tabular}
\caption{Mean translation and rotation differences for pairs of initial and goal object states generated in this section. }
\label{tab:goal_average_movement}
\vskip -0.2 true in
\end{table}

\subsection{Effect of Planning Horizon for Local MPC}
\label{appendix:local-planning-results:planning_horizon}
We study how the planning horizon $T$ used in \eqref{eq:nonlinear-to} affects the MPC errors. Specifically, for the \code{IiwaBimanual} system and the corresponding goals generated using the scheme in \Cref{appendix:goal_selection}, we run \Cref{alg:mpc} with planning horizon from $1$ to $5$, a maximum iteration of $2$ ($T=1 \dots 5$, $n_\mathrm{max} = 2$ in \Cref{alg:trajopt_ctr}), and a MPC rollout horizon $H$ of $10$. 

The results are summarized in \Cref{tab:planning-horizon-comparison}. Surprisingly, we found that the planning horizon $T$ has little to no effect on the final error for MPC: our method is able to reach the goals within a tight tolerance for all values of $T$, even $T=1$. 

We believe this can be explained by the Bellman's principle of optimality. A 1-step greedy policy is long-term optimal if it minimizes the sum of the immediate 1-step cost and the true optimal cost-to-go. This has been experimentally verified on smooth dynamical systems~\cite{orrico2024building}. 

When casting our MPC cost \eqref{eq:nonlinear-to:cost} for $T=1$ into the Bellman equation, $\|u_1 - u_0\|^2$ corresponds to the 1-step cost, and $\|q_\text{goal} - q_1\|^2_\mathbf{Q}$ serves as an approximation of the cost-to-go. The empirical success of $T=1$ suggests that for goals generated by our method (\Cref{appendix:goal_selection}), the simple quadratic term is a highly effective approximation of the true optimal cost-to-go. 
\begin{table}[h]
\centering
\small
\begin{tabular}{@{}lccccc@{}}
\toprule
 & $T=1$ & $T=2$ & $T=3$ & $T=4$ & $T=5$ \\
\midrule
\multirow{2}{*}{Translation\,[mm]}
  & 2.02 & 1.99 & 1.64 & 2.11 & 2.10 \\
  & (3.10) & (4.38) & (3.14) & (3.79) & (3.17) \\
\addlinespace
\multirow{2}{*}{Rotation\,[mrad]}
  & 2.72 & 3.51 & 3.25 & 4.13 & 5.03 \\
  & (0.75) & (1.18) & (2.87) & (1.41) & (2.82) \\
\bottomrule
\end{tabular}
\caption{Translation and rotation Errors for \code{IiwaBimanual} under different planning horizon $T$. The first row is the mean; the second row (numbers in brackets) is the standard deviation.}
\label{tab:planning-horizon-comparison}
\vskip -0.2 true in
\end{table}

\subsection{Hyperparameters for Local MPC}
\label{app:hyperparameters_for_lcoal_mpc}
In \Cref{tab:hyperparameters_for_lcoal_mpc}, we summarize additional hyperparameters for local MPC. These parameters are kept the same for all experiments involving \Cref{alg:mpc} and \Cref{alg:mpc_real_dynamics}, including the results in \Cref{fig:planning-comparison} and \Cref{fig:second-order-mpc-comparison}.

The MPC cost function is defined by the quadratic cost matrices $\mathbf{Q}$ and $\mathbf{R}$ in \eqref{eq:linear_trajopt:cost}. $\mathbf{Q}_\mathrm{o}$ is the submatrix of $\mathbf{Q}$ corresponding to the object.
As we are primarily concerned with moving objects to goal poses, we set the robot-related entries in the cost matrix $\mathbf{Q}$ in \eqref{eq:linear_trajopt} to $0$, and rely on the action cost \eqref{eq:linear_trajopt:cost} and the CTR constraint \eqref{eq:linear_trajopt:ctr} to stay close to the linearization points. 

Other key hyperparameters include the CQDC smoothing parameter $\kappa$ and time step $h$. The choice of $\kappa$ and $h$ is a trade-off between the ease of planning and model accuracy. Larger $\kappa$ reduces the effect of contact dynamics smoothing: at $\kappa = 10^7$, the smoothed CQDC dynamics \eqref{eq:q_dynamics_log} is very close to the non-smooth formulation \eqref{eq:q_dynamic_socp}. However, gradients of non-smooth dynamics are well-known to be less informative for local MPC. Similarly, smaller $h$ increases simulation fidelity, but also means MPC needs to take more steps to reach the goal. In practice, we tune these by sweeping across a range of values and selecting the best performers. Specifically, we tested $\kappa$ from $10^2$ and $10^6$ in orders or magnitude (e.g. $10^2, 10^3, \dots$), and $h$ from $0.01$ to $0.1$.

To improve computational efficiency, we introduce a signed distance threshold, denoted as $\phi$-threshold in \Cref{tab:hyperparameters_for_lcoal_mpc}. When constructing a CTR or R-CTR, we only include feasibility constraints for contact pairs whose signed distance $\phi_i$ is less than this threshold. This technique effectively prunes contact pairs that are far from making contact, which reduces the number of constraints in the trust region and speeds up the MPC. 

For our experiments, the values were chosen based on the object's size: slightly smaller than the $0.28\mathrm{m}$ bucket diameter for \code{IiwaBimanual} and the $0.06\mathrm{m}$ cube edge length for \code{AllegroHand}. This is a very conservative choice. To speed up computation, we can compute a lower bound on the signed distance at contact $i$ at the next time step from $\phi_i$, $\mathbf{J}_i$ and $\mathbf{\Sigma}$. Excluding feasibility constraints corresponding to contact pairs whose lower bound is greater than 0 (or some small positive threshold) would significantly reduce problem size, thereby speeding up computation.

\begin{table}[h]
\centering
\small
\begin{tabular}{llcc}
\toprule
 &  &  \code{IiwaBimanual} &  \code{AllegroHand} \\
\midrule
\multirow{2}{*}{$\mathbf{Q}_\mathrm{o}$} & Rotation & $0.1\mathbf{I}_1$  & $\mathbf{I}_3$ \\
  & Translation & $\mathbf{I}_2$ & $5\mathbf{I}_3$  \\
$\mathbf{R}$ & & $0.01\mathbf{I}_6$ & $0.01\mathbf{I}_{16}$ \\
$\kappa$ & & $10^4$ & $10^4$ \\
$h$ [s] & & $0.02$ & $0.05$ \\
$\phi$-threshold [m] & & 0.2 & 0.05 \\
\bottomrule
\end{tabular}
\caption{Additional hyperparameters for MPC, used by both \Cref{alg:mpc} and \Cref{alg:mpc_real_dynamics}.}
\label{tab:hyperparameters_for_lcoal_mpc}
\vskip -0.1 true in
\end{table}

\subsection{Hyperparameters for 2nd-order Dynamics Stabilization}
\label{app:2nd_order_mpc:hyperparameters}
We use the same $\mathbf{CtrTrajOpt}$ parameters from \Cref{tab:planning-hyperparameters-allegro-and-iiwa}. Other important MPC hyperparameters are summarized in \Cref{tab:MPC_2nd_order_hyperparameters}. For closed-loop \code{IiwaBimanual} experiments, we use larger $N$ for goals that are further away, while keeping $N \times H$ around $25$. In contrast, as slippage happens a lot more often on the \code{AllegroHand} system, we keep $N$ constant for all goals to give MPC more chances to recover.

To validate our choice of the rollout horizon $H$, we studied how different values of $H$ affect tracking errors in simulation, with results presented in \Cref{fig:different_H_closed_loop}. We computed the Pearson correlation coefficient ($r$) between the  and tracking error. The analysis revealed a negligible correlation for \code{IiwaBimanual} ($r=0.00082$) and a very weak one for \code{AllegroHand} ($r=-0.059$). These findings confirm that performance is largely insensitive to $H$.
Ultimately, our selection of $H$ is constrained by computational limitations. Because the robot must pause while the MPC replans, we selected values of $H$ that balanced the replanning frequency and duration of these pauses.

\begin{table}[ht]
\centering
\small
\setlength{\tabcolsep}{2pt}
\begin{tabular}{@{}lcccc@{}}
\toprule
& \multicolumn{2}{c}{$N$ (Number of re-plans)} & \multicolumn{2}{c}{$H$ (MPC rollout horizon)} \\
\cmidrule(lr){2-3}\cmidrule(lr){4-5}
\textbf{System} 
                & open-loop & closed-loop
                & open-loop & closed-loop \\
\midrule
\code{IiwaBimanual}    & 1 & Variable & 25 & 5 \\
\code{AllegroHand} & 1 & 5           & 50 & 10 \\
\bottomrule
\end{tabular}
\caption{Hyperparameters for the MPC in \Cref{alg:mpc_real_dynamics}. \textbf{No Heuristics} and \textbf{Closed-loop} share the same parameters.}
\label{tab:MPC_2nd_order_hyperparameters}
\vskip -0.1 true in
\end{table}

\begin{figure}
\centering\includegraphics[width = 0.48\textwidth]{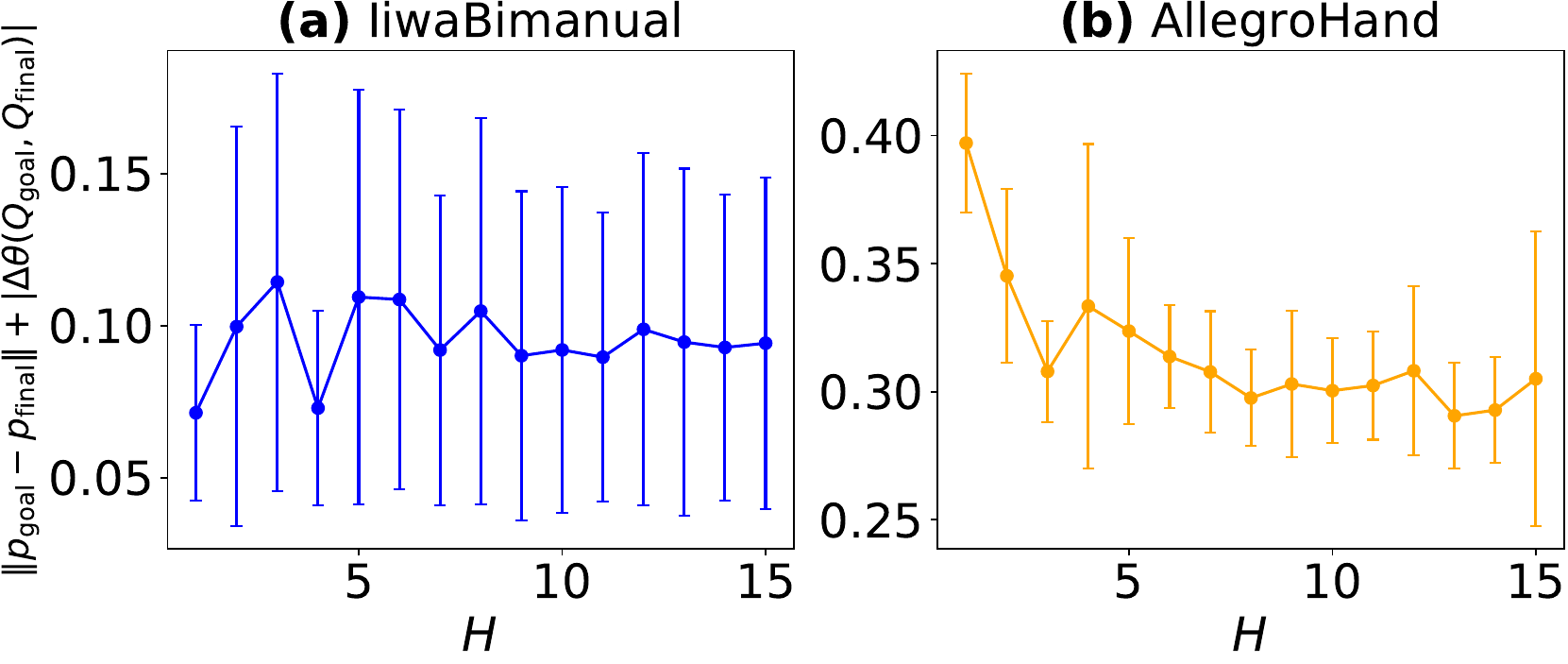}
\caption{Analysis of closed-loop tracking error as a function of the rollout horizon $H$. Performance is largely insensitive to this hyperparameter. For \code{IiwaBimanual}, there is a negligible correlation between $H$ and tracking error ($r = 0.00082$). A minor negative correlation is observed for \code{AllegroHand} ($r = -0.059$), which we attribute to violations of the quasi-dynamic assumption. Error bars represent $0.1\sigma$.}
\label{fig:different_H_closed_loop}
\vskip -0.1 true in
\end{figure}

\begin{figure}
\centering\includegraphics[width = 0.45\textwidth]{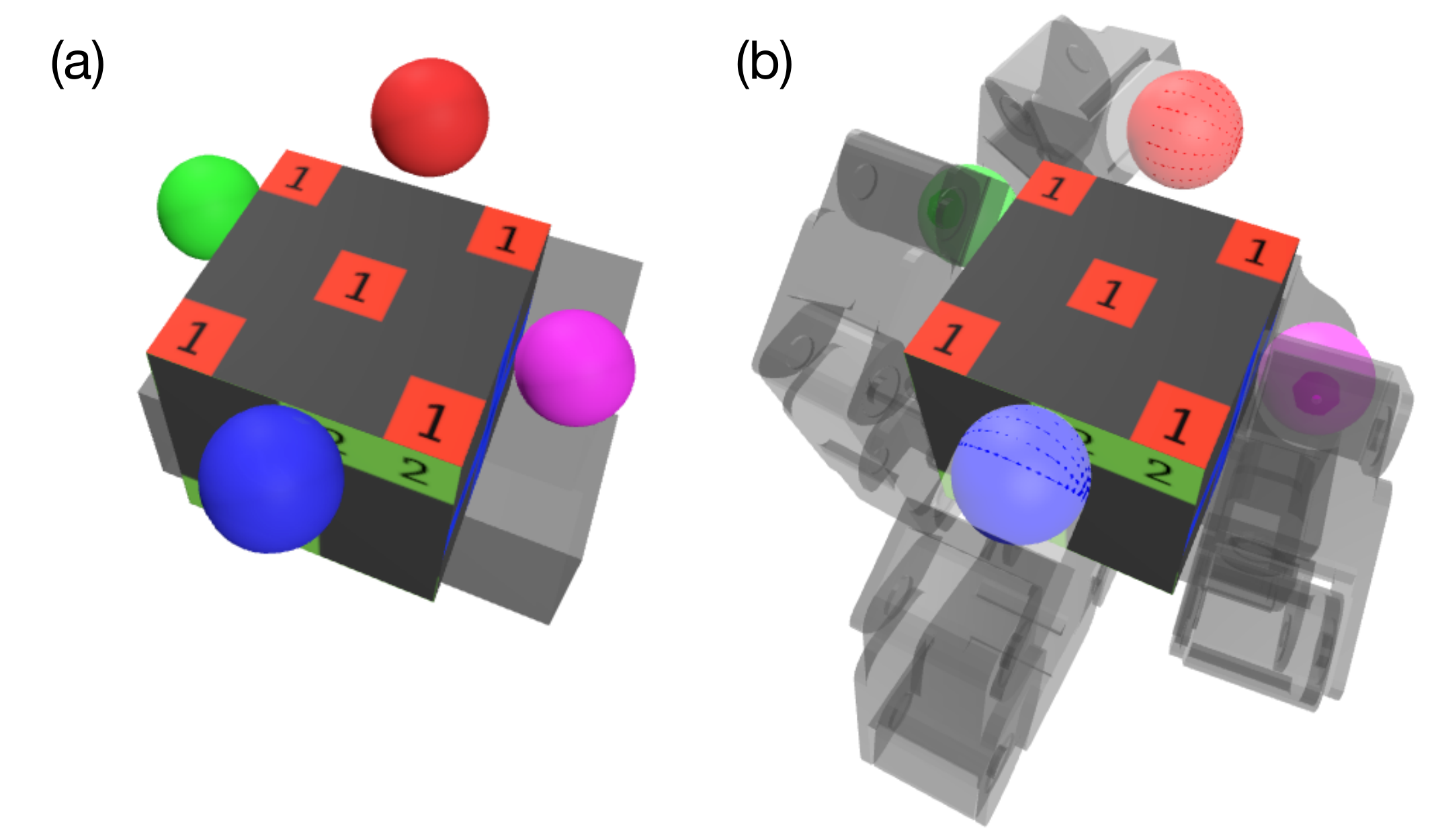}
\caption{(\textbf{a}) Visualization of the reduced-order system for the Allegro in \Cref{sec:reduced-order}, which consists of four spheres. (\textbf{b}) Visualization of the IK result on the optimized solution, mapping the reduced-order system to the full system.} 
\label{fig:reduced-order}
\vskip -0.1 true in
\end{figure}
\subsection{Heuristics for Generating Allegro Contact Configurations}
\label{appendix:allegro_contact_sampling_heurisitics}
\subsubsection{A Reduced-Order Model}\label{sec:reduced-order} Inspired by \cite[\S 5.4]{mls}, instead of solving \eqref{eq:nonlinear-grasp-sampling} directly on the full configuration space of $\qa$, we first solve the problem on a reduced-order model, then map the solution to $\qa$. For the Allegro hand, our reduced-order model consists of four spheres, each free to move in 3 dimensions with bounding-box joint limits (\Cref{fig:reduced-order}a). Then, we solve for $\qa$ by matching the fingertip positions to those of the spheres with Inverse Kinematics (IK), as shown in \Cref{fig:reduced-order}b. Our IK procedure solves the following Quadratic Program (QP) iteratively, 
\begin{subequations}\label{eq:diffik}
    \begin{align}
        \min_{\delta q, p_{+,k}}\;\; & \sum_k\|p_{+,k} - p_{\text{des},k}\|^2\label{eq:diffik:cost}\\
        \text{s.t.}\;\; & p_{+,k} = p_k +\frac{ \partial{p_k}}{\partial q}\delta q\label{eq:diffik:jacobian} \\
        & \mathbf{J}_{n_i}\delta q + \phi_i \geq 0 \label{eq:diffik:non-penetration}\\
        & -\varepsilon\mathbf{1}\leq \delta q \leq \varepsilon\mathbf{1},\label{eq:diffik:bounds}
    \end{align}
\end{subequations}
where $k$ indexes each fingertip and its corresponding sphere, $p_{\text{des},k}\in\mathbb{R}^3$ is the location of the sphere, $p_k$ is the location of the fingertips at the current iteration $q$. Note that \eqref{eq:diffik:jacobian} corresponds to a linearization of forward kinematics and \eqref{eq:diffik:non-penetration} enforces non penetration at every iteration. After an optimal $\delta q^\star$ is found, we start the next iteration with $q\leftarrow q + \delta q^\star$.

\begin{figure*}[h]
\centering
\subfloat[A goal which MPC with R-CTR successfully reached but the ETR baseline failed to reach. Note the initial contact points in frame (1) do not form a force-closure grasp. Without the dual feasibility constraint, ETR does not recognize the unilateral-ness of the initial contact configuration and fails catastrophically. ]{
    \includegraphics[width=0.98\textwidth]{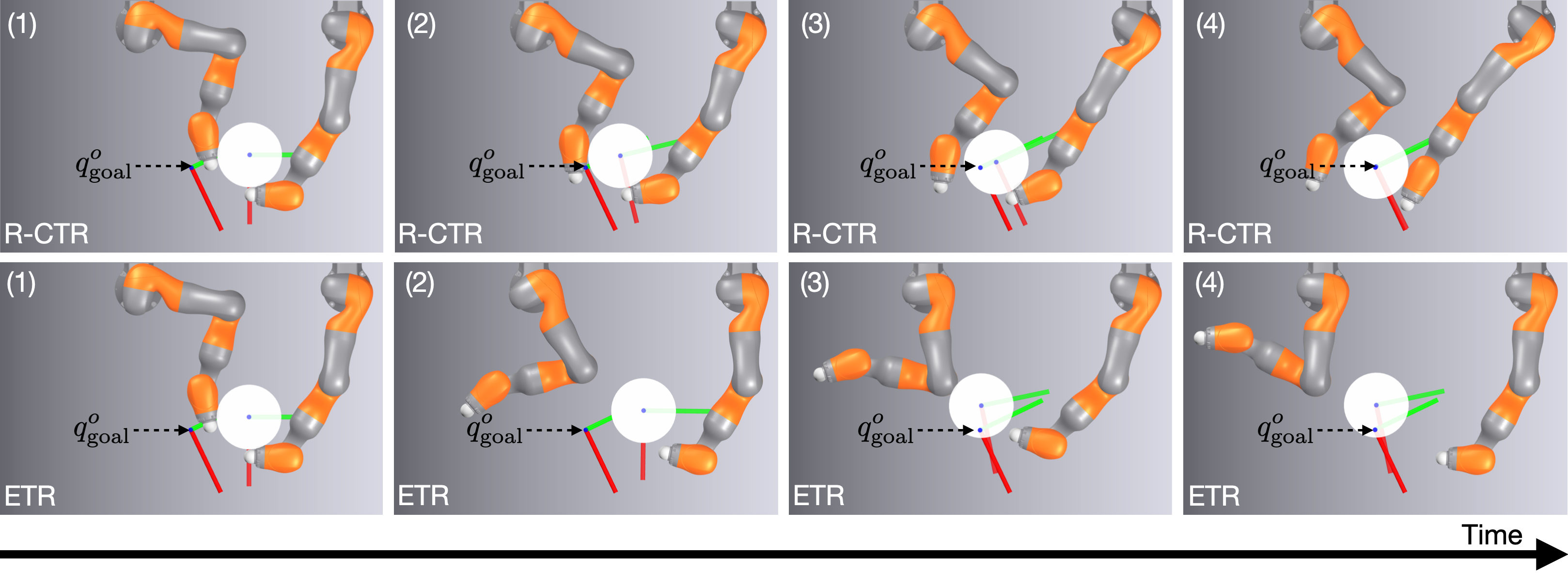}
    \label{fig:iiwa_planning_examples:etr_failure}
}
\hfill
\subfloat[A goal which both R-CTR and ETR-based MPC successfully reached. Note how the initial contact points form a nice anti-podal grasp on the cylinder.]{
    \includegraphics[width=0.98\textwidth]{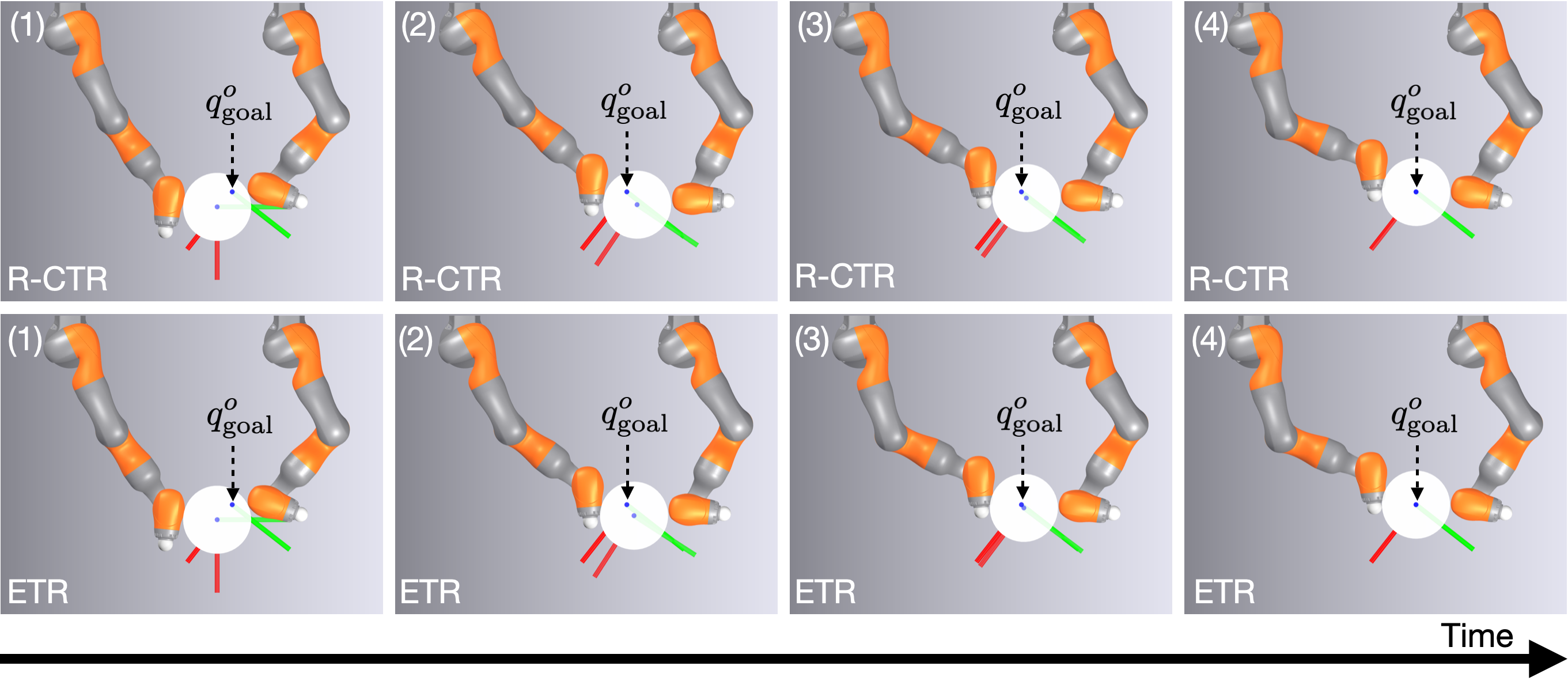}
    \label{fig:iiwa_planning_examples:etr_success}
}
\caption{Example trajectories for \code{IiwaBimanual} from running local MPC (\Cref{alg:mpc}) to reach the goals generated in \Cref{sec:local-planning-results:goal_selection}. While R-CTR-based MPC reaches both goals successfully, ETR-based MPC fails in (\textbf{a}) but succeeds in (\textbf{b}).}
\label{fig:iiwa_planning_examples}
\end{figure*}

\subsubsection{Contact Sampling Distribution}
Furthermore, we noticed that when we sample from the feasible set of \eqref{eq:nonlinear-grasp-sampling}, the spheres in the reduced-order model frequently do not fully end up in contact with the cube, resulting in a low robustness metric. Thus, after sampling sphere positions from their respective joint limits, we project the spheres to the nearest point on the surface of the cube, so that we are effectively sampling from a distribution of grasps.

\subsection{Local MPC Examples}
\label{app:local_mpc_examples}
For brevity and to maintain the narrative flow of the main text, we present detailed examples of local MPC (\Cref{alg:mpc}) here, which include \Cref{fig:iiwa_planning_examples} and \Cref{fig:allegro_planning_examples}.

\begin{figure*}[t]
\subfloat[A goal which MPC with R-CTR successfully reached but the ETR baseline failed to reach. The failure begins in frame (2), where a small gap appears between the thumb and the cube. Without the dual feasibility constraint, the ETR-based MPC incorrectly calculates that moving the thumb away from the object will reduce tracking error, leading to a catastrophic failure.]{
    \includegraphics[width=0.98\textwidth]{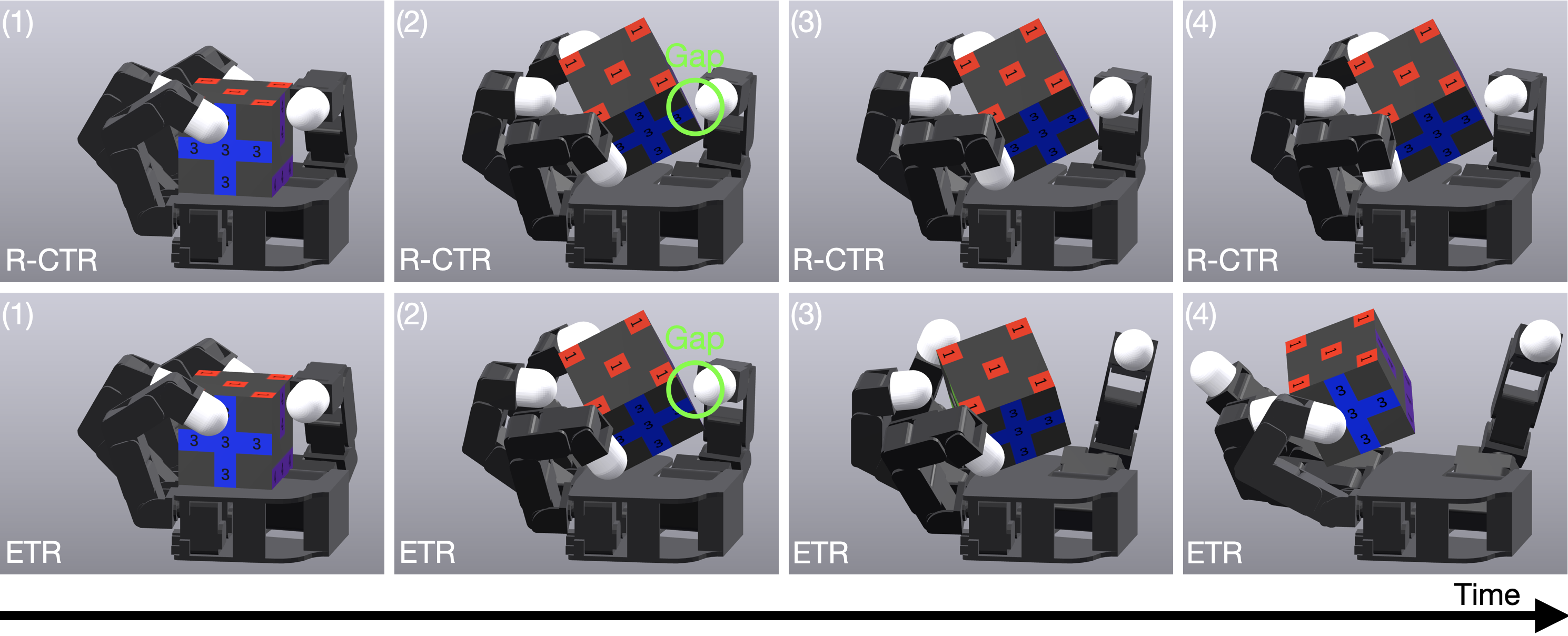}
    \label{fig:allegro_planning_examples:etr_failure}
}
\hfill
\subfloat[A goal which both R-CTR and ETR-based MPC successfully reached. Note how the fingers form a nice closure grasp on the cube.]{
    \includegraphics[width=0.98\textwidth]{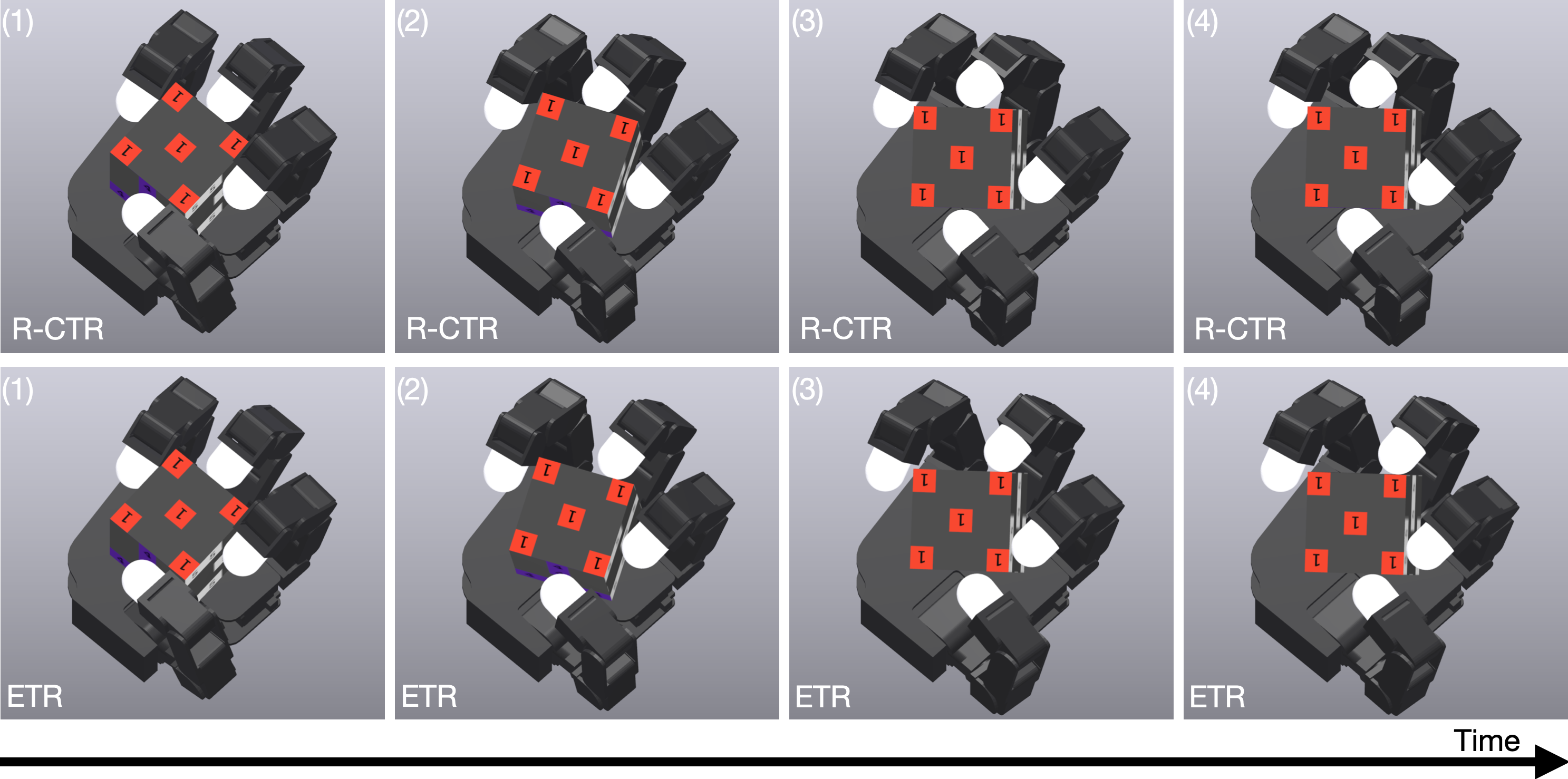}
    \label{fig:allegro_planning_examples:etr_success}
}
\caption{Example trajectories for \code{AllegroHand} from running local MPC (\Cref{alg:mpc}) to reach the goals generated in \Cref{sec:local-planning-results:goal_selection}. While R-CTR-based MPC reaches both goals successfully, ETR-based MPC fails in (\textbf{a}) but succeeds in (\textbf{b}).}
\label{fig:allegro_planning_examples}
\end{figure*}
